\documentclass{article} % For LaTeX2e
\usepackage{iclr2027_conference,times}

\PassOptionsToPackage{table,xcdraw,dvipsnames}{xcolor}
\usepackage{iclr2027_conference,times}

\usepackage[table]{xcolor}
\usepackage{xcolor}
\usepackage{graphicx}
\usepackage{amsmath}
\usepackage{amssymb}
\usepackage{amsthm}
\usepackage{wrapfig}
\usepackage{caption}
\usepackage{subcaption}
\usepackage{wasysym}
\usepackage{mathrsfs}
\usepackage{cancel}
\usepackage[mathcal]{euscript}
\usepackage{bbm}
\usepackage{lipsum}
\usepackage{algorithm}
\usepackage{algpseudocode}
\usepackage{xspace}
\usepackage[section]{placeins}
\usepackage{float}
\usepackage{makecell}
\usepackage[normalem]{ulem}
\usepackage{tabularx}
\usepackage{booktabs}
\usepackage{array}
\usepackage{etoolbox}
\usepackage{multirow}
\usepackage{xfrac}
\usepackage{outlines}
\usepackage{enumitem}
\usepackage{fontawesome5}
\usepackage[most]{tcolorbox}
\usepackage{url}
\usepackage[hidelinks]{hyperref}
\usepackage[capitalise]{cleveref}

\definecolor{citecolor}{HTML}{2E8B6E}

\usepackage{hyperref}
\hypersetup{
    colorlinks=true,
    citecolor=citecolor,
    linkcolor=black,
    urlcolor=black
}

\newtcolorbox{conditionbox}{
    colback=black!4,
    colframe=black!4,
    boxrule=0pt,
    arc=0pt,
    left=4pt,
    right=4pt,
    top=3pt,
    bottom=3pt,
    boxsep=0pt,
    before skip=4pt,
    after skip=4pt,
    width=\columnwidth,
    breakable,
}

\algrenewcommand\algorithmicrequire{\textbf{Require:}}
\algrenewcommand\algorithmicensure{\textbf{Ensure:}}

\newcolumntype{Y}{>{\centering\arraybackslash}X}

\setenumerate[2]{label=\alph*.}
\setenumerate[3]{label=\roman*.}

\crefname{equation}{Eq.}{Eqs.}
\crefname{figure}{Fig.}{Figs.}
\crefname{table}{Tab.}{Tabs.}
\crefname{section}{Sec.}{Secs.}
\crefname{algorithm}{Alg.}{Algs.}
\crefname{appendix}{Appx.}{Appxs.}
\crefname{thm}{Thm.}{Thms.}
\crefname{cor}{Cor.}{Cors.}
\crefname{lemma}{Lem.}{Lems.}
\crefname{problem}{Prob.}{Probs.}

\usepackage{amsmath,amsfonts,bm}

\def\eqref#1{equation~\ref{#1}}
\def\1{\bm{1}}

\def\va{{\bm{a}}}

\def\vg{{\bm{g}}}

\def\vo{{\bm{o}}}
\def\vp{{\bm{p}}}

\def\vx{{\bm{x}}}
\def\vy{{\bm{y}}}
\def\vz{{\bm{z}}}

\DeclareMathAlphabet{\mathsfit}{\encodingdefault}{\sfdefault}{m}{sl}
\SetMathAlphabet{\mathsfit}{bold}{\encodingdefault}{\sfdefault}{bx}{n}

\def\sA{{\mathbb{A}}}

\def\sX{{\mathbb{X}}}

\def\sZ{{\mathbb{Z}}}

\newcommand{\E}{\mathbb{E}}

\newcommand{\R}{\mathbb{R}}

\DeclareMathOperator*{\argmin}{arg\,min}

\newtheorem{defn}{Definition}

\newtheorem{prop}[defn]{Proposition}
\newtheorem{assum}[defn]{Assumption}

\newtheorem{thm}[defn]{Theorem}
\newtheorem{cor}[defn]{Corollary}

\providecommand{\R}{\ensuremath \mathbb{R}}

\newcommand{\rank}{r}

\newcommand{\policy}{\bm{\pi}_\theta}

\newcommand{\sg}{\mathrm{sg}} % conceptual stop-gradient operator

\newcommand{\traindist}{P_{\Train}}
\newcommand{\testdist}{P_{\Test}}
\newcommand{\traindiststate}{\traindist^{X}}
\newcommand{\testdiststate}{\testdist^{X}}
\newcommand{\Train}{\mathrm{tr}}
\newcommand{\Test}{\mathrm{te}}

\newcommand{\statespace}{\sX}
\newcommand{\actspace}{\sA}
\newcommand{\chunkspace}{\sA_H}
\newcommand{\latspace}{\sZ}

\newcommand{\DemoData}{\mathcal D}
\newcommand{\expertdist}{q^\star}
\newcommand{\distrisk}{\mathcal R^{\mathrm{dist}}}

\newcommand{\obsenc}{\phi}              % Change only this line to use \psi.
\newcommand{\flowfield}{v_\theta}

\newcommand{\rhotrain}{\rho_{\Train}}
\newcommand{\rhotest}{\rho_{\Test}}
\newcommand{\anchorparam}{\psi}       % Distinct from the encoder symbol.
\newcommand{\anchorpolicy}{\rho_{\anchorparam}}

\newcommand{\ancmark}{\mathsf \dagger}
\newcommand{\anc}[1]{#1^{\ancmark}}

\newcommand{\QryEmb}{\Xi}

\newcommand{\RelDelta}{\Delta}

\newcommand{\reldelta}{\bm\delta}
\newcommand{\querylatent}{\vz}
\newcommand{\anchorlatent}{\vz^{\ancmark}}
\newcommand{\AnchorLatent}{Z^{\ancmark}}

\newcommand{\state}{\vx}
\newcommand{\act}{\va}
\newcommand{\achunk}{\bm{A}}
\newcommand{\obs}{\vo}
\newcommand{\goal}{\vg}

\newcommand{\actdim}{d_a}
\newcommand{\chunkdim}{d_A}
\newcommand{\flowstate}{\bm{\achunk}}
\newcommand{\flowtime}{\tau}

\newcommand{\trainpair}{\overline D_{\Train}}
\newcommand{\testpair}{\overline D_{\Test}}

\definecolor{reqgreen}{HTML}{2E8B6E}

\newcommand{\reqtag}[1]{%
  \tikz[baseline=(tag.base)]{
    \node[
      fill=reqgreen,
      fill opacity=0.20,
      text opacity=1,
      inner xsep=2pt,
      inner ysep=0.5pt
    ] (tag) {#1};
  }%
}

\newcommand{\structurereq}{\reqtag{\textbf{(S) Bilinear Structure}}}
\newcommand{\coveragereq}{\reqtag{\textbf{(C) Residual Coverage}}}

\newcommand{\flowrisk}[1]{\mathcal E_{#1}^{\mathrm{flow}}}

\newcommand{\Law}{\operatorname{Law}}

\newcommand{\ours}{\textsc{BFP}\xspace}
\newcommand{\brp}{\textsc{BT}}
\newcommand{\flow}{\textsc{Flow}}
\newcommand{\cta}{\textsc{CTA}}
\newcommand{\darp}{\textsc{DARP}}
\newcommand{\edp}{\textsc{EquiDP}}

\title{Bilinear Flow Policy: \\
Distributional Extrapolation for \\
Goal-Conditioned Visuomotor Imitation
\vspace{0.6em}
}

\author{
\vspace{0.4em}
\begin{tabular}{c}
Wonsuhk Jung$^{1}$ \qquad
Sundhar Vinodh Sangeetha$^{1}$ \qquad
Chen Xu$^{2}$ \\[3pt]
Abhishek Gupta$^{3}$ \qquad
Masha Itkina$^{2}$ \qquad
Shreyas Kousik$^{1}$ \qquad
Haruki Nishimura$^{2}$ \\[6pt]
\normalfont\mdseries\small
$^{1}$Georgia Institute of Technology
\qquad
$^{2}$Toyota Research Institute
\qquad
$^{3}$University of Washington \\[2pt]
\footnotesize
\texttt{https://bilinearflow.netlify.app}
\end{tabular}
\vspace{-1.6em}
}

\newcommand{\sr}[2]{#1{\scriptstyle\,\pm\,#2}}
\definecolor{avgblue}{RGB}{232,243,255}

\iclrfinalcopy % Uncomment for camera-ready version, but NOT for submission.
\begin{document}

\maketitle

\begin{abstract}

% motivating sentence of goal-condi
Goal-conditioned imitation learning (GCIL) with flow matching is a promising framework that can represent multimodal behaviors while adapting to diverse, user-specified goals, %through multiple valid behaviors,
% Yet the imitation objective directly supervises only demonstrated situations; pursuing a goal outside the demonstrated range can take the robot into new situations for which no action behavior was shown.
yet often fails when goals lie outside the demonstration support.
% The central challenge is \textit{distributional extrapolation}, where policies must generalize to unseen visual observation-goal pairs without collapsing multiple valid behaviors present in the training data.
To extrapolate to such unseen goals without collapsing multimodality---a problem we call \emph{distributional extrapolation}---we introduce \emph{Bilinear Flow Policy} (\ours{}), 
%We study this problem in visuomotor GCIL, where a policy must generalize to unseen goals from visual demonstrations while preserving behavioral diversity.
%While extrapolative generalization remains challenging in general, we can often achieve decouple a novel observation-goal pair into a known training example and an in-distribution additive relation between that training example and the given pair.
% A natural route to extrapolation is to learn, from comparable example pairs in the demonstrations, how differences in observation and goal translate into changes in behavior.
% For an unseen goal, the policy can then compare its current input with a related training example and reuse that learned relation to infer the required behavior.
%While prior work demonstrated the effectiveness of this transductive, retrieval-based learning in low-dimensional, priviledged state space, doing this in the visuomotor setting is difficult: it is unclear which training examples to compare, image observations are high-dimensional unlike compact states, and the same goal may admit multiple valid behaviors.
a generative visuomotor policy that combines transductive retrieval with a bilinear conditional flow.
Given an unseen observation-goal pair, \ours{} retrieves an ``anchor" training example and transductively reformulates the unseen pair as this familiar anchor plus a residual term. 
For this decomposition to guide action prediction, the residual must compactly encode how the current observation-goal pair differs from the anchor, and the anchor must be chosen so that this difference is predictive of the corresponding action distribution.
% For this decomposition to guide action prediction, anchor must provide a relevant reference and the residual must compactly capture how the current observation and goal differ from it. 
BFP achieves this with pretrained visual features and a novel learned anchor-selection algorithm.
% Our novel, learnable anchor-selection algorithm retrieves an anchor such that the corresponding residual is still well supported by the training data. 
The novel bilinear flow then models how the anchor and the residual jointly determine the multimodal action distribution. %allowing relationships learned from familiar combinations to be extrapolated to novel ones at unseen goals.
%which selects useful reference examples, encodes their difference in a compact visuomotor space, and transfers the learned relation through a bilinear conditional flow over multimodal actions.
We prove that, for bilinear flow under suitable assumptions, action distribution error at unseen goals is bounded by the in-distribution flow-matching error up to problem-dependent factors.
%connecting in-distribution flow error to error at unseen goals.
Across five manipulation tasks in simulation, \ours{} achieves $2.63\times$ the out-of-distribution success rate of a 
% matched generative 
GCIL policy and $1.36\times$ that of the strongest extrapolation-targeted baseline. 
On two real-world tasks, \ours{} improves over GCIL by $32\%$.
Finally, our theory yields practical, pre-deployment diagnostics for predicting which trained policies will extrapolate well and to which unseen goals.

% Goal-conditioned imitation learning (GCIL) policies enable a single policy to reuse demonstrations across many desired outcomes, but their behavior for goals beyond demonstration support receives no direct supervision, making extrapolation unreliable.
% We introduce \emph{Bilinear Flow Policy} (\ours{}), a generative goal-conditioned visuomotor policy that extrapolates to unseen goals by relating new observation-goal conditions to demonstrated ones.
% Unlike existing transductive learning methods that always output deterministic action predictions, \ours{} natively support conditional flow matching for action generation, enabling structured extrapolation of multimodal action distributions.
% For visuomotor inputs, we further use a compact representation and learned anchor policy to construct the relations used for extrapolation.
% We provide a theorem relating in-distribution flow error to extrapolation error.
% Across goal-conditioned manipulation benchmarks, \ours{} achieves $2.35\times$ the out-of-distribution (OOD) success of a matched generative GCIL policy and $1.95\times$ that of the strongest extrapolation-targeted baseline.
% Our analysis further provides practical signals for when unseen-goal extrapolation is probable.
% Across two real-world manipulation tasks, \ours{} improves success rate over the GCIL baseline by an average of $30\%$ percentage points.
\end{abstract}
\section{Introduction}

% should we be confined to GCIL? Or GCP makes sense?
Goal-conditioned imitation learning (GCIL) is a promising framework to learn versatile visuomotor behavior: a single policy can adapt to diverse, user-specified goals at deployment \citep{ding2019goal, park2025ogbench}.
Conditional flow matching further makes this framework expressive by modeling multimodal distributions of actions, allowing the policy to represent multiple behavior modes to achieve the same goal \citep{reuss2023goal, reuss2024multimodal}.
Nevertheless, deployment may require the policy to act toward unseen goals, where it encounters observation--goal pairs not seen during training.
Standard GCIL policies, including generative variants, can become unreliable on such unseen inputs \citep{yang2023essential, urain2024deep}.
% Nevertheless, practical deployment may require goals beyond demonstration support, inducing unseen observations along the resulting trajectory.
% Flow-matching objective provides no direct supervision for these out-of-distribution (OOD) observation-goal combinations, so a policy that fits the demonstrations well may still extrapolate unreliably.
This motivates \emph{distributional extrapolation}: extending goal-conditioned behavior to unseen goals while preserving multimodality.
% Thus, unseen-goal extrapolation requires a generalizable structure for extending demonstrated behavior beyond the data support.

% Prior work obtains unseen-goal generalization by strengthening state-goal supervision or structuring goal-conditioned value functions \citep{yang2023essential,hong2022bilinear}.
One natural route to extrapolation is to identify how behavior should change as the policy input (i.e., current observation and goal) changes, then reuse the same rule to infer behavior for a new input by comparing it with a training example
\citep{ousherovitch2026learning}.
When this relation is known, it can be encoded directly into the policy; for instance, the behavior should remain invariant to task-irrelevant perceptual perturbations \citep{zhu2023viola}, transform equivariantly under known geometric changes \citep{wang2024equivariant}, or be composed of reusable policy components \citep{mishra2023generative}.
These methods generalize by encoding a particular, known deployment shift as an inductive bias. 
They are effective when this prescribed relation matches the nature of the shift, but require the relevant task-specific semantics or transformation laws to be specified a priori.

When the underlying observation-goal-action relation is unknown, an attractive alternative is to \emph{learn} it from comparisons among demonstrations.
Bilinear transduction \citep{netanyahu2023learning} achieves this by decomposing an observation-goal condition $x$ into an ``anchor" training example $a$ and a residual $\delta = x - a$, then constrains the predictor to a low-rank interaction between features of the two factors.
Thus, rather than specifying a task-specific transformation explicitly, the model learns from the data how the anchor and the residual jointly determine behaviors within a generic low-rank relationship.
% the anchor-delta representation and their bilinear interaction are prescribed, while the task-specific effect of a change on behavior is learned from paired examples.
This reparameterization converts an out-of-support query into a potentially unseen combination of individually familiar factors; under low-rank, coverage, and conditioning assumptions, prediction error at unsupported queries is provably bounded by the training error \citep{shah2020sample, simchowitz2023tackling}.
Existing robotic instantiations, however, are state-based and deterministic, leaving open how to construct this decomposition from visuomotor demonstrations and how to extrapolate a multimodal action distribution.

\begin{figure}[t]
    \centering
    \includegraphics[width=0.9\linewidth]{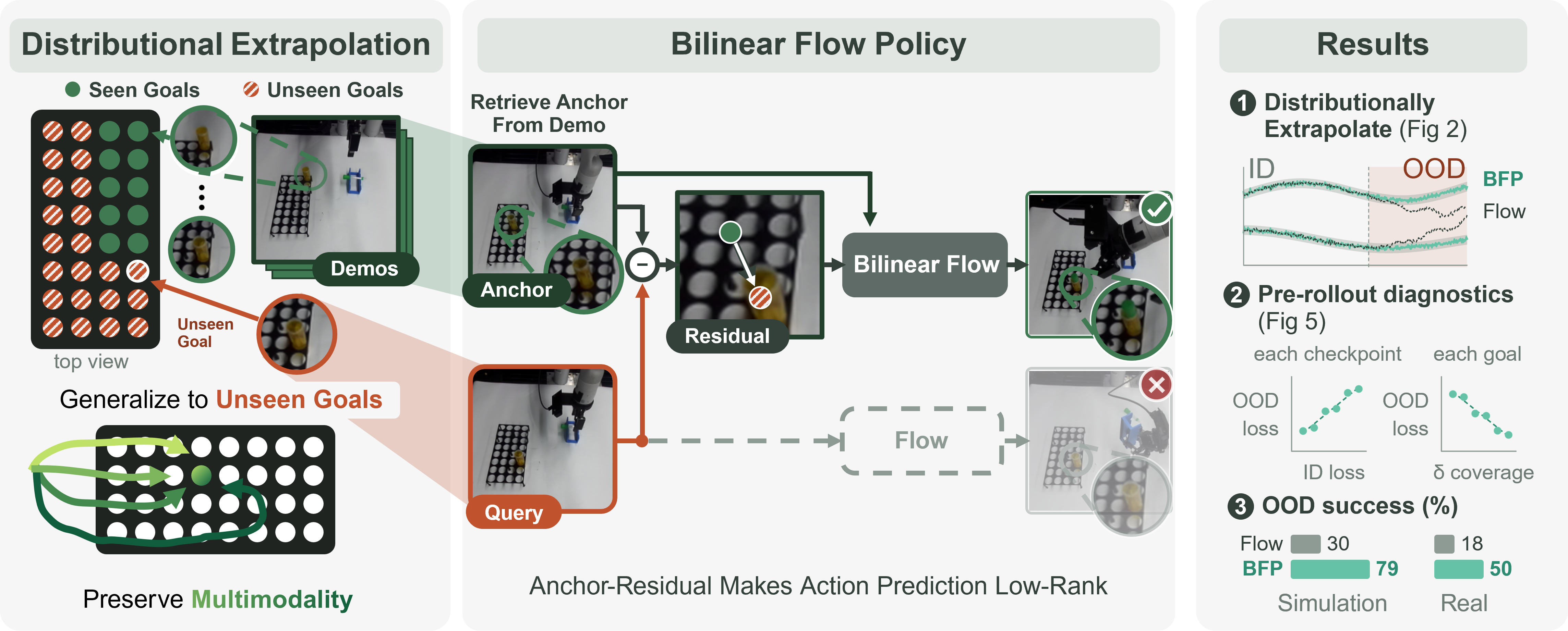}
    \caption{
    \textbf{\ours{} Overview.}
    \ours{} learns how changes across observation-goal conditions affect behavior and reuses this relation to extrapolate from a demonstrated anchor to an unseen goal.
    A bilinear conditional flow generates the resulting multimodal action distribution, while the explicit extrapolation rule enables theoretical guarantees and practical diagnostics.
    }
    \label{fig:teaser}
    \vspace{-1.5em}
\end{figure}

In this work, we ask whether this relational formulation can support generative goal-conditioned visuomotor policies.
We introduce \textit{Bilinear Flow Policy} (BFP), which extrapolates to unseen goals by transductively relating novel observation-goal conditions to demonstrated examples in the training data.
\ours{} applies bilinear transduction to a conditional flow field, extending relational extrapolation from deterministic actions to multimodal action distributions. It also combines a compact visual representation with a learned anchor policy to find generalizable relations in visuomotor settings.

We also derive an end-to-end distributional extrapolation bound that relates in-distribution flow-matching error to the Wasserstein error of the generated action distribution at unseen goals.
The bound's training-error and coverage terms further motivate two pre-deployment diagnostics: held-out flow loss for selecting among trained policies, and anchor--residual coverage for identifying unseen-goal queries on which a selected policy is likely to remain reliable.

We evaluate \ours{} on five goal-conditioned visuomotor manipulation tasks across three simulation benchmarks.
On the simulation tasks, \ours{} achieves $79\%$ average out-of-distribution (OOD) success ($2.63\times$ that of matched GCIL and $1.36\times$ that of the strongest extrapolation-targeted baseline) while retaining comparable in-distribution performance, across 4,200 rollouts.
On the two real-world robotics tasks, \ours{} improves over the matched GCIL baseline by 32 percentage points on average.
% Across hundreds of trained policies, theory-motivated diagnostics further indicate which checkpoints are most generalizable, and for a given checkpoint, which unseen goal conditions the policy can generalize to.

In summary, our contributions are threefold.
\begin{itemize}
    \item We introduce \ours{}, a generative goal-conditioned visuomotor imitation policy learning framework that extrapolates to unseen goals preserving behavior multimodality.
    % bilinear conditional-flow architecture that extrapolates multimodal action-chunk distributions to unseen observation-goal conditions.
    % We make bilinear transduction operational in visuomotor policies through compact pre-trained visual representations and learned anchor selection.

    \item 
    We derive an end-to-end bound relating in-distribution flow error to unseen-goal action-distribution error.
    The theorem motivates practical pre-deployment diagnostics for selecting trained policies and assessing their reliability on specific unseen goals.
    % together with pre-deployment diagnostics for selecting among trained policies and assessing unseen-goal reliability.

    \item Across controlled multimodal regression, five simulated manipulation tasks, and two real-world tasks, \ours{} substantially improves OOD performance over matched generative GCIL and extrapolation-targeted baselines.

\end{itemize}

\section{Related Work}
We first situate our work within the generalization regimes studied in goal-conditioned policies.
We then categorize extrapolation methods by how they map changes in policy inputs into behavior changes and reuse that rule beyond training support.
Prior work either hand-crafts this rule or use demonstration data to learn it.
\ours{} follows the latter approach.

\textbf{Generalization Regimes of Goal-Conditioned Policies.}
Goal-conditioned policies use a single policy  across many desired outcomes \citep{gupta2019relay, park2025ogbench}, while generative variants model multimodal action distributions for a given goal
\citep{reuss2023goal, reuss2024multimodal}.
Their generalization beyond training state--goal combinations has been studied under two distinct regimes.
In \emph{stitching}, where both the state and goal have been seen during training, but not together, the policy must recombine familiar transitions or trajectory segments to connect them
\citep{lawson2026self, luo2026generative}.
% where familiar states, transitions, or trajectory segments must be recombined to reach state--goal pairs not observed together during training
% \citep{ghugare2024closing, lawson2026self}.
Another studies \emph{unseen-goal generalization}, where the desired goal itself lies outside the training support and pursuing it requires the policy to act at out-of-distribution state--goal inputs
\citep{hong2022bilinear, yang2023essential}.
We study this regime in generative visuomotor imitation, where policies must generalize from visual inputs while preserving multimodal behavior.
% We focus on the latter.
% extrapolation beyond demonstrated goals, where pursuing a new goal also induces unseen observation--goal combinations along the rollout.
% \ours{} combines a generative goal-conditioned policy with a relational mechanism for this form of extrapolation.

\textbf{Hand-crafted Extrapolation Rules.}
% We use \emph{behavioral relation} to denote how a change in the policy input---its current observation and goal---should alter the required behavior.
Some approaches extrapolate through hand-specified rules for how behavior changes with the policy input
\citep{urain2024deep}.
Representation-based methods seek invariance to task-irrelevant variation by reducing observations to task-relevant objects, regions, keypoints, or geometry 
\citep{zhu2023viola, manuelli2019kpam, zhu2023learning}.
Equivariant policies prescribe how known scene transformations should transform actions or action distributions
\citep{wang2024equivariant,funk2024actionflow}.
Compositional methods define reusable factors (e.g., skills, objectives, or sensor modalities) and prescribe how they may be recombined \citep{mishra2023generative, urain2023se, wang2024poco}.
Their representations and components may be learned, but the rule is fixed a priori; this is effective when relevant structure is known, but restrictive otherwise.
We instead study settings where this rule is not obvious and must be inferred from demonstrations.

\textbf{Learning Extrapolation Rules from Demonstrations.}
% mention transduction?
When no rule is known \textit{a priori}, some prior works learn it by predicting a desired action from training examples and residuals that quantify the input change from each example to the query.
\darp{} \citep{pfeifer2026difference} instantiates this locally by retrieving the $k$ nearest examples, predicting from each example and its residual, and aggregating the results.
% too scarce information, how do they predict in a minimally technical details
Bilinear transduction (BT) \citep{netanyahu2023learning} instead assumes that a suitable anchor and compact residual induce a low-rank action-prediction problem, and uses a bilinear model to predict actions for unseen anchor--residual combinations.
% we should mention why this is good
Subsequent work extends this principle to offline RL and task-level extrapolation
\citep{kim2026compositional, ousherovitch2026learning}.
Robotic applications, however, have considered state inputs, where residuals are directly defined in compact coordinates and anchors are chosen using task-specific heuristics.
We study whether the same principle extends to generative visuomotor imitation, where defining a compact residual, selecting an anchor that exposes the low-rank relation, and preserving multiple valid actions are all nontrivial.
\section{Problem Formulation and Preliminaries}
\label{sec:problem_prelim}

\textbf{Problem Formulation: Distributional Extrapolation.}
\label{sec:problem}
Let $\actspace \subset \R^{\actdim}$ be the primitive action space and $\chunkspace := \actspace^{H} \subset \R^{\chunkdim}$, $\chunkdim=H\actdim$ be the action-chunk space.
We are given expert demonstrations:
$
    \DemoData=\{\tau^{(i)}\}_{i=1}^{N},
    \tau^{(i)}
    =\{(\state^{(i)}_t,\act^{(i)}_t)\}_{t=1}^{T_i},
    % \label{eq:dataset}
$
where 
% $i$ indexes a trajectory, $t$ is environment time,
$\act^{(i)}_t\in\actspace$ is a primitive action, and $\state_t^{(i)}=(\obs^{(i)}_t,\goal^{(i)})\in\statespace$ is the policy condition.
Here $\obs_t^{(i)}$ is a state or visuomotor observation, and $\goal^{(i)}$ is a compact, explicitly specified episode goal (e.g., target object pose; See \Cref{fig:teaser}).
Each timestep defines an $H$-step action-chunk target
$
    \achunk_t^{(i)}
    :=(\act^{(i)}_t,\ldots,\act^{(i)}_{t+H-1})\in\chunkspace.
    \label{eq:action_chunk}
$
% The resulting $(\state, \achunk)$ pairs define an empirical training distribution $\hat{\traindiststate}$.

Let $\traindiststate$ and $\testdiststate$ denote the marginal distribution of policy conditions under training and evaluation, respectively; we refer to conditions within the demonstrated support as in-distribution (ID) and those outside it as out-of-distribution (OOD).
% The finite dataset $\DemoData$ contains trajectories collected under $\traindiststate$.
We assume the same expert conditional $\expertdist(\cdot \mid \state)$ in both regimes, while $\testdiststate$ may include conditions induced by unseen goals outside $\operatorname{supp}(\traindiststate)$.
% that specifies a distribution over valid action chunks for each policy condition $\state \in \statespace$.
% Training and evaluation share this conditional expert behavior,
% but may encounter different conditions.
% Let $\traindist$ and $\testdist$ denote the marginal distributions over policy conditions at training and at test, respectively.
% Demonstrations are collected under $\state\sim P_{\Train}^X$, whereas evaluation uses $\state\sim P_{\Test}^X$, which may include held-out goals or initial configurations outside $\operatorname{supp}(P_{\Train}^X)$.
Given only $\DemoData$, we seek a policy $\policy$ with small test distributional risk
\begin{equation}
    \distrisk_{\Test}(\policy)
    :=
    \E_{\state \sim \testdiststate}
    \left[
        W_2^2\!\left(
            \policy(\cdot \mid \state),
            \expertdist(\cdot \mid \state)
        \right)
    \right].
    \label{eq:distribution_goal}
\end{equation}
where $W_2^2$ is the squared $2$-Wasserstein distance \citep{peyre2019computational}.

\begin{wrapfigure}{r}{0.5\columnwidth}
    \centering
    \includegraphics[width=1.0\linewidth]{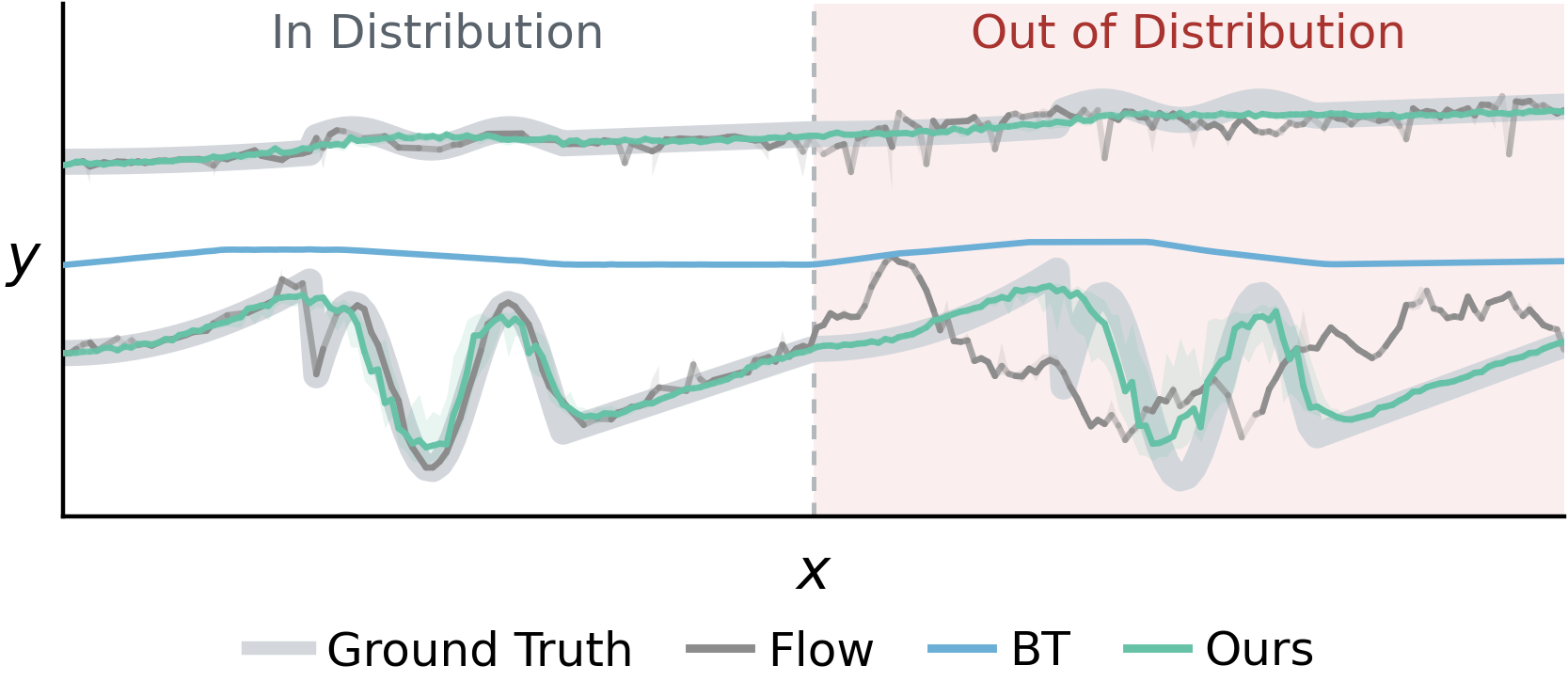}
    \caption{
        \textbf{Distributional extrapolation.}
        The expert conditional distribution $\expertdist (y \mid x)$ is a mixture of two modes (grey), evaluated outside $\operatorname{supp}(\traindiststate)$ (shaded red).
        \brp{} averages the modes (orange), while CFM captures multimodality but not extrapolation (purple).
        \ours{} achieves both (red).
    }
    \label{fig:bfp_regression}
    \vspace{-2mm}
\end{wrapfigure}
We call this problem \emph{distributional extrapolation}; a policy learned from demonstrated conditions must match the full expert action distribution at unseen conditions.
This matters when a condition admits multiple action chunks: as illustrated in \Cref{fig:bfp_regression}, mean regression is vulnerable to mode-averaging, while distributional modeling alone does not specify how the expert modes should extend beyond training support.

\textbf{Conditional Flow Matching.}
Conditional flow matching (CFM) \citep{lipman_flow_2024} transports a simple prior $\flowstate_0 \sim p_0$  (e.g., Gaussian) to the expert action chunk $\flowstate_1 = \achunk \sim \expertdist(\cdot \mid \state)$.  
% Let
% \begin{equation}   \flowstate_1=\achunk\sim\expertdist(\cdot\mid \state),
%     \qquad
%     \flowstate_0\sim p_0,
% \end{equation}
For a conditional probability path $p_{\flowtime}(\flowstate_{\flowtime}\mid\achunk)$ with corresponding velocity field $u_{\flowtime}(\flowstate_{\flowtime}\mid\achunk)$, CFM minimizes 
% where
% $\flowtime\in[0,1]$ is flow time.  
% A conditional velocity model is trained by minimizing
\begin{equation}
\begin{split}
    \mathcal L_{\rm CFM}(\theta)
    :={}&
    \E_{\substack{
        (\state,\achunk)\sim\DemoData,
        \flowtime\sim\mathcal U[0,1],
        \flowstate_{\flowtime}\sim
        p_{\flowtime}(\cdot\mid\achunk)
    }}
    \left[
        \left\|
        \flowfield(
            \flowstate_{\flowtime},
            \flowtime
            \mid \state
        )
        -
        u_{\flowtime}(
            \flowstate_{\flowtime}
            \mid\achunk
        )
        \right\|_2^2
    \right].
    \label{eq:cfm_objective_main}
\end{split}
\end{equation}
Under squared loss, the regression target is the marginal velocity field
\begin{equation}
    u^\star_{\flowtime}(\vy \mid \state)
    :=
    \E\!\left[
        u_{\flowtime}(\flowstate_{\flowtime}\mid\achunk)
        \,\middle|\,
        \flowstate_{\flowtime} = \vy, \state
    \right].
    \label{eq:cfm_population_field}
\end{equation}

\noindent At test-time, integrating $\dot{\flowstate}_\flowtime = \flowfield(\flowstate_\flowtime, \flowtime \mid \state)$ from $\flowstate_0\sim p_0$ generates an action chunk.

\textbf{Bilinear Transduction (BT).}
\label{sec:prelim_blt}
To extrapolate, BT \citep{netanyahu2023learning} selects a training example $\anc{\state} \in \DemoData$ as an anchor for a query condition $\state$, defines the residual $\reldelta:=\state-\anc{\state}$, and predicts the target through a low-rank bilinear interaction between the anchor and residual.
Specifically, it reparameterizes a regression target $h^\star(\state)\in\R^{\chunkdim}$ as $\widetilde h^\star(\reldelta,\anc{\state}):=h^\star(\anc{\state}+\reldelta)$ and models
\begin{equation}
    \widetilde h_\theta(\reldelta,\anc{\state})
    =
    \left[
        \left\langle
            f_{\theta,k}(\reldelta),
            g_{\theta,k}(\anc{\state})
        \right\rangle_{\rank}
    \right]_{k=1}^{\chunkdim}
    \qquad
    f_{\theta,k}(\reldelta),\,
    g_{\theta,k}(\anc{\state})
    \in \mathbb{R}^{\rank},
\end{equation}
where $\langle\cdot,\cdot\rangle_{\rank}$ denotes an inner product over a low-dimensional rank-$\rank$ latent space.
Together with the anchor-selection rule, this predictor induces training and test risks $\mathcal R_{\Train}$ and $\mathcal R_{\Test}$, measured by
the squared prediction error $\|\widetilde h_\theta-\widetilde h^\star\|_2^2$.
Under the assumptions in Appx. C.1.2 and sufficiently small training error, \brp{} admits the following extrapolation bound:
\begin{equation}
    \mathcal R_{\Test}(\widetilde h_\theta)
    \le
    C_{\mathrm{BT}}\,
    \mathcal R_{\Train}(\widetilde h_\theta),
    \label{eq:bt_risk_thm}
\end{equation}
where $C_{\mathrm{BT}}$ is a problem-dependent factor.
We highlight two of these assumptions for practical relevance:
\structurereq{} the anchor must be selected, and the residual defined, so that the resulting anchor--residual-to-action prediction problem has a well-conditioned low-rank bilinear structure.
\coveragereq{} a test-time anchor should yield a residual covered by the training residual distribution.

\emph{Limitations.} 
BT addresses \coveragereq{} through test-time anchor filtering, but largely
takes \structurereq{} for granted.
Consequently, prior demonstrations are limited to specific tasks in low-dimensional state spaces, where residuals are meaningful and anchor selection rules can be hand-designed.
Real-world visuomotor imitation introduces two hurdles.
First, inducing \structurereq{} requires finding both a representation with meaningful residuals and a retrieval policy that yields a simple prediction problem.
Second, expert actions may be multimodal, whereas BT predicts deterministically.
\ours{} addresses both challenges in the next section.

\section{Method}
\label{sec:method}

\begin{figure*}[t]
    \centering
    \includegraphics[width=0.96\textwidth]{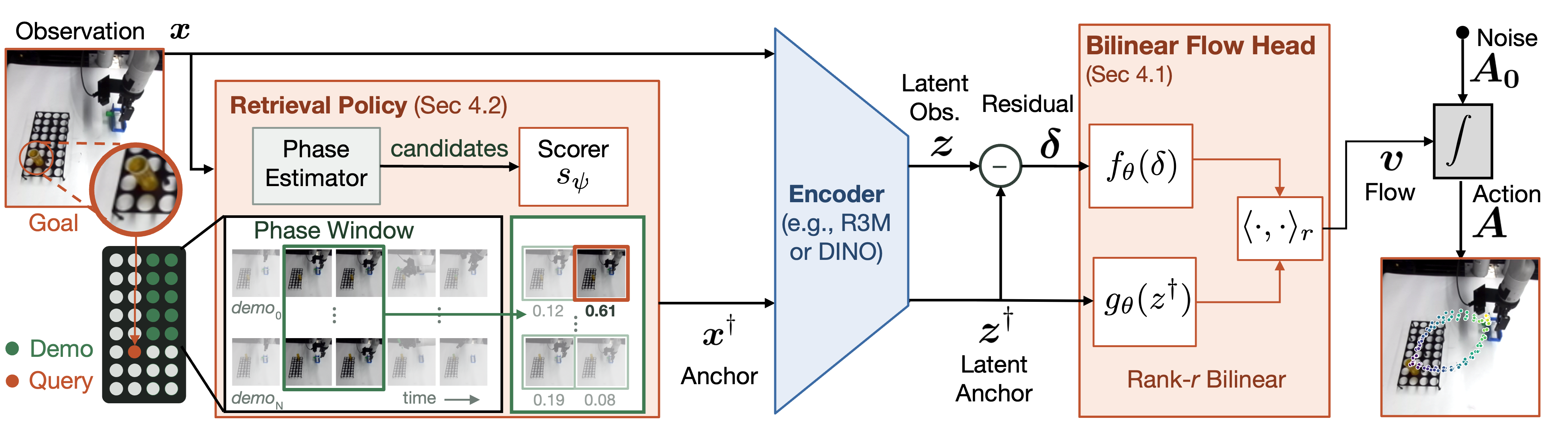}
    \caption{
        \textbf{Overview of \ours.}
        Given an unseen observation--goal condition $\state$, \ours retrieves an anchor $\anc{\state}$ and predicts from the resulting anchor--residual pair.
        The novel retrieval policy (\Cref{sec:learning_bfp_from_demo}) uses a phase estimator to form a phase-compatible candidate set and a learned scorer $s_\psi$, trained from held-out BFP loss, to select $\anc{\state}$.
        A frozen visual encoder computes $\querylatent=\obsenc(\state)$ and $\anchorlatent=\obsenc(\anc{\state})$, with residual $\reldelta = \querylatent-\anchorlatent$.
        The novel bilinear conditional-flow predictor (Sec.~4.1) predicts the CFM velocity through the low rank-$r$ interaction; flow integration then generates the action chunk.
    }
    \label{fig:bfp_overview}
    \vspace{-2mm}
\end{figure*}

\textbf{Overview.} \Cref{fig:bfp_overview} illustrates our proposed \textit{Bilinear Flow Policy} (\ours), a generative goal-conditioned visuomotor policy that extrapolates to unseen goals while preserving behavior multimodality.
\ours addresses the two challenges through a pretrained representation, learned anchor retrieval, and a bilinear conditional flow model.
Formally, given a query $\state = (\obs, \goal)$, \ours{} retrieves an anchor $\anc{\state}$ from the demonstrations, encodes both conditions with a latent representation $\obsenc$, then forms
\begin{equation}
    \querylatent=\phi(\state),
    \qquad
    \anchorlatent=\phi(\anc{\state}),
    \qquad
    \reldelta=\querylatent-\anchorlatent.
    \label{eq:bfp_coordinates}
\end{equation}
% somehow this sentence below does not connect well to the previous "its intuition ---" and "formally,..."
The goal is to choose a representation and anchor policy so that the retrieved anchor and its residual make the required action at query \emph{straightforward to predict}; concretely, we seek an induced anchor--residual--action prediction problem with low-rank bilinear structure (i.e., \structurereq{}).
% Intuition is to retrieve a training example whose residual to the current query makes the required action .
\Cref{sec:bfp_predictor} first establishes our main distributional extrapolation theorem, assuming this structure holds in the CFM marginal velocity field.
\Cref{sec:learning_bfp_from_demo} then constructs $\obsenc$ and $\anchorpolicy$ to make this structure plausible in visuomotor demonstrations.

\subsection{Bilinear Conditional Flow Head}
\label{sec:bfp_predictor}
We first isolate the predictor question.
Given a representation $\obsenc$ and retrieval policy $\anchorpolicy$,  how can we use bilinear structure to achieve \emph{distributional extrapolation}?

The representation $\obsenc$ and anchor policy $\anchorpolicy$ induce a conditional action distribution $\expertdist(\cdot\mid\reldelta,\anchorlatent)$ from the expert conditional $q^\star(\cdot\mid\state)$ using the anchor--residual coordinates in \Cref{eq:bfp_coordinates}.
Together with the conditional probability path $p_\flowtime(\flowstate_\flowtime\mid\achunk)$, this defines the corresponding marginal
velocity field $u^\star_\flowtime(\flowstate_\flowtime\mid\reldelta,\anchorlatent)$, whose flow generates this induced action distribution.

\ours{} approximates this field with the coordinate-wise bilinear model
\begin{equation}
    v_{\theta,k}
    (A_\tau,\tau\mid\reldelta,\anchorlatent)
    =
    \left\langle
        f_{\theta,k}(\reldelta;A_\tau,\tau),
        g_{\theta,k}(\anchorlatent;A_\tau,\tau)
    \right\rangle_r,
    \qquad k\in[d_A].
    \label{eq:bfp_bilinear_flow}
\end{equation}
For each fixed flow context $c=(A_\tau,\tau)$ and action coordinate $k$, this assumes that the marginal velocity field is approximated by a rank-$r$ bilinear function of the residual and anchor.
Thus, \structurereq{} is imposed on the velocity field rather than directly on the action.

We train this predictor with the CFM objective.
For a demonstration $(\state,A)$ and retrieved anchor $\anc{\state}$,
\begin{equation}
    \ell_{\mathrm{BFP}}
    (\theta;\state,\anc{\state},A)
    :=
    \mathbb E_{\substack{
        \tau\sim\mathcal U[0,1]\\
        A_\tau\sim p_\tau(\cdot\mid A)}}
    \left[
        \left\|
        v_\theta
        (A_\tau,\tau\mid\reldelta,\anchorlatent)
        -
        u_\tau(A_\tau\mid A)
        \right\|_2^2
    \right].
    \label{eq:bfp_cfm_loss}
\end{equation}
We also define $\flowrisk{s}(\theta)
    :=
    \mathbb E_s
    \left[
        \left\|
        v_\theta-u^\star
        \right\|_2^2
    \right]$ for $s\in[\Train, \Test]$.
We now present our main result:
\begin{thm}[Distributional extrapolation of \ours, informal]
\label{thm:bfp_distributional_extrapolation}

Suppose that the representation $\obsenc$ satisfies \Cref{ass:app_policy_sufficiency} in \Cref{app:induced_pair_distribution}, the marginal velocity field satisfies \structurereq{}, and test-time anchor selection satisfies \coveragereq{}.
Under the remaining technical conditions detailed in \Cref{app:main_distributional_theorem}, sufficiently small context-wise training flow error implies
\begin{equation}
    \mathcal R_{\Test}^{\mathrm{dist}}(\pi_\theta)
    \le
    C_{\mathrm{ODE}}\,
    \flowrisk{\Test}(\theta)
    \le
    C_{\mathrm{ODE}} C_{\mathrm{BFP}}\,
    \flowrisk{\Train}(\theta).
    \label{eq:bfp_distributional_extrapolation}
\end{equation}
\end{thm}
\begin{proof}
Deferred to \Cref{app:main_distributional_theorem}.
\end{proof}
% \textit{Proof sketch.}
% At each fixed flow context and output coordinate,
% \Cref{eq:bfp_bilinear_flow} is the bilinear regression problem covered
% by \Cref{eq:bt_risk_thm}.
% Under \structurereq{} and \coveragereq{}, applying that result context-wise transfers training velocity-field error to covered unseen anchor--residual combinations.
% ODE stability then converts the resulting test-time field error into the Wasserstein error of the generated action distribution.
% The complete proof is given in
% \Cref{app:bfp_distributional_extrapolation}.

\textbf{Practical interpretation.}
Under the theorem's assumptions, in-distribution flow error controls the error in the full generated action distribution at unseen goals.
The theorem motivates the practical diagnostics for assessing extrapolation performance based on \structurereq{} and \coveragereq{}, which we evaluate in \Cref{sec:exp_diagnostics}.

\subsection{Learning \ours{} from Demonstrations}

We now instantiate the representation and anchor policies left abstract in
\Cref{sec:bfp_predictor}: training-time retrieval aims to induce a simple low-rank relation (\structurereq{}), while test-time retrieval aims to keep the required residual covered by training (\coveragereq{}).

\label{sec:learning_bfp_from_demo}

\textbf{Representation.}
Let $\state=(I^{1:V},\vp,\goal)$ contain images, proprioception, and the goal.
\ours uses
\begin{equation}
    \obsenc(\state)
    =\left[E_{\mathrm{vis}}(I^1);\ldots;E_{\mathrm{vis}}(I^V);\vp;\goal\right],
    \label{eq:bfp_encoder}
\end{equation}
where $E_{\mathrm{vis}}$ is a frozen pretrained visual representation such as R3M \citep{nair2022r3m} or DINOv3 \citep{simeoni2025dinov3}.
Applying the same encoder to query and anchor provides a shared compact space for defining
$\reldelta=\obsenc(\state) -\obsenc(\anc{\state})$.
As it is pretrained independent to demonstrations, they provide a common train--test coordinate system.
% Applying the same encoder to query and anchor provides a compact shared space in which the residual $\reldelta=\obsenc(\state)-\obsenc(\anc{\state})$ can be meaningfully defined.
% Because the visual coordinates are pretrained independently of our demonstration set and then frozen, they provide a common train--test coordinate system, crucial to our generalization.

\textbf{Learning train-time anchor retrieval.}
Anchors at similar task phases are more likely to induce simple residuals, but phase compatibility alone does not guarantee bilinear predictability.
For each query $\state$, a coarse phase estimator therefore forms a candidate window $\mathcal C(\state) = \{\anc{\state}_m\}_{m=1}^{K}$ from demonstrations at similar task progress.
We use a simple hand-designed estimator, although learned temporal-alignment methods could be used instead \citep{dwibedi2019temporal,memmel2025strap}.

For each candidate, let
$\anchorlatent_m=\obsenc(\anc{\state}_m)$ and
$\reldelta_m=\obsenc(\state)-\anchorlatent_m$.
A learned scorer defines
\begin{equation}
    p_\psi(m\mid\state,\mathcal C(\state))
    =
    \frac{
        \exp\!\left(s_\psi(\reldelta_m,\anchorlatent_m)/\tau_\rho\right)
    }{
        \sum_{n=1}^{K}
        \exp\!\left(s_\psi(\reldelta_n,\anchorlatent_n)/\tau_\rho\right)
    }.
    \label{eq:anchor_policy}
\end{equation}

We alternate predictor and scorer updates so that each provides the learning signal for the other.
With the scorer fixed, the predictor learns primarily from anchors assigned high probability:
\begin{equation}
    \mathcal L_{\mathrm{pred}}(\theta)
    =
    \E_{(\state,\achunk)\sim\DemoData}
    \left[
        \sum_{m=1}^{K}
        \sg[p_\psi(m\mid\state,\mathcal C(\state))]
        \ell_{\mathrm{BFP}}
        (\theta;\state,\anc{\state}_m,\achunk)
    \right].
    \label{eq:weighted_predictor_loss}
\end{equation}
Conversely, with the predictor fixed, the scorer favors anchors with low held-out BFP loss:
\begin{equation}
    \mathcal L_{\mathrm{ret}}(\psi)
    =
    \E_{(\state,\achunk)\sim\DemoData_{\mathrm{val}}}
    \left[
        \sum_{m=1}^{K}
        p_\psi(m\mid\state,\mathcal C(\state))
        \sg[\widetilde\ell_m]
    \right],
    \label{eq:retrieval_objective}
\end{equation}
where $\widetilde\ell_m$ is the normalized held-out BFP loss.
Thus, retrieval favors anchor--residual pairs that the shared rank-$r$ predictor can explain well.
Algorithm and details are given in  \Cref{app:anchor_policy}.

\textbf{Test-time anchor selection.}
At test time, motivated by \coveragereq{}, we prioritize goal-residual coverage when selecting anchors.
Following BT, we first select demonstrations whose goal residual is closest to the empirical training goal-delta set $\widehat{\mathcal G}_{\Delta}$,
\begin{equation}
    j^\star
    \in
    \argmin_j
    \min_{\Delta g\in\widehat{\mathcal G}_{\Delta}}
    \left\|(g_q-g^{(j)})-\Delta g\right\|_2^2 .
    \label{eq:test_anchor_goal_delta}
\end{equation}
We apply the learned scorer from \Cref{eq:anchor_policy} over
$\mathcal C_{\Test}(\state)$ and sample
$
M\sim\mathrm{Cat}\!\left(p_\psi(\cdot\mid\state,\mathcal C_{\Test}(\state))\right).
$
With the sampled anchor fixed, we form
$\anchorlatent_M=\obsenc(\anc{\state}_M)$ and
$\reldelta_M=\obsenc(\state)-\anchorlatent_M$, sample
$\flowstate_0\sim p_0$, and integrate
\begin{equation}
    \dot{\flowstate}_{\flowtime}
    =
    v_\theta(
        \flowstate_\flowtime,\flowtime
        \mid\reldelta_M,\anchorlatent_M
    ),
    \qquad \flowtime\in[0,1].
    \label{eq:test_flow_generation}
\end{equation}
The terminal state $\flowstate_1$ is the generated action chunk.
Algorithm and details are given in  \Cref{app:anchor_policy}.

\section{Experiments}
\label{sec:experiments}

We organize our experiments around three questions that follow the claims of the method.
First (\Cref{sec:exp_bfp_solves_dist_extrapolation}), can \ours{} solve distributional extrapolation given oracle representation and retrieval? 
Second (\Cref{sec:exp_main_results}), does \ours{} generalize to unseen goals and transfer to real-world deployment?
Third  (\Cref{sec:exp_diagnostics}), can theory-motivated signals identify which trained policies and unseen queries are likely to extrapolate reliably without performing OOD rollouts?
We additionally ablate core components of \ours{} (\Cref{sec:ablations}).
Full results are reported in \Cref{app:hardware,app:experiment_details}.

\subsection{Can BFP solve distributional extrapolation?}
\label{sec:exp_bfp_solves_dist_extrapolation}
\begin{wraptable}{r}{0.50\textwidth}
    \vspace{-0.5em}
    \centering
    \scriptsize
    \setlength{\tabcolsep}{4.2pt}
    \renewcommand{\arraystretch}{1.05}
    \begin{tabular}{@{}lccc@{}}
        \toprule
        Method
        & $\operatorname{NW}_2$ ID $\downarrow$
        & $\operatorname{NW}_2$ OOD $\downarrow$
        & Balance OOD $\downarrow$ \\
        \midrule
        \emph{Oracle} & $0.078$ & $0.086$ & $0.019$ \\
        \flow{}       & $\mathbf{0.221}$ & $0.506$ & $0.160$ \\
        \brp{}        & $0.726$ & $0.763$ & $0.500$ \\
        \ours{}       & $0.224$ & $\mathbf{0.321}$ & $\mathbf{0.103}$ \\
        \bottomrule
    \end{tabular}
    \caption{\textbf{\ours{} distributionally extrapolates.} Results averaged over three problems and three seeds. Lower is better for all metrics.  \emph{Oracle} samples directly from the expert conditional and gives the Monte Carlo floor.}
    \label{tab:toy_main}
    \vspace{-1.5em}
\end{wraptable}
\textbf{Setting.} We first isolate the predictor from the representation and anchor-retrieval challenges studied later.
We consider three one-dimensional analytical problems that instantiate the bimodal setting in \Cref{fig:bfp_regression}, where inputs are already compact and random anchor retrieval suffices.
In each problem, the expert conditional is an equal mixture of two analytical modes.
This setting lets us quantify directly whether the predictor can extrapolate
both modes beyond the training support.
Each method is trained on $20{,}000$ input--output pairs with three seeds and evaluated from $500$ samples per input. 
$\operatorname{NW}_{2}$ measures overall mismatch to the expert conditional, normalized by the distance between its modes, while balance error, measures deviation from the prescribed $50/50$ mode mass;
formal definitions are given in \Cref{app:toy}.

\textbf{Hypothesis.}
We expect \flow{} to preserve multimodality but extrapolate it unreliably, and deterministic \brp{} to extrapolate structurally but fail to represent both modes.

\textbf{Result.}
\ours{} and \flow{} show similar ID $\operatorname{NW}_{2}$, but \ours{} reduces OOD $\operatorname{NW}_{2}$ from $0.506$ to $0.321$ ($36.6\%$) and improves mode balance from $0.160$ to $0.103$.
\brp{} has the highest OOD distribution error and the maximal balance error of $0.500$, reflecting its inability to represent both modes with deterministic regression.
Thus, when representation and retrieval are controlled, only \ours{} achieves the lowest average OOD distributional error.
Per-problem results are provided in \Cref{app:toy}.

\subsection{Does BFP Improve Unseen Goal Generalization?}
\label{sec:exp_main_results}

\textbf{Setting.}
We evaluate seven simulation settings spanning the five manipulation tasks shown in \Cref{fig:task_suite}, together with two real-world tasks in \Cref{fig:hw_task_suite}.
% The \texttt{D0} settings evaluate held-out goals while keeping the object initial condition fixed.
% The \texttt{D1} settings tests goal shift with a larger initial condition set.
% When no suffix is provided, we refer to \texttt{D1}.
% otherwise unspecified, we assume D1 setting
Exact demonstration, ID, and OOD supports are provided in \Cref{app:experiment_details,app:hardware}.
We report results under state-domain and visuomotor observations.
State uses simulator states.
Visuomotor uses frozen R3M features \citep{nair2022r3m}, proprioception, and the explicit goal.
We study other pretrained visual features (e.g., DINO-v3) in \Cref{sec:ablations}.

We compare the following:
\flow{} \citep{lipman_flow_2024}, a standard generative GCIL policy, and four extrapolation-targeted baselines.
\edp{} \citep{wang2024equivariant} represents prespecified extrapolation by encoding a known geometric equivariance into a generative policy.
\darp{} \citep{pfeifer2026difference} retrieves the $k$ nearest anchors and predicts from each anchor and its residual using a local, non-bilinear model.
\brp{} \citep{netanyahu2023learning} uses a hand-crafted time-synchronized anchor policy and deterministic low-rank bilinear regression; for visuomotor tasks, we equip \brp{} with the same frozen pretrained visual representation as \ours{}.
\cta{} \citep{kim2026compositional} learns a task-endogenous analogy and uses it as the residual in hierarchical bilinear transduction; a high-level policy predicts an analogy, and a low-level policy maps it to primitive actions.
See \Cref{app:sec:baseline_study} for details.

For each pair of task and observation space in simulation, every training seed is evaluated on fixed 100 ID and 100 OOD initial conditions.
We train 3 training seeds and report the mean and standard error of the mean (SEM) across the seeds.
We use statistical hypothesis testing for policy comparison with 5\% family-wise error rate (FWER) control within each task \citep{snyder2025your, snyder2026beyond}.

\textbf{Simulation Result.} 
% We now evaluate whether the full-stack \ours{} generalizes to unseen goals in closed-loop visuomotor manipulation, following the evaluation protocol in \Cref{sec:exp_setup}.
As shown in \Cref{fig:ood_reward_violin_panel_sim}, \ours{} achieves the highest OOD success on all seven settings, averaging $79\%$, compared
with $58\%$ for \brp{} and $30\%$ for \flow{}.
The normalized dense-reward analysis shows the same ordering, with \ours{} forming the top FWER-controlled statistical group as indicated by the Compact Letter Display (CLD)~\citep{piepho2004algorithm} shown on top of each sub-plot. Columns not sharing the same alphabet are significantly different with 5\% FWER, verifying that \ours{} outperforms all the baselines.
ID success remains high at $97\%$, indicating that the OOD gains do not come from sacrificing performance within demonstration support.
Same trend holds for state-space result.
Detailed analysis and per-task results are provided in \Cref{app:experiment_details}.

The baselines clarify where the gain comes from.
\ours{} outperforms \edp{} despite its carefully tuned (see \Cref{app:sec:baseline_study}), task-specific geometric rule, showing the benefit of learning the extrapolation relation from demonstrations rather than prescribing it \textit{a priori}.
Its gains over \darp{} and \cta{} show that local retrieval or learned analogy alone does not match the proposed structured predictor, while the $21\%$ gain over \brp{} indicates that learned anchor retrieval and generative bilinear-flow prediction add to the fixed-anchor deterministic bilinear recipe.

\textbf{Real-world Result.}
We evaluate generalization to unseen goals on a real-world robot in a pick-and-place task (\texttt{PnPMug}) and a pipette insertion task (\texttt{InsertPipette}) the latter requiring within $2 \mathrm{cm}$ precision.
As shown in \Cref{fig:ood_reward_violin_panel_real}, empirical results in this setting show the same trends observed in simulation.
\ours{} achieves $50\%$ average OOD success rate across the two tasks, compared with $18\%$ for \flow{} and $39\%$ for \brp{}.
These results suggest that BFP's generalization to unseen goals through extrapolation holds under real-world conditions including perception and goal estimation error.
We report experimental setup and discuss results further in \Cref{app:hardware}.

\setcounter{figure}{3}
\begin{figure}[t]
    \centering
    \begin{subfigure}[t]{0.49\linewidth}
        \centering
        \vspace{0pt}
        \includegraphics[width=\linewidth]
            {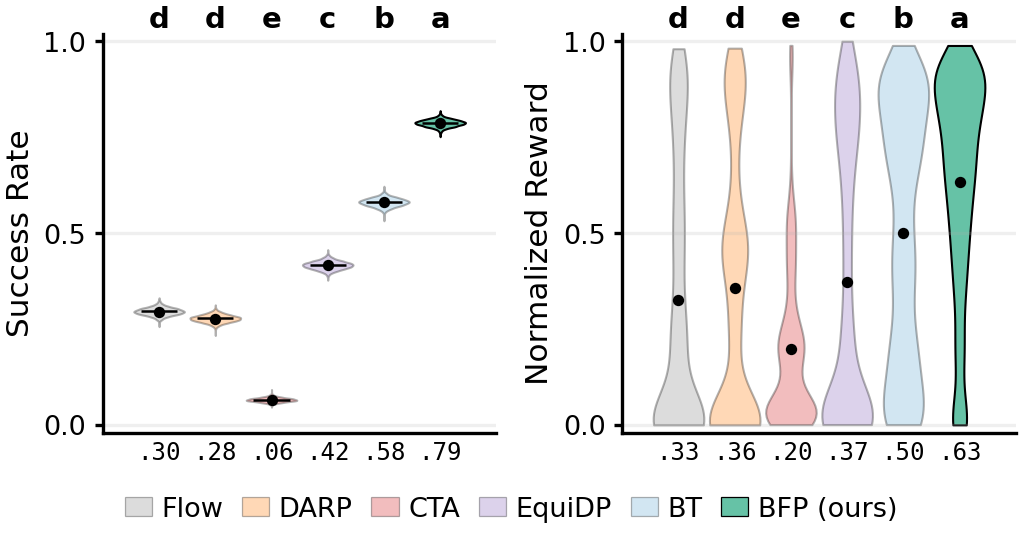}
        \caption{OOD performance in simulation.}
        \label{fig:ood_reward_violin_panel_sim}
    \end{subfigure}\hfill%
    \begin{subfigure}[t]{0.49\linewidth}
        \centering
        \vspace{0pt}
        \includegraphics[width=\linewidth]
            {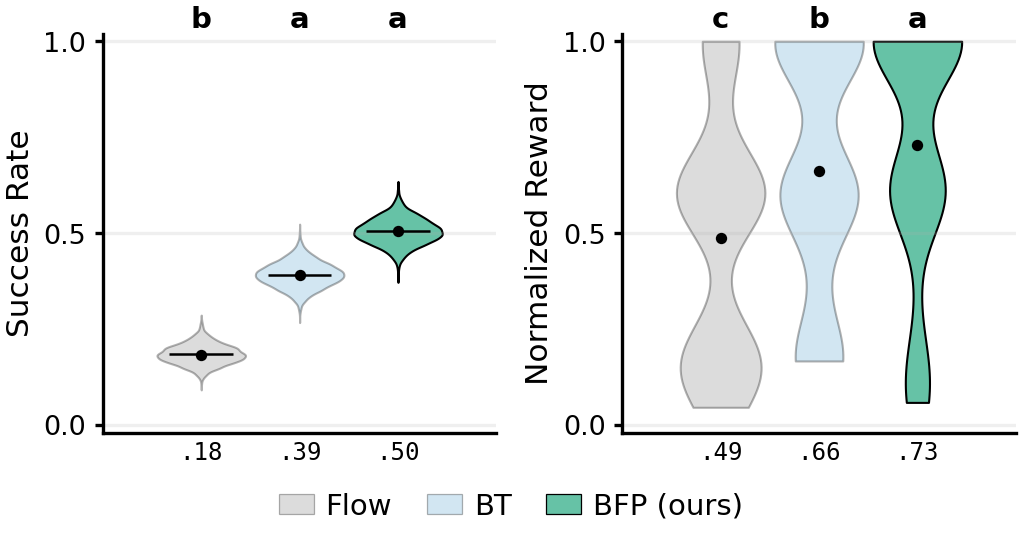}
        \caption{OOD performance in real-world.}
        \label{fig:ood_reward_violin_panel_real}
    \end{subfigure}\hfill%

    \caption{
        \textbf{\ours{} generalizes in visuomotor manipulation.}
        \textbf{(a)} Success rate / Normalized reward averaged across seven simulation settings
        \textbf{(b)} Same across hardware tasks.
        Success rate violins show Bayesian posterior estimates under a uniform Beta prior. Normalized reward violins illustrate empirical distributions of the results. Dots represent sample means and horizontal lines show posterior means (success-rate only).
        Letters on top of the plots show statistical comparison, where columns not sharing the same alphabet are significantly different under 5\% FWER control.
    }
    \label{fig:visuomotor_ood_summary}
    \vspace{-1em}
\end{figure}

\newlength{\hwpanelheight}
\setlength{\hwpanelheight}{3.4cm}

\subsection{In-Distribution Diagnostics of Extrapolation}
\label{sec:exp_diagnostics}

\textbf{Setting.} 
We ask whether theory-motivated diagnostics can help assess extrapolation before closed-loop OOD rollout.
Our bound motivates two empirical proxies: held-out ID CFM loss to represent \structurereq{} and a residual-coverage score for \coveragereq{}.
We measure coverage using the realism metric \citep{kynkaanniemi2019improved} against a bank of training residuals, with higher scores indicating better coverage.
We evaluate whether these proxies predict open-loop OOD CFM loss on expert observation--action pairs.
For checkpoint selection, we rank checkpoints within each task by held-out ID CFM loss and compare their OOD-error percentile ranks.
For query-level diagnosis, we group OOD queries within each checkpoint into residual-coverage percentile and compare their OOD-error percentile ranks.
Full evaluation details are provided in \Cref{app:extrapolation_diagnostics}.

% These yield two practical questions: can held-out ID flow error identify which trained checkpoint will extrapolate better, and, for a fixed checkpoint, can residual coverage identify which unseen goal-conditioned queries are better supported?
% For the first probe, we rank checkpoints within matched task--encoder--retrieval pools by held-out ID flow error and compare their normalized OOD flow-error ranks, where lower is better.
% For the second, we group OOD queries within each checkpoint by residual-coverage decile and compare their normalized flow error.
% Exact probe construction is deferred to \Cref{app:extrapolation_diagnostics}.

% Required package:
% \usepackage{subcaption}

\begin{figure}[t]
    \centering

    % Restore readable subfigure captions.
    % Remove: \captionsetup{font=footnotesize}
    \captionsetup[subfigure]{font=normalsize,skip=2pt}

    \begin{subfigure}[t]{0.48\linewidth}
        \centering
        \includegraphics[
            width=\linewidth
        ]{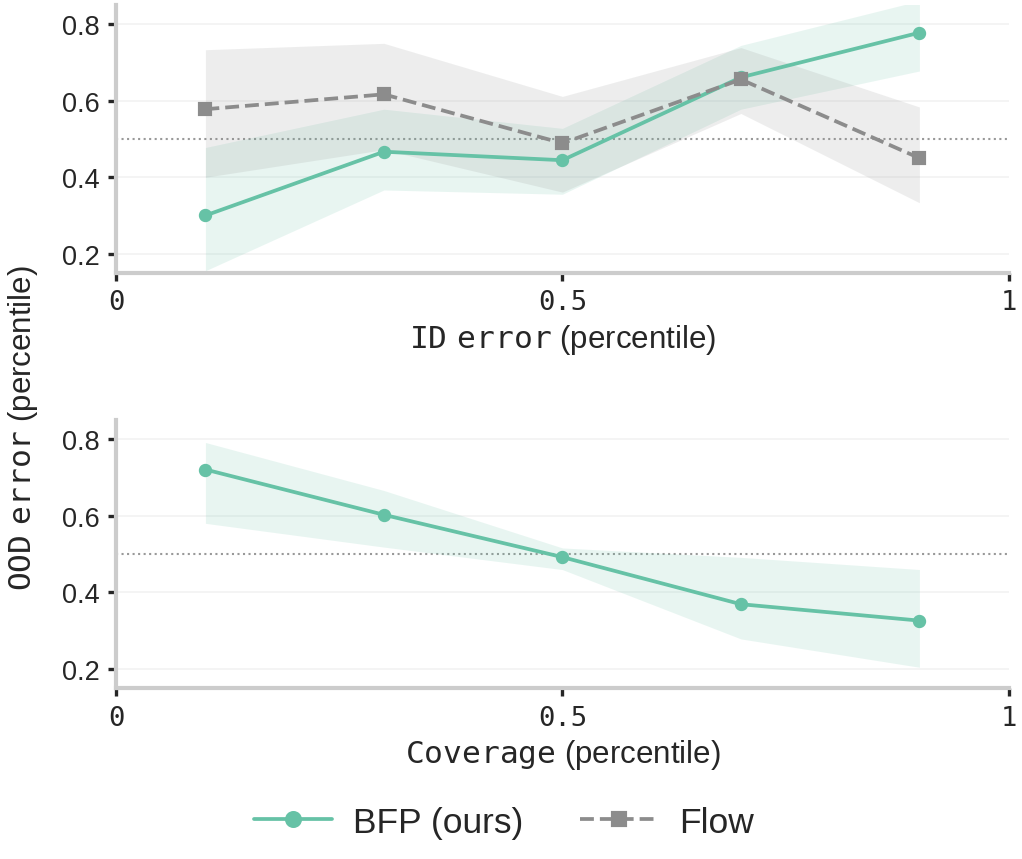}
        \caption{Diagnostics.}
        \label{fig:extrapolation_diagnostics}
    \end{subfigure}
    \hfill
    \begin{subfigure}[t]{0.48\linewidth}
        \centering
        \includegraphics[
            width=\linewidth
        ]{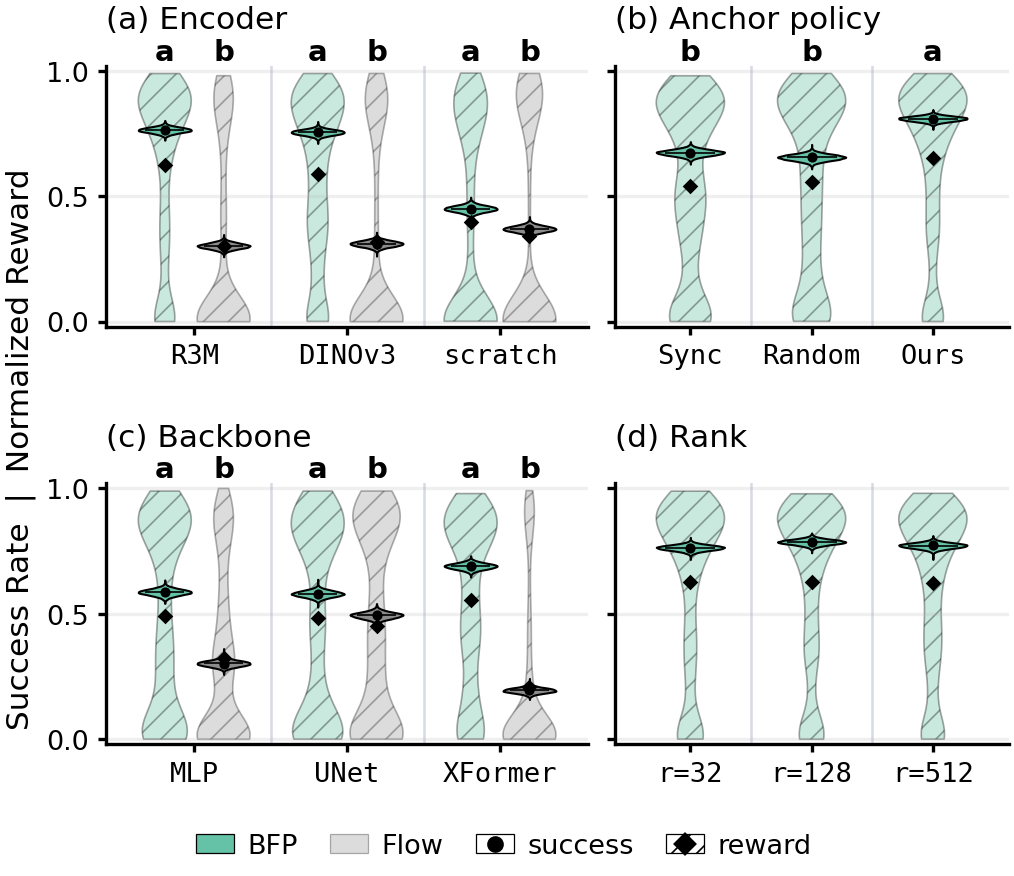}
        \caption{Ablations.}
        \label{fig:bfp_component_ablations}
    \end{subfigure}

    \caption{
        \textbf{Pre-rollout diagnostics and \ours{} ablations.}
        \textbf{(a)} Held-out ID flow error ranks OOD error for
        \ours{} but not for \flow{}; within a fixed \ours{}
        checkpoint, better residual coverage is associated with
        lower extrapolation error.
        \textbf{(b)} Ablations of the visual encoder, anchor policy, backbone, and bilinear rank. Frozen pretrained features and phase-conditioned retrieval both lead to better performance; 
        \ours{} can exploit stronger sequence backbones and performs similarly across the ranks.
    }
    \label{fig:diagnostics_and_ablations}
\end{figure}

\textbf{Result.} 
\flow{} shows the failure first: lower held-out ID CFM loss does not consistently correspond to lower OOD error.
For \ours{}, however, the checkpoints in the highest ID error quintile have a mean OOD-error rank approximately $48$ percentile points higher than those in the lowest quintile.
These results suggest that held-out ID CFM loss can guide checkpoint selection and hyperparameter search for \ours{} without OOD rollout.

Residual coverage provides a complementary query-level signal.
In the within-checkpoint analysis, queries in the lowest-coverage
quintile have higher OOD-error percentile ranks than those in the
highest-coverage quintile.
Residual coverage therefore serves as a query-level predictor of performance on unseen goals and motivates test-time retrieval that favors anchors inducing residuals well covered by training.
% These diagnostics concern open-loop flow error rather than closed-loop success, and the coverage result is correlational rather than causal; directly optimizing coverage-aware retrieval remains future work.

\subsection{Ablations}
\label{sec:ablations}
We ablate the visual representation, anchor retrieval, flow predictor backbone, and bilinear rank in \Cref{fig:bfp_component_ablations}.
A frozen pretrained visual representation and phase-conditioned retrieval are the most consequential components: \ours{} with a frozen R3M or DINOv3 encoder reaches both $76\%$ OOD success,
$+31\%$ over jointly training the encoder from scratch ($45\%$), and learned phase-conditioned retrieval reaches $81\%$ versus $67\%$ and $66\%$ for time-synchronized (as in BT) and random retrieval.

% \begin{wrapfigure}{r}{0.5\textwidth}
%     \centering
%     \includegraphics[width=\linewidth]{figures/fig_abl_all_r3mP_dual_overlay_ood_col.png}
%     \caption{}
%     \label{fig:ablations}
% \end{wrapfigure}

One concern is that \ours{}'s low-rank structure makes the model simpler and more prone to underfitting.
Perhaps this underfitting, rather than bilinear interaction itself, explains the extrapolation gain.
If so, a more expressive backbone would erase the gain by allowing the model to overfit. We observe the opposite: replacing the MLP with a sequence-aware Transformer raises \ours{} from $59\%$ to $69\%$ OOD success, while lowering \flow{} from $30\%$ to $20\%$, despite $94$--$99\%$ ID success across trained variants. 
These results suggest that \ours{}’s gains are not merely a consequence of an underpowered backbone: \ours{} can benefit from greater backbone expressiveness while retaining its bilinear anchor–residual structure.
Finally, performance is similar across the tested bilinear ranks.

\section{Conclusion}
\label{sec:conclusion}

We introduce \ours{}, a generative goal-conditioned visuomotor policy that extrapolates to unseen goals by learning generalizable observation-goal-action relations among demonstrated conditions.
\ours{} extends bilinear transduction to conditional flow matching and combines it with compact representations and learned anchor selection to generate multimodal action distributions beyond demonstration support.
Across five simulation manipulation tasks and two real-world tasks, \ours{} substantially improves OOD success over competing methods, while providing practical signals for checkpoint selection and for identifying unseen conditions where extrapolation is better supported.

\textbf{Limitations and Future Work.}
\ours{} currently assumes compact goal representations.
Extending \ours{} to image- or language-conditioned goals remains an important direction for future work.

\subsection*{AI use statement}

We used generative AI tools to assist with literature review, experiment analysis, implementation of baseline methods, and editing and organization of
manuscript text. 
AI-assisted literature review was verified against the original sources, baseline implementations were reviewed and tested by
the authors, and all AI-assisted writing was manually reviewed and revised. 
We take responsibility for the final content of this work, including text, claims, code, and artifacts produced with the aid of generative AI.

\subsection*{Reproducibility statement}
We are committed to making our results reproducible.
The main paper specifies the BFP architecture, and pseudocode details provided in the \Cref{app:anchor_policy}.
\Cref{app:hardware} documents the real-world task setups.
We release the implementation, experiment configurations, and supporting materials through our project website: \url{https://bilinearflow.netlify.app}.


\subsubsection*{Acknowledgments}
Toyota Research Institute (TRI) provided funds to assist the authors with their research; this article solely reflects the opinions and conclusions of its authors, and not TRI or any other Toyota entity.
This project was also partially supported by the AI Manufacturing Pilot Facility project under Georgia Artificial Intelligence in Manufacturing (Georgia AIM) from the U.S. Department of Commerce Economic Development Administration, Award 04-79-07808.

% \bibliography{iclr2027_conference}
\bibliography{references}  % .bib
\bibliographystyle{iclr2027_conference}

\clearpage
\appendix
% =====================================================================
% Bilinear Flow Policy: reorganized appendix draft (v9)
% This file is intended to be included after \appendices.
% Prose formatting convention: one complete sentence per source line.
% =====================================================================

\section*{Appendix}
 We include several appendices with additional details, derivations, and results:
  \begin{itemize}
      \item \Cref{app:hardware} describes the real-world robot setup, tasks, and additional hardware results.

      \item \Cref{app:experiment_details} details the simulation benchmark, including tasks, goal splits, evaluation protocol, and additional results.
 
      \item \Cref{app:theory} develops the theory: \Cref{app:theory:background} reviews the prior low-rank bilinear extrapolation result, \Cref{app:bfp_problem_setup}
  shows that BFP conditional flow matching induces exactly such a bivariate regression problem, and \Cref{app:bfp_theory} proves the distributional extrapolation guarantee and its
  multi-anchor corollary.
      
      \item \Cref{app:anchor_policy} gives the train- and test-time anchor-policy instantiations and their training and inference algorithms. 

      \item \Cref{app:toy} reports a controlled numerical study of distributional extrapolation.

      \item \Cref{app:sec:baseline_study} describes how each baseline is instantiated, reports encoder and baseline sensitivity studies, and evaluates the pre-rollout extrapolation diagnostics.
  \end{itemize}

%%%%%%%%%%%%%%%%%%%%%%%%%%%%%%%%%%%%
% Hardware
%%%%%%%%%%%%%%%%%%%%%%%%%%%%%%%%%%%%
\clearpage
\section{Real-world Experiment Details \& Analysis}
\label{app:hardware}

\begin{figure}[h]
    \centering
    \includegraphics[width=\linewidth]{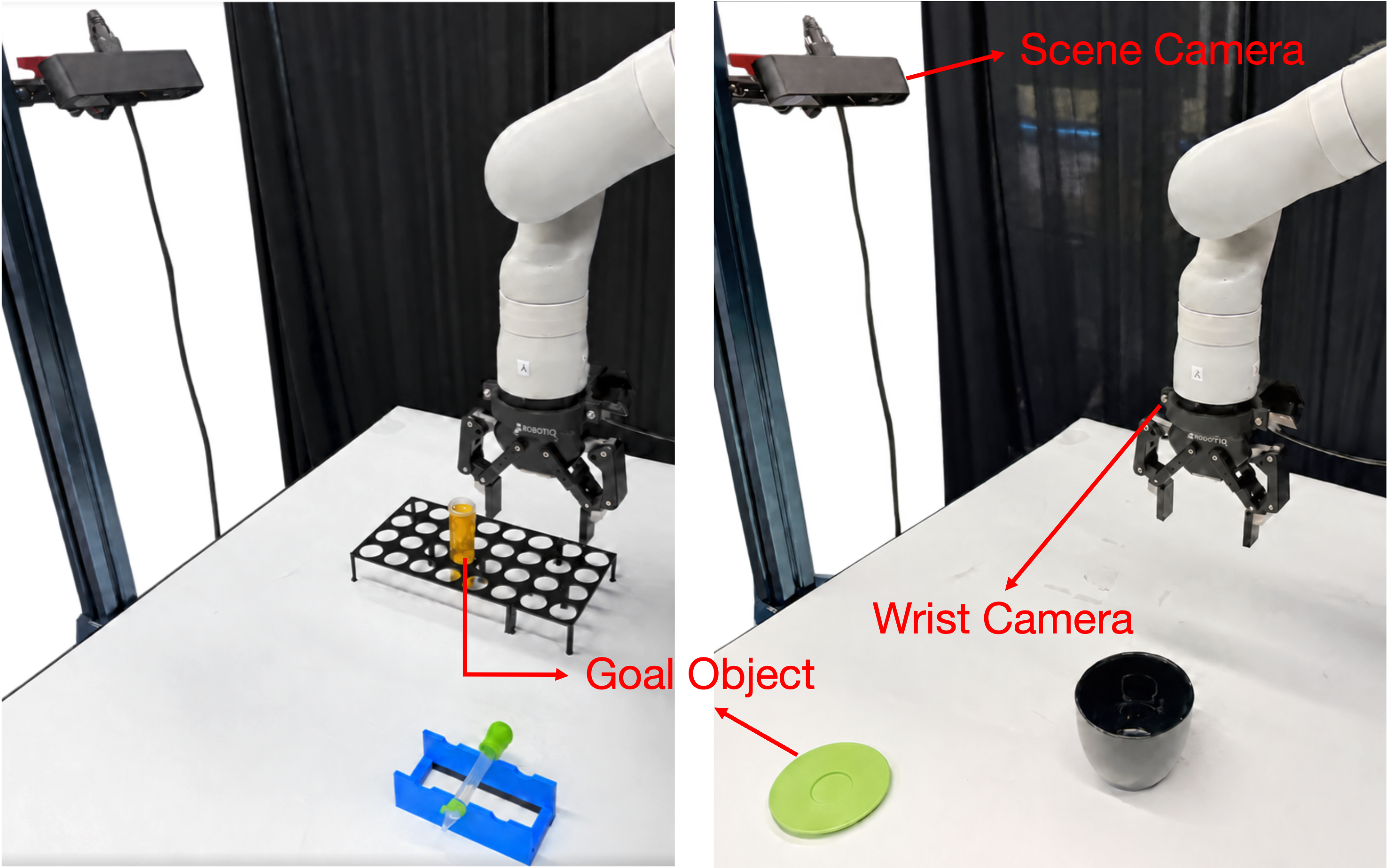}
    \caption{ \textbf{Task Setup.} \textbf{(Left)} \texttt{InsertPipette} \textbf{(Right)} \texttt{PnPMug}}
    \label{fig:hw_task_suite}
\end{figure}

\textbf{Hardware Setup.}
We conduct all real-world experiments on a 7-DoF Kinova Gen3 arm with a Robotiq 2F-85 parallel-jaw gripper; the full setup is shown in \Cref{fig:hw_task_suite}.
We use two RGB-D cameras: a ZED~2i as a scene camera, and an OAK-D Pro as a wrist camera.
Depth information is only used to extract the goal and not used for policy learning.
The controller operates at $20$\,Hz.
Inference runs on a single NVIDIA RTX 4090 Laptop GPU with 16\,GB of memory.
We use a Meta Quest 3 VR headset to collect teleoperated demonstrations.

The explicit task goal is extracted from the ZED RGB-D stream and represented in the robot base frame using hand--eye calibration.
SAM~3 \citep{carion2026sam} segments the target plate and FoundationPose \citep{wen2024foundationpose} registers and tracks its 6-DoF pose.

\textbf{Task and Distribution Shift.}
We evaluate two goal-conditioned manipulation tasks, shown in \Cref{fig:hw_task_suite}.
In \texttt{PnPMug}, the robot picks up a black mug and places it on a green plate.
We use the 3D position of the plate as the policy goal.
We collect $127$ teleoperated demonstrations for policy training with the plate restricted to $y>0$ and $x<0.7$.
For OOD evaluation, the plate is placed exclusively in the held-out region, such that no goal lies within the bounding box of the demonstrated goals.
We evaluate $50$ OOD initial conditions (ICs) with 2 trials per IC and $50$ ID initial conditions with 1 trial per IC, with a maximum episode length of $30$\,s ($1.5\times$ average human demo duration).

\texttt{InsertPipette} evaluates more precise goal extrapolation.
We use the discrete coordinate of test tube as the policy goal.
The robot picks up a pipette from a cradle and inserts it into a test tube positioned in one of the holes of a $4\times8$ rack, whose hole pitch is $36$\,mm.
The demonstration set for policy training contains $133$ trajectories spanning $10$ training holes in a $2\times5$ region.
OOD evaluation uses the remaining $22$ held-out holes, with five trials per hole. Each policy is evaluated in-distribution with 3 trials per hole. Maximum episode length is set to $60$\,s ($1.5\times$ average human demo duration).

\textbf{Policy Training.}
All policies use the same frozen, pretrained R3M ResNet-34 visual encoder \citep{nair2022r3m}.
The policy input consists of both camera views, a $10$-dimensional proprioceptive state containing the end-effector position, 6-D rotation representation, and gripper state, together with the explicit task goal.
Policies output absolute end-effector poses and gripper commands as $16$-step action chunks using a 1-D U-Net backbone.
We use prediction horizon and execution horizon $(T_p, T_a) = (16, 8)$.
We train all policies with a batch size of $32$ and a learning rate of $10^{-4}$.
Flow-based policies use $10$ integration steps at inference.
All hardware results use a single trained seed per method.

\textbf{Baselines and Evaluation.}
We compare \ours{} with two baselines.
\flow{} \citep{lipman_flow_2024} is a goal-conditioned flow-matching policy that uses the same visual encoder and action-generation backbone as \ours{}, but removes the anchor--residual decomposition and bilinear interaction.
\brp{} uses bilinear action regression, but replaces our learned candidate-set anchor scorer with synchronized, phase-based anchor selection.
We use temporal ensemble technique \citep{zhao2023learning}.

A human evaluator determines binary task success.
We additionally report a normalized dense reward in $[0,1]$ based on sequential task progress.
For \texttt{PnPMug}, the progress sequence is
\emph{reach $\rightarrow$ grasp $\rightarrow$ lift $\rightarrow$ release $\rightarrow$ place on plate $\rightarrow$ retreat}.
For \texttt{InsertPipette}, it is
\emph{reach $\rightarrow$ grasp-and-lift $\rightarrow$ align $\rightarrow$ insert $\rightarrow$ retreat},
where alignment requires the pipette to come within $2$\,cm of the target hole.
A stage contributes to the reward only after all preceding stages have been completed, while partial progress toward the next stage is credited fractionally.

\textbf{Results.}
\Cref{tab:hw_results} reports the complete task-level OOD and ID results.
We report both success rate and normalized dense reward; the corresponding task-averaged comparison is discussed in the main text.
\Cref{fig:spoid_ic_overlay_results,fig:pnp_ic_overlay_results} visualize the full per-task spatial distributions of successes and failures for both tasks.
\Cref{fig:hardware_violin} visualizes the per-task distribution of dense reward for both tasks.
Finally, \Cref{fig:hw_failure_mode_strips} shows time-lapse strips with progress phase labels on successful \ours{} rollouts for both tasks.
This figure also shows failed \brp{} and \flow{} rollouts which exhibit the common failure modes of these baselines.

\definecolor{srbg}{RGB}{232,241,250}
% Highlight SR cells while keeping the method column white.
\newcommand{\srshade}{\rowcolor{srbg}\cellcolor{white}}

\begin{table}[t]
\centering
\caption{
Performance in- and out-of-distribution on real-world robot hardware.
Success rate (SR, \%), our primary metric, is highlighted in blue.
Reward denotes normalized dense reward ($\times 100$).
Higher is better for both metrics.
Values are means $\pm$ standard error over trials:
\texttt{PnPMug} 50 ID / 100 OOD trials,
\texttt{InsertPipette} 30 ID / 110 OOD trials per method.
Bold and underlined values indicate the best and second-best
means, respectively, including ties.
}
\label{tab:hw_results}
\small
\setlength{\tabcolsep}{3pt}
\renewcommand{\arraystretch}{1.12}

\begin{tabular}{@{}llccc@{}}
\toprule
Method & Metric & \texttt{PnPMug} & \texttt{InsertPipette} & Avg. \\
\midrule
\multicolumn{5}{c}{\textit{(a) In-distribution (ID)}} \\
\midrule

\srshade
\multirow{2}{*}{\flow{}}
& \textbf{SR}
& $\sr{86}{5}$ & $\sr{83}{7}$ & $85$ \\
& Reward
& $\sr{95}{2}$ & $\underline{\sr{92}{3}}$ & $93$ \\
\midrule[0.3pt]

\srshade
\multirow{2}{*}{\brp{}}
& \textbf{SR}
& $\mathbf{\sr{100}{0}}$ & $\mathbf{\sr{100}{0}}$ & $\mathbf{100}$ \\
& Reward
& $\underline{\sr{97}{1}}$ & $\mathbf{\sr{99}{1}}$ & $\underline{98}$ \\
\midrule[0.3pt]

\srshade
\multirow{2}{*}{\textbf{\ours{}}}
& \textbf{SR}
& $\underline{\sr{98}{2}}$ & $\underline{\sr{97}{3}}$ & $\underline{97}$ \\
& Reward
& $\mathbf{\sr{99}{1}}$ & $\mathbf{\sr{99}{1}}$ & $\mathbf{99}$ \\

\midrule[0.8pt]
\multicolumn{5}{c}{\textit{(b) Out-of-distribution (OOD)}} \\
\midrule

\srshade
\multirow{2}{*}{\flow{}}
& \textbf{SR}
& $\sr{18}{4}$ & $\sr{18}{4}$ & $18$ \\
& Reward
& $\sr{35}{3}$ & $\sr{61}{2}$ & $48$ \\
\midrule[0.3pt]

\srshade
\multirow{2}{*}{\brp{}}
& \textbf{SR}
& $\underline{\sr{45}{5}}$ & $\underline{\sr{34}{5}}$ & $\underline{39}$ \\
& Reward
& $\underline{\sr{62}{4}}$ & $\underline{\sr{70}{2}}$ & $\underline{66}$ \\
\midrule[0.3pt]

\srshade
\multirow{2}{*}{\textbf{\ours{}}}
& \textbf{SR}
& $\mathbf{\sr{51}{5}}$ & $\mathbf{\sr{50}{5}}$ & $\mathbf{50}$ \\
& Reward
& $\mathbf{\sr{66}{4}}$ & $\mathbf{\sr{79}{2}}$ & $\mathbf{73}$ \\

\bottomrule
\end{tabular}
\end{table}

\begin{figure}
    \centering
    \includegraphics[width=\linewidth]{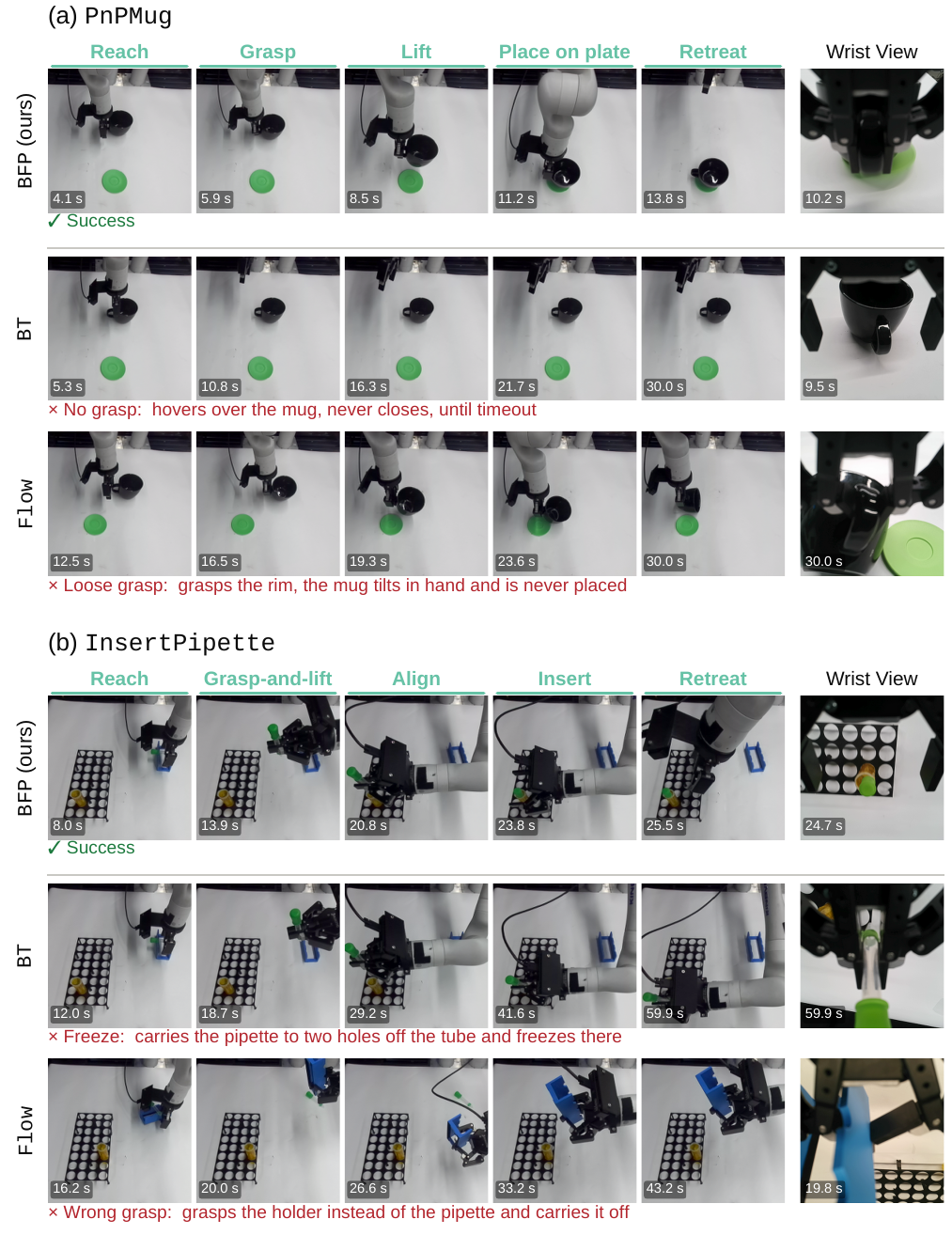}
    \caption{\textbf{OOD Rollout Strips.} Successful \ours{} rollouts in OOD labeled with the progress phase used to compute dense reward. \brp{} and \flow{} failure rollouts in OOD exhibit common failure modes. In \texttt{PnPMug} and \texttt{InsertPipette}, \flow{} often fails early in the task in grasping the manipuland despite its unchanging pose in the scene between ID training demonstrations and OOD rollouts. \brp{}'s time-synchronized anchor cannot re-synchronize once the robot lags behind it, so the residual drifts off support and leads to a freezing failure mode where the policy fails to make further progress. In \texttt{PnPMug} this often occurs in the grasping phase, and in \texttt{InsertPipette} often occurs during the alignment phase.}
    \label{fig:hw_failure_mode_strips}
\end{figure}

\begin{figure}
    \centering
    \includegraphics[width=\linewidth]{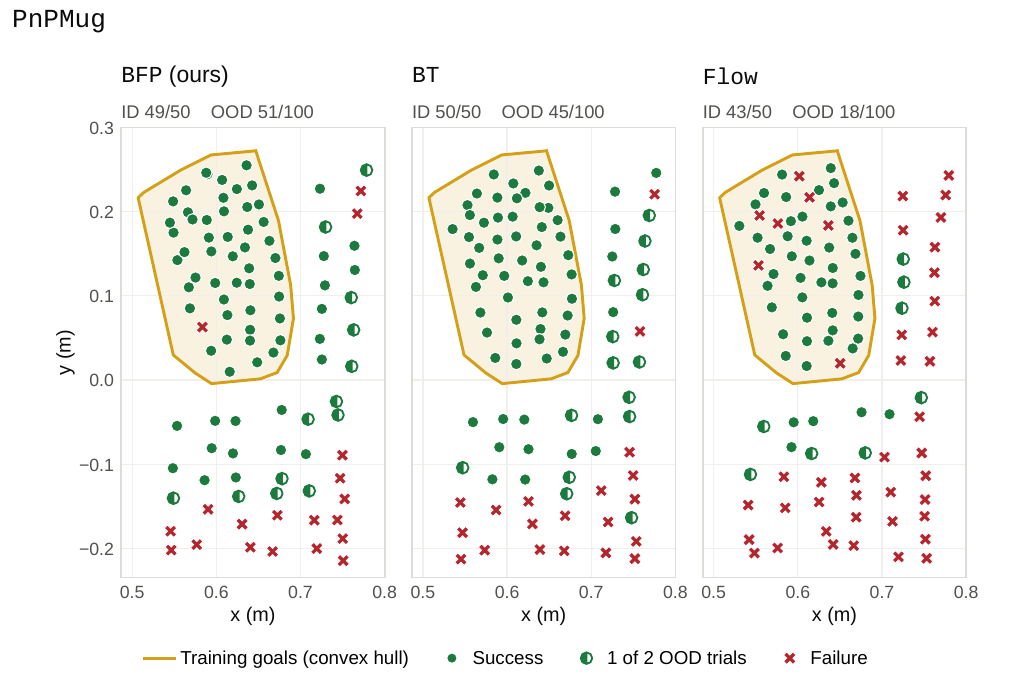}
    \caption{
        \textbf{\texttt{PnPMug} spatial distribution of successes and failures}. Each dot represents a goal object (plate) IC used in evaluation. Failures increase as the goal object is moved further from the training data distribution (ID). \ours{} achieves successful generalization to these OOD conditions more often and further away from the training data goal pose distribution compared to \flow{}.
    }
    \label{fig:pnp_ic_overlay_results}
\end{figure}

\begin{figure}
    \centering
    \includegraphics[width=\linewidth]{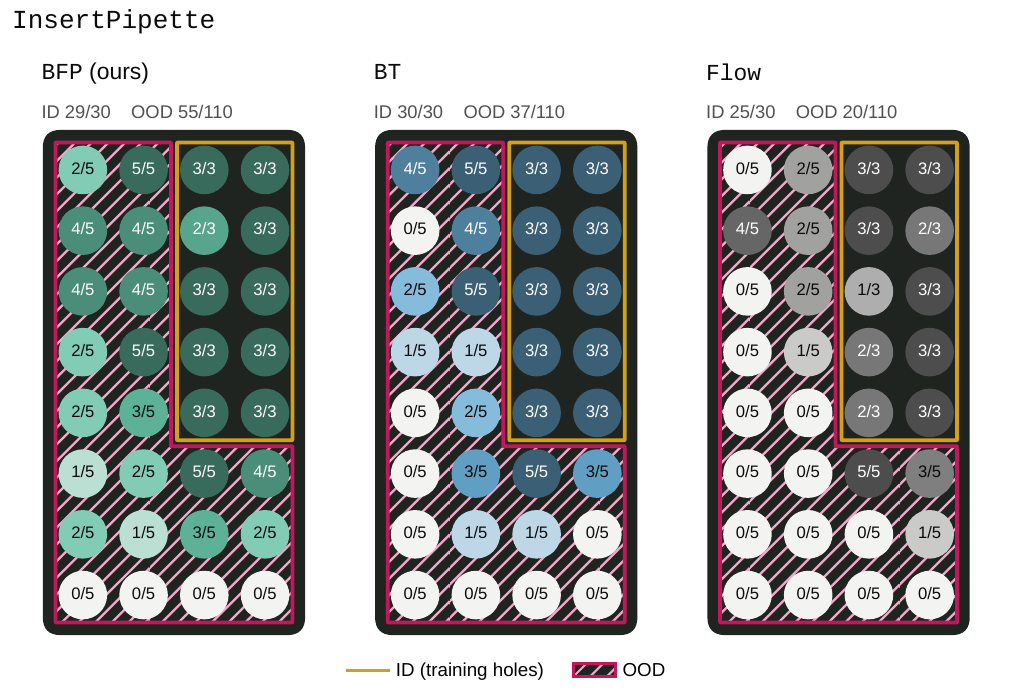}
    \caption{
        \textbf{\texttt{InsertPipette} spatial distribution of successes and failures}. Each circle represents a test-tube IC used for evaluation in a $4\times8$ test-tube tray. Generalization to unseen, OOD goals is substantially improved with \ours{}.
    }
    \label{fig:spoid_ic_overlay_results}
\end{figure}

\begin{figure}
    \centering
    \includegraphics[width=\linewidth]{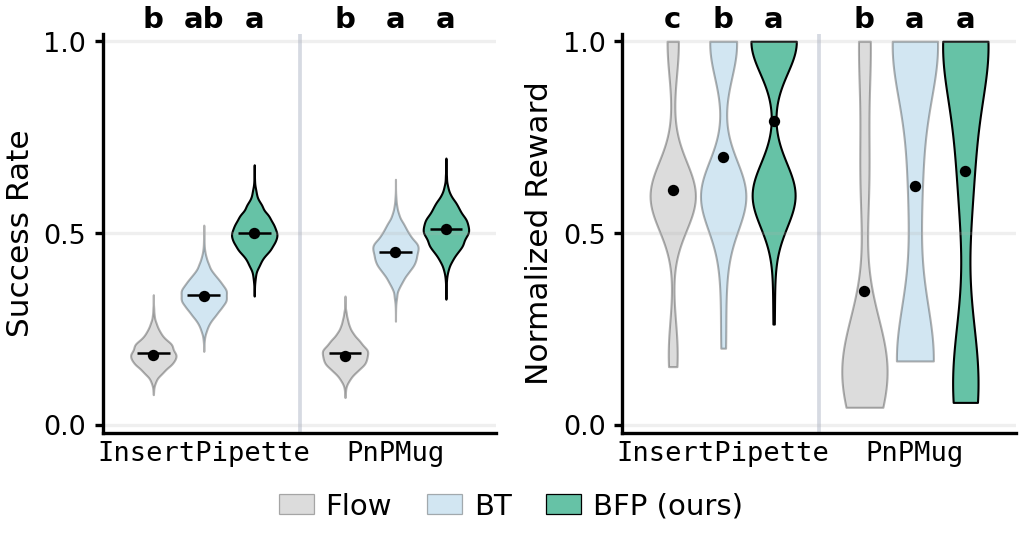}
    \caption{\textbf{Success rate and Dense reward under OOD evaluation on real-world robot hardware.}}
    \label{fig:hardware_violin}
\end{figure}

\clearpage
\section{Simulation Benchmark Details and Additional Results}
\label{app:experiment_details}

\textbf{Tasks and goals.}
We evaluate five manipulation tasks across seven settings:
\texttt{PushCube-D0/D1}, \texttt{Relocate-D0/D1},
\texttt{RotatePen}, \texttt{StackCube}, and \texttt{PickCube}
(\Cref{fig:task_suite}).
The tasks span versatile skills including planar pushing, dexterous relocation and rotation, pick-and-place, and stacking.
Each policy is conditioned on an explicit task goal; the goal definition, success criterion, and dense reward for each task are summarized in \Cref{tab:sim_task_definitions}.

\textbf{Distribution shifts.}
We evaluate extrapolation by holding out goal conditions from the training support
(\Cref{fig:sim_distribution_shift}).
The D0 settings isolate goal extrapolation under a restricted initial-condition distribution,
whereas D1 evaluates the same goal shift with increased initial-condition variation.
For tasks without a D0/D1 suffix, we use a D1 setting.
Exact demonstration counts and train, ID, and OOD supports are reported in
\Cref{tab:sim_distribution_splits}.

\textbf{Training Details.}
All methods use the same demonstration set for each task and observation regime.
$5\%$ of demonstrations is held out for validation and is not used for policy training, as specified in \Cref{tab:sim_distribution_splits}.
Method-specific architectures, extrapolation rules, and optimization details are provided in \Cref{app:impl_details}.

\textbf{Evaluation protocol and metrics.}
For each task--observation pair, each of three independently trained seeds is evaluated on the same $100$ ID and $100$ OOD initial conditions, yielding $200$ evaluation rollouts per seed.
We report the mean and standard error of the mean (SEM) across training seeds.
Our primary metric is binary task success.
We additionally report normalized dense reward, which measures partial task progress when the binary success criterion is not reached.

\begin{table}[t]
    \centering
    \small
    \setlength{\tabcolsep}{4pt}
    \renewcommand{\arraystretch}{1.15}
    \caption{
        \textbf{Simulation task definitions.}
        Goal denotes the compact representation used for policy conditioning.
        Success is evaluated at the final step of the episode.
        Dense reward is the environment's task-progress reward, normalized by the maximum episode return for each task.
    }
    \label{tab:sim_task_definitions}
    \begin{tabularx}{\textwidth}{@{}lXXX@{}}
        \toprule
        Task (Benchmark) & Goal & Success criterion & Dense reward \\
        \midrule

        \texttt{PushCube} (MetaWorld)
        &
        2D target position
        &
        target dist. $<5$\,cm
        &
        reach, grasp, and target progress
        \\

        \texttt{Relocate} (Adroit)
        &
        3D target position
        &
        target dist. $<10$\,cm
        &
        reach, lift, and target progress
        \\

        \texttt{RotatePen} (Adroit)
        &
        target orientation
        &
        ori. cos. $>0.95$, pos. err. $<7.5$\,cm
        &
        orientation and position progress
        \\

        \texttt{StackCube} (ManiSkill3)
        &
        target cube position
        &
        stacked, released, static
        &
        reach, grasp, stack, and release progress
        \\

        \texttt{PickCube} (ManiSkill3)
        &
        3D target position
        &
        target dist. $<2.5$\,cm
        &
        reach, grasp, and place progress
        \\

        \bottomrule
    \end{tabularx}
\end{table}

\begin{table}[t]
    \centering
    \small
    \setlength{\tabcolsep}{4pt}
    \renewcommand{\arraystretch}{1.12}
    \caption{
        \textbf{Simulation data and evaluation splits.}
        ID goals are sampled from the training goal support while OOD goals are held out.
        D0 fixes the object initial-state variation.
        Positions and angles are reported in meters and radians, respectively.
    }
    \label{tab:sim_distribution_splits}

    \begin{tabularx}{\textwidth}{@{}l c X X X@{}}
        \toprule
        Setting
        & Demos
        & ID goal support
        & OOD goal support
        & Object initial condition
        \\
        \midrule

        \texttt{PushCube-D0}
        & 200
        & $x\in[0,0.3]$ \newline $y\in[0.5,0.7]$
        & $x\in[-0.3,0]$ \newline $y\in[0.5,0.7]$
        & fixed
        \\

        \texttt{PushCube-D1}
        & 200
        & $x\in[0,0.3]$ \newline $y\in[0.5,0.7]$
        & $x\in[-0.3,0]$ \newline $y\in[0.5,0.7]$
        & puck $x\in[-0.06,0.06]$
        \\

        \texttt{Relocate-D0}
        & 100
        & $x,y\in[0,0.3]$ \newline $z\in[0.15,0.35]$
        & $x,y\in[-0.3,0]$ \newline $z\in[0.15,0.35]$
        & fixed
        \\

        \texttt{Relocate-D1}
        & 250
        & $x,y\in[0,0.3]$ \newline $z\in[0.15,0.35]$
        & $x,y\in[-0.3,0]$ \newline $z\in[0.15,0.35]$
        & ball $x\in[-0.1,0.1]$
        \\

        \texttt{RotatePen}
        & 187
        & roll $\in[0.2,1.0]$ \newline pitch $\in[-0.2,0.6]$
        & roll $\in[-0.3,0]$ \newline pitch $\in[-0.2,0.6]$
        & fixed
        \\

        \texttt{StackCube}
        & 200
        & $x\in[0, 0.2]$ \newline $y\in[0, 0.27]$
        & $x\in[-0.19,0]$ \newline $y\in[-0.25, 0]$
        & $x\in[-0.19,0.19]$ \newline
          $y\in[-0.29,0.30]$
        \\

        \texttt{PickCube}
        & 200
        & $x,y\in(0,0.1]$
        & $x,y\in[-0.1,0)$
        & cube $x,y\in[-0.1,0.1]$
        \\

        \bottomrule
    \end{tabularx}
\end{table}

\textbf{Results.}
\Cref{tab:ood_results} reports the complete task-level OOD results under state and visuomotor
observations.
We report both success rate and normalized dense reward; the corresponding task-averaged comparison
is discussed in the main text.
\Cref{fig:success_ood,fig:dense_reward_ood} visualize the full per-task distributions.

\Cref{tab:ind_results} reports the corresponding ID results.
Across methods, ID performance is generally high relative to OOD performance, indicating that
the primary differences between extrapolation methods emerge outside the demonstrated goal support.
\Cref{fig:success_id,fig:dense_reward_id} provide the complete per-task ID distributions.

\newcommand{\splitdatasetpanel}[4]{%
    \begin{subfigure}[t]{#1}
        \centering
        % \makebox[\linewidth]{%
        %     \makebox[0.5\linewidth][c]{\scriptsize\bfseries ID}%
        %     \makebox[0.5\linewidth][c]{\scriptsize\bfseries OOD}%
        % }
        \vspace{1pt}
        \includegraphics[width=\linewidth]{#2}
        \caption{#3}
        \label{#4}
    \end{subfigure}%
}

\begin{figure}[t]
    \centering

    % First row
    \subfloat[\texttt{PushCube}]{
        \includegraphics[
            width=0.18\columnwidth,
            keepaspectratio
        ]{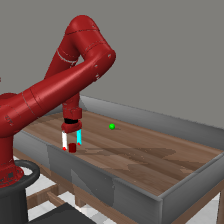}
    }
    % \hfill
    \subfloat[\texttt{Relocate}]{
        \includegraphics[
            width=0.18\columnwidth,
            keepaspectratio
        ]{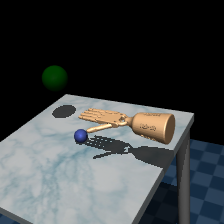}
    }
    % \hfill
    \subfloat[\texttt{RotatePen}]{
        \includegraphics[
            width=0.18\columnwidth,
            keepaspectratio
        ]{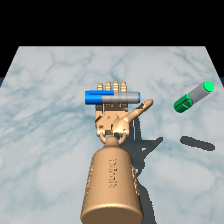}
    }
    % \hfill
    \subfloat[\texttt{StackCube}]{
        \includegraphics[
            width=0.18\columnwidth,
            keepaspectratio
        ]{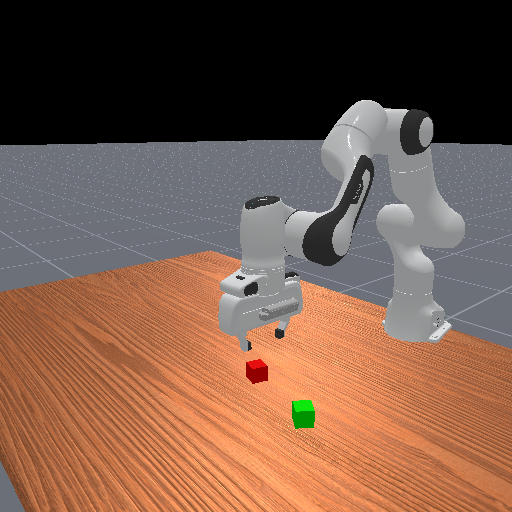}
    }
    % \hfill
    \subfloat[\texttt{PickCube}]{
        \includegraphics[
            width=0.18\columnwidth,
            keepaspectratio
        ]{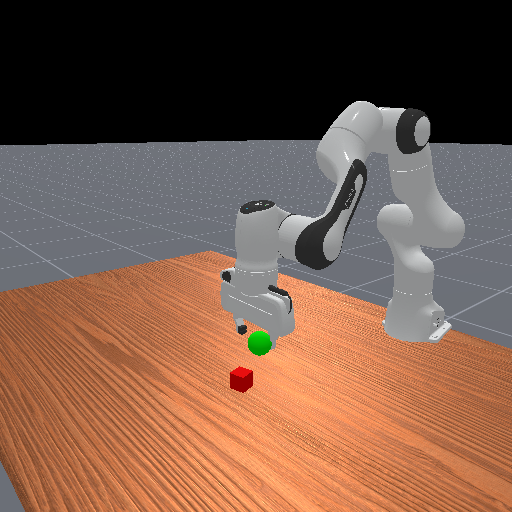}
    }
    % \hfill
    % \hfill
    % \subfloat[\texttt{{PegInsert}}]{
    %     \includegraphics[
    %         width=0.3\columnwidth,
    %         keepaspectratio
    %     ]{figures/task_peg_insert.png}
    % }

    \caption{
    \textbf{Evaluation Tasks.}
    (a) \texttt{PushCube},
    (b) \texttt{Relocate},
    (c) \texttt{RotatePen},
    (d) \texttt{StackCube},
    (e) \texttt{PickCube}.
    % (f) \texttt{PegInsert}.
    The suite spans pushing, dexterous object relocation and rotation, pick-and-place, stacking, and insertion.
    For definition of ID/OOD per task, please see \Cref{fig:sim_distribution_shift}.
    }
    \label{fig:task_suite}
\end{figure}

\begin{figure*}[t]
    \centering

    % First row
    \splitdatasetpanel{0.315\textwidth}
        {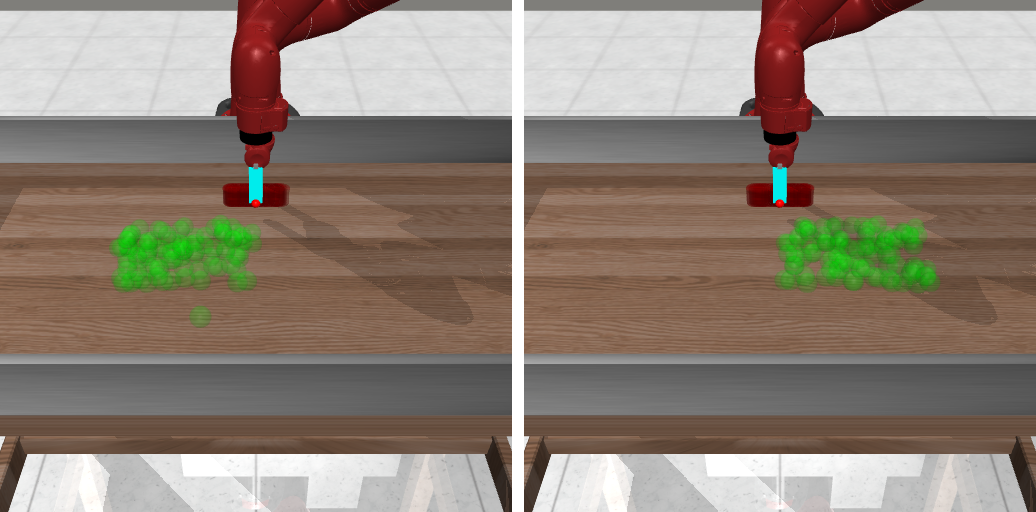}
        {\texttt{PushCube-D1}}
        {fig:split_push_d0}
    \hfill
    \splitdatasetpanel{0.315\textwidth}
        {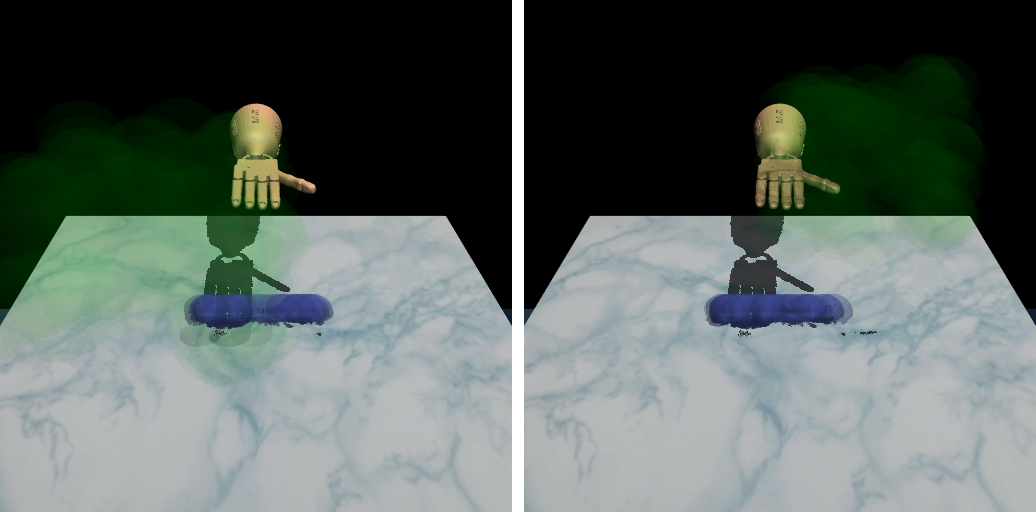}
        {\texttt{Relocate-D1}}
        {fig:split_push_d1}
    \hfill
    \splitdatasetpanel{0.315\textwidth}
        {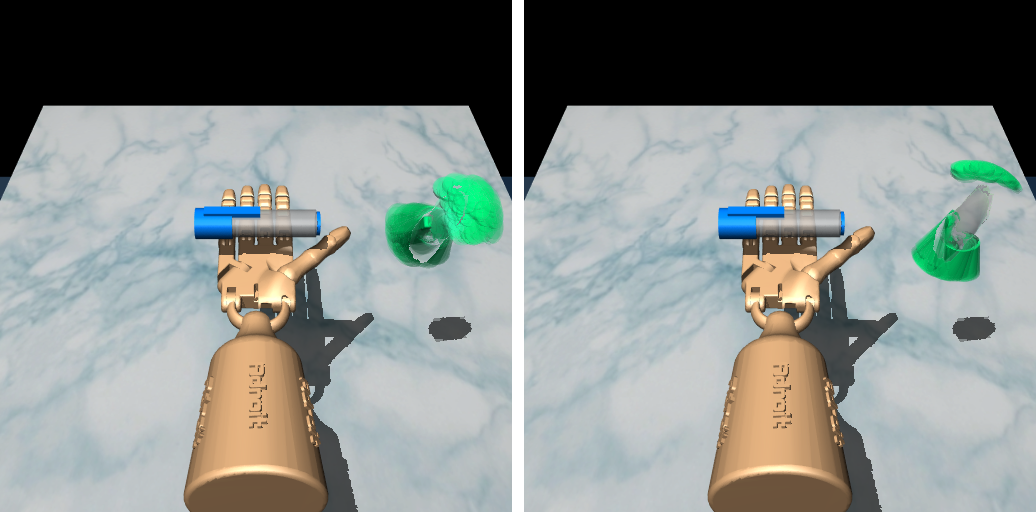}
        {\texttt{RotatePen}}
        {fig:split_relocate_d0}

    \vspace{0.8em}

    % Second row, centered
    \hspace*{0.16\textwidth}%
    \splitdatasetpanel{0.315\textwidth}
        {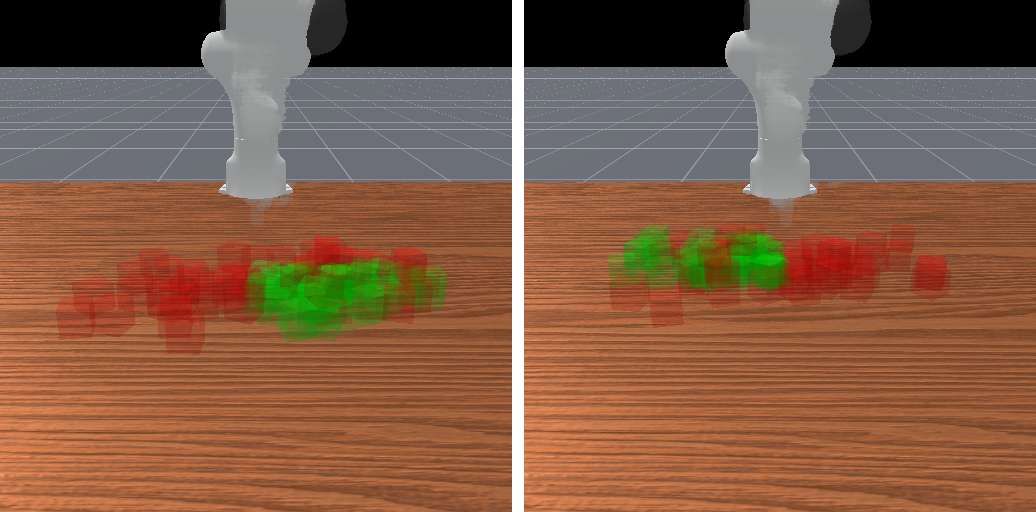}
        {\texttt{StackCube}}
        {fig:split_relocate_d1}
    \hfill
    \splitdatasetpanel{0.315\textwidth}
        {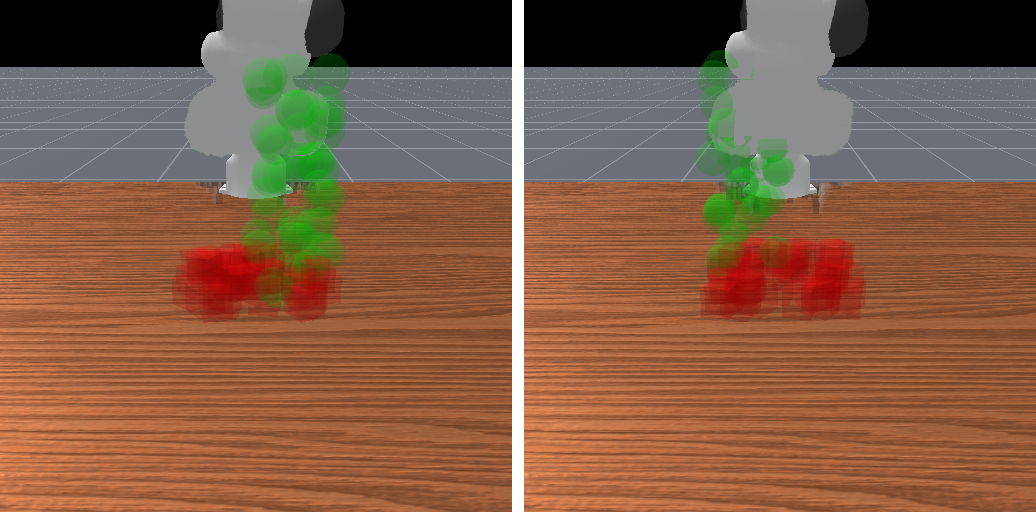}
        {\texttt{PickCube}}
        {fig:split_pen}
    \hspace*{0.16\textwidth}

    \caption{
        \textbf{Distribution Shift.}
        Each task panel visualizes the support of the training conditions
        (ID, left) and the held-out evaluation conditions (OOD, right).
        The paired views use the same camera and scene configuration so that
        the corresponding distribution shift is directly visible.
    }
    \label{fig:sim_distribution_shift}
\end{figure*}

\begin{table}[t]
\centering
\caption{
In-distribution performance with state and visuomotor observations.
Success rate (SR, \%), our primary metric, is highlighted in blue.
Reward denotes normalized dense reward ($\times 100$).
Higher is better for both metrics.
Bold and underlined values indicate the best and second-best
means, respectively, including ties.
Mean $\pm$ standard error over three training seeds.
}
\label{tab:ind_results}
\small
\setlength{\tabcolsep}{3pt}
\renewcommand{\arraystretch}{1.12}

\resizebox{\textwidth}{!}{%
\begin{tabular}{@{}llcccccccc@{}}
\toprule
Method & Metric
& \texttt{PushCube-D0}
& \texttt{PushCube-D1}
& \texttt{Relocate-D0}
& \texttt{Relocate-D1}
& \texttt{RotatePen}
& \texttt{StackCube}
& \texttt{PickCube}
& Avg. \\
\midrule
\multicolumn{10}{c}{\textit{(a) State}} \\
\midrule

\srshade
\multirow{2}{*}{\flow{}}
& \textbf{SR}
& $\mathbf{\sr{100}{0}}$ & $\mathbf{\sr{100}{0}}$
& $\sr{98}{1}$ & $\mathbf{\sr{100}{0}}$
& $\underline{\sr{97}{1}}$ & $\sr{62}{3}$
& $\underline{\sr{96}{1}}$ & $93$ \\
& Reward
& $\mathbf{\sr{89}{0}}$ & $\mathbf{\sr{88}{0}}$
& $\sr{81}{1}$ & $\sr{80}{1}$
& $\underline{\sr{87}{0}}$ & $\sr{64}{0}$
& $\sr{63}{0}$ & $\underline{79}$ \\
\midrule[0.3pt]

\srshade
\multirow{2}{*}{\darp{}}
& \textbf{SR}
& $\mathbf{\sr{100}{0}}$ & $\mathbf{\sr{100}{0}}$
& $\mathbf{\sr{100}{0}}$ & $\mathbf{\sr{100}{0}}$
& $\underline{\sr{97}{0}}$ & $\underline{\sr{75}{0}}$
& $\sr{87}{1}$ & $94$ \\
& Reward
& $\mathbf{\sr{89}{0}}$ & $\mathbf{\sr{88}{0}}$
& $\mathbf{\sr{84}{1}}$ & $\underline{\sr{82}{1}}$
& $\underline{\sr{87}{0}}$ & $\underline{\sr{67}{0}}$
& $\sr{64}{0}$ & $\mathbf{80}$ \\
\midrule[0.3pt]

\srshade
\multirow{2}{*}{\cta{}}
& \textbf{SR}
& $\mathbf{\sr{100}{0}}$ & $\underline{\sr{96}{3}}$
& $\mathbf{\sr{100}{0}}$ & $\mathbf{\sr{100}{0}}$
& $\underline{\sr{97}{0}}$ & $\sr{52}{2}$
& $\sr{82}{2}$ & $89$ \\
& Reward
& $\mathbf{\sr{89}{0}}$ & $\underline{\sr{86}{2}}$
& $\sr{78}{1}$ & $\sr{81}{1}$
& $\underline{\sr{87}{0}}$ & $\sr{58}{0}$
& $\sr{62}{0}$ & $77$ \\
\midrule[0.3pt]

\srshade
\multirow{2}{*}{\edp{}}
& \textbf{SR}
& $\mathbf{\sr{100}{0}}$ & $\mathbf{\sr{100}{0}}$
& $\underline{\sr{99}{0}}$ & $\mathbf{\sr{100}{0}}$
& $\mathbf{\sr{98}{1}}$ & $\sr{58}{2}$
& $\sr{90}{1}$ & $92$ \\
& Reward
& $\mathbf{\sr{89}{0}}$ & $\mathbf{\sr{88}{0}}$
& $\underline{\sr{82}{1}}$ & $\mathbf{\sr{83}{1}}$
& $\mathbf{\sr{88}{0}}$ & $\sr{63}{0}$
& $\sr{62}{0}$ & $\underline{79}$ \\
\midrule[0.3pt]

\srshade
\multirow{2}{*}{\brp{}}
& \textbf{SR}
& $\mathbf{\sr{100}{0}}$ & $\mathbf{\sr{100}{0}}$
& $\mathbf{\sr{100}{0}}$ & $\underline{\sr{99}{0}}$
& $\underline{\sr{97}{0}}$ & $\sr{74}{3}$
& $\mathbf{\sr{100}{0}}$ & $\underline{96}$ \\
& Reward
& $\mathbf{\sr{89}{0}}$ & $\mathbf{\sr{88}{0}}$
& $\sr{80}{0}$ & $\sr{81}{1}$
& $\underline{\sr{87}{0}}$ & $\sr{66}{1}$
& $\mathbf{\sr{67}{0}}$ & $\mathbf{80}$ \\
\midrule[0.3pt]

\srshade
\multirow{2}{*}{\textbf{\ours{}}}
& \textbf{SR}
& $\mathbf{\sr{100}{0}}$ & $\mathbf{\sr{100}{0}}$
& $\mathbf{\sr{100}{0}}$ & $\mathbf{\sr{100}{0}}$
& $\sr{96}{0}$ & $\mathbf{\sr{89}{1}}$
& $\mathbf{\sr{100}{0}}$ & $\mathbf{98}$ \\
& Reward
& $\underline{\sr{87}{0}}$ & $\mathbf{\sr{88}{0}}$
& $\underline{\sr{82}{1}}$ & $\sr{81}{0}$
& $\underline{\sr{87}{0}}$ & $\mathbf{\sr{71}{1}}$
& $\underline{\sr{65}{0}}$ & $\mathbf{80}$ \\

\midrule[0.8pt]
\multicolumn{10}{c}{\textit{(b) Visuomotor}} \\
\midrule

\srshade
\multirow{2}{*}{\flow{}}
& \textbf{SR}
& $\mathbf{\sr{100}{0}}$ & $\sr{94}{1}$
& $\underline{\sr{94}{2}}$ & $\sr{95}{2}$
& $\mathbf{\sr{99}{1}}$ & $\sr{60}{1}$
& $\sr{94}{1}$ & $91$ \\
& Reward
& $\mathbf{\sr{89}{0}}$ & $\sr{83}{1}$
& $\sr{77}{1}$ & $\sr{79}{2}$
& $\mathbf{\sr{89}{0}}$ & $\sr{60}{1}$
& $\sr{65}{0}$ & $77$ \\
\midrule[0.3pt]

\srshade
\multirow{2}{*}{\darp{}}
& \textbf{SR}
& $\mathbf{\sr{100}{0}}$ & $\sr{97}{0}$
& $\mathbf{\sr{100}{0}}$ & $\sr{89}{1}$
& $\mathbf{\sr{99}{1}}$ & $\sr{77}{1}$
& $\sr{90}{1}$ & $93$ \\
& Reward
& $\mathbf{\sr{89}{0}}$ & $\sr{86}{0}$
& $\mathbf{\sr{81}{0}}$ & $\sr{75}{1}$
& $\mathbf{\sr{89}{0}}$ & $\sr{68}{0}$
& $\mathbf{\sr{68}{0}}$ & $79$ \\
\midrule[0.3pt]

\srshade
\multirow{2}{*}{\cta{}}
& \textbf{SR}
& $\underline{\sr{98}{0}}$ & $\sr{75}{4}$
& $\sr{90}{2}$ & $\sr{74}{3}$
& $\sr{96}{2}$ & $\sr{55}{7}$
& $\sr{11}{0}$ & $71$ \\
& Reward
& $\sr{64}{8}$ & $\sr{43}{5}$
& $\sr{41}{1}$ & $\sr{36}{4}$
& $\sr{85}{0}$ & $\sr{58}{3}$
& $\sr{38}{1}$ & $52$ \\
\midrule[0.3pt]

\srshade
\multirow{2}{*}{\edp{}}
& \textbf{SR}
& $\mathbf{\sr{100}{0}}$ & $\mathbf{\sr{100}{0}}$
& $\mathbf{\sr{100}{0}}$ & $\mathbf{\sr{100}{0}}$
& $\sr{97}{0}$ & $\underline{\sr{83}{3}}$
& $\sr{94}{1}$ & $\underline{96}$ \\
& Reward
& $\mathbf{\sr{89}{0}}$ & $\mathbf{\sr{88}{0}}$
& $\underline{\sr{80}{1}}$ & $\mathbf{\sr{83}{1}}$
& $\underline{\sr{87}{0}}$ & $\underline{\sr{69}{2}}$
& $\mathbf{\sr{68}{1}}$ & $\mathbf{81}$ \\
\midrule[0.3pt]

\srshade
\multirow{2}{*}{\brp{}}
& \textbf{SR}
& $\mathbf{\sr{100}{0}}$ & $\sr{97}{1}$
& $\mathbf{\sr{100}{0}}$ & $\mathbf{\sr{100}{0}}$
& $\sr{97}{0}$ & $\underline{\sr{83}{2}}$
& $\mathbf{\sr{97}{1}}$ & $\underline{96}$ \\
& Reward
& $\mathbf{\sr{89}{0}}$ & $\sr{86}{1}$
& $\underline{\sr{80}{0}}$ & $\sr{80}{0}$
& $\sr{86}{0}$ & $\sr{68}{1}$
& $\underline{\sr{67}{0}}$ & $79$ \\
\midrule[0.3pt]

\srshade
\multirow{2}{*}{\textbf{\ours{}}}
& \textbf{SR}
& $\mathbf{\sr{100}{0}}$ & $\underline{\sr{99}{1}}$
& $\mathbf{\sr{100}{0}}$ & $\underline{\sr{99}{0}}$
& $\underline{\sr{98}{0}}$ & $\mathbf{\sr{88}{0}}$
& $\underline{\sr{96}{1}}$ & $\mathbf{97}$ \\
& Reward
& $\underline{\sr{87}{0}}$ & $\underline{\sr{87}{1}}$
& $\underline{\sr{80}{1}}$ & $\underline{\sr{81}{0}}$
& $\underline{\sr{87}{0}}$ & $\mathbf{\sr{70}{0}}$
& $\sr{66}{1}$ & $\underline{80}$ \\

\bottomrule
\end{tabular}%
}
\end{table}

% Uses your existing method macros and \sr{mean}{uncertainty}.
\begin{table}[t]
\centering
\caption{
OOD performance with state and visuomotor observations.
Success rate (SR, \%), our primary metric, is highlighted in blue.
Reward denotes normalized dense reward ($\times 100$).
Higher is better for both metrics.
Bold and underlined values indicate the best and second-best
means, respectively, including ties.
}
\label{tab:ood_results}
\small
\setlength{\tabcolsep}{3pt}
\renewcommand{\arraystretch}{1.12}

\resizebox{\textwidth}{!}{%
\begin{tabular}{@{}llcccccccc@{}}
\toprule
Method & Metric
& \texttt{PushCube-D0}
& \texttt{PushCube-D1}
& \texttt{Relocate-D0}
& \texttt{Relocate-D1}
& \texttt{RotatePen}
& \texttt{StackCube}
& \texttt{PickCube}
& Avg. \\
\midrule
\multicolumn{10}{c}{\textit{(a) State}} \\
\midrule

\srshade
\multirow{2}{*}{\flow{}}
& \textbf{SR}
& $\sr{31}{4}$ & $\sr{49}{2}$
& $\sr{26}{7}$ & $\sr{24}{2}$
& $\sr{52}{5}$ & $\sr{41}{1}$
& $\sr{80}{2}$ & $43$ \\
& Reward
& $\sr{27}{4}$ & $\sr{42}{1}$
& $\sr{18}{5}$ & $\sr{22}{2}$
& $\sr{49}{5}$ & $\sr{55}{0}$
& $\sr{50}{3}$ & $38$ \\
\midrule[0.3pt]

\srshade
\multirow{2}{*}{\darp{}}
& \textbf{SR}
& $\underline{\sr{36}{1}}$ & $\sr{45}{1}$
& $\sr{43}{11}$ & $\sr{47}{8}$
& $\sr{75}{10}$ & $\sr{2}{1}$
& $\sr{5}{1}$ & $36$ \\
& Reward
& $\sr{31}{1}$ & $\sr{40}{1}$
& $\sr{29}{7}$ & $\sr{36}{6}$
& $\sr{65}{8}$ & $\sr{38}{1}$
& $\sr{40}{0}$ & $40$ \\
\midrule[0.3pt]

\srshade
\multirow{2}{*}{\cta{}}
& \textbf{SR}
& $\underline{\sr{36}{8}}$ & $\sr{48}{9}$
& $\sr{15}{4}$ & $\sr{24}{4}$
& $\sr{77}{12}$ & $\sr{12}{3}$
& $\sr{25}{3}$ & $34$ \\
& Reward
& $\sr{49}{4}$ & $\sr{45}{7}$
& $\sr{10}{2}$ & $\sr{16}{1}$
& $\sr{70}{10}$ & $\sr{35}{2}$
& $\sr{45}{2}$ & $39$ \\
\midrule[0.3pt]

\srshade
\multirow{2}{*}{\edp{}}
& \textbf{SR}
& $\mathbf{\sr{100}{0}}$ & $\mathbf{\sr{100}{0}}$
& $\sr{27}{4}$ & $\sr{85}{2}$
& $\underline{\sr{88}{1}}$ & $\underline{\sr{61}{2}}$
& $\sr{85}{1}$ & $78$ \\
& Reward
& $\mathbf{\sr{88}{0}}$ & $\mathbf{\sr{88}{0}}$
& $\sr{22}{4}$ & $\sr{53}{2}$
& $\sr{76}{1}$ & $\underline{\sr{60}{2}}$
& $\sr{59}{0}$ & $64$ \\
\midrule[0.3pt]

\srshade
\multirow{2}{*}{\brp{}}
& \textbf{SR}
& $\mathbf{\sr{100}{0}}$ & $\underline{\sr{94}{2}}$
& $\underline{\sr{76}{9}}$ & $\underline{\sr{99}{1}}$
& $\sr{86}{1}$ & $\sr{51}{1}$
& $\underline{\sr{99}{0}}$ & $\underline{86}$ \\
& Reward
& $\underline{\sr{85}{0}}$ & $\sr{77}{1}$
& $\underline{\sr{47}{3}}$ & $\underline{\sr{72}{2}}$
& $\underline{\sr{78}{1}}$ & $\sr{57}{1}$
& $\mathbf{\sr{62}{0}}$ & $\underline{68}$ \\
\midrule[0.3pt]

\srshade
\multirow{2}{*}{\textbf{\ours{}}}
& \textbf{SR}
& $\mathbf{\sr{100}{0}}$ & $\mathbf{\sr{100}{0}}$
& $\mathbf{\sr{100}{0}}$ & $\mathbf{\sr{100}{0}}$
& $\mathbf{\sr{96}{1}}$ & $\mathbf{\sr{78}{9}}$
& $\mathbf{\sr{100}{0}}$ & $\mathbf{96}$ \\
& Reward
& $\sr{83}{0}$ & $\underline{\sr{83}{0}}$
& $\mathbf{\sr{67}{1}}$ & $\mathbf{\sr{81}{1}}$
& $\mathbf{\sr{81}{2}}$ & $\mathbf{\sr{68}{3}}$
& $\underline{\sr{61}{1}}$ & $\mathbf{75}$ \\

\midrule[0.8pt]
\multicolumn{10}{c}{\textit{(b) Visuomotor}} \\
\midrule

\srshade
\multirow{2}{*}{\flow{}}
& \textbf{SR}
& $\sr{28}{0}$ & $\sr{34}{1}$
& $\sr{21}{3}$ & $\sr{17}{2}$
& $\sr{50}{4}$ & $\sr{11}{3}$
& $\sr{45}{6}$ & $30$ \\
& Reward
& $\sr{26}{0}$ & $\sr{31}{1}$
& $\sr{13}{1}$ & $\sr{14}{2}$
& $\sr{68}{3}$ & $\sr{34}{1}$
& $\sr{43}{2}$ & $33$ \\
\midrule[0.3pt]

\srshade
\multirow{2}{*}{\darp{}}
& \textbf{SR}
& $\sr{31}{3}$ & $\sr{23}{1}$
& $\sr{33}{8}$ & $\underline{\sr{32}{12}}$
& $\underline{\sr{57}{7}}$ & $\sr{16}{1}$
& $\sr{0}{0}$ & $28$ \\
& Reward
& $\sr{28}{2}$ & $\sr{21}{1}$
& $\sr{20}{5}$ & $\underline{\sr{24}{8}}$
& $\underline{\sr{69}{4}}$ & $\sr{43}{1}$
& $\sr{45}{0}$ & $36$ \\
\midrule[0.3pt]

\srshade
\multirow{2}{*}{\cta{}}
& \textbf{SR}
& $\sr{11}{1}$ & $\sr{8}{2}$
& $\sr{1}{0}$ & $\sr{0}{0}$
& $\sr{23}{3}$ & $\sr{1}{0}$
& $\sr{0}{0}$ & $6$ \\
& Reward
& $\sr{24}{2}$ & $\sr{12}{2}$
& $\sr{3}{0}$ & $\sr{2}{0}$
& $\sr{29}{5}$ & $\sr{38}{2}$
& $\sr{31}{1}$ & $20$ \\
\midrule[0.3pt]

\srshade
\multirow{2}{*}{\edp{}}
& \textbf{SR}
& $\sr{66}{11}$ & $\underline{\sr{59}{4}}$
& $\sr{7}{2}$ & $\sr{26}{4}$
& $\sr{14}{3}$ & $\underline{\sr{41}{2}}$
& $\underline{\sr{79}{2}}$ & $42$ \\
& Reward
& $\sr{58}{10}$ & $\underline{\sr{51}{3}}$
& $\sr{6}{0}$ & $\sr{17}{2}$
& $\sr{25}{3}$ & $\underline{\sr{47}{2}}$
& $\sr{57}{2}$ & $37$ \\
\midrule[0.3pt]

\srshade
\multirow{2}{*}{\brp{}}
& \textbf{SR}
& $\underline{\sr{99}{1}}$ & $\sr{57}{0}$
& $\underline{\sr{57}{3}}$ & $\sr{29}{3}$
& $\underline{\sr{57}{1}}$ & $\sr{28}{1}$
& $\underline{\sr{79}{2}}$ & $\underline{58}$ \\
& Reward
& $\mathbf{\sr{78}{1}}$ & $\sr{50}{0}$
& $\underline{\sr{40}{1}}$ & $\sr{19}{2}$
& $\sr{60}{2}$ & $\sr{40}{0}$
& $\mathbf{\sr{62}{0}}$ & $\underline{50}$ \\
\midrule[0.3pt]

\srshade
\multirow{2}{*}{\textbf{\ours{}}}
& \textbf{SR}
& $\mathbf{\sr{100}{0}}$ & $\mathbf{\sr{83}{1}}$
& $\mathbf{\sr{99}{1}}$ & $\mathbf{\sr{54}{4}}$
& $\mathbf{\sr{78}{5}}$ & $\mathbf{\sr{48}{5}}$
& $\mathbf{\sr{88}{2}}$ & $\mathbf{79}$ \\
& Reward
& $\underline{\sr{74}{1}}$ & $\mathbf{\sr{72}{1}}$
& $\mathbf{\sr{68}{2}}$ & $\mathbf{\sr{38}{1}}$
& $\mathbf{\sr{80}{1}}$ & $\mathbf{\sr{52}{1}}$
& $\underline{\sr{60}{0}}$ & $\mathbf{63}$ \\

\bottomrule
\end{tabular}%
}
\end{table}

\begin{figure}[t]
    \centering
    \includegraphics[width=\linewidth]{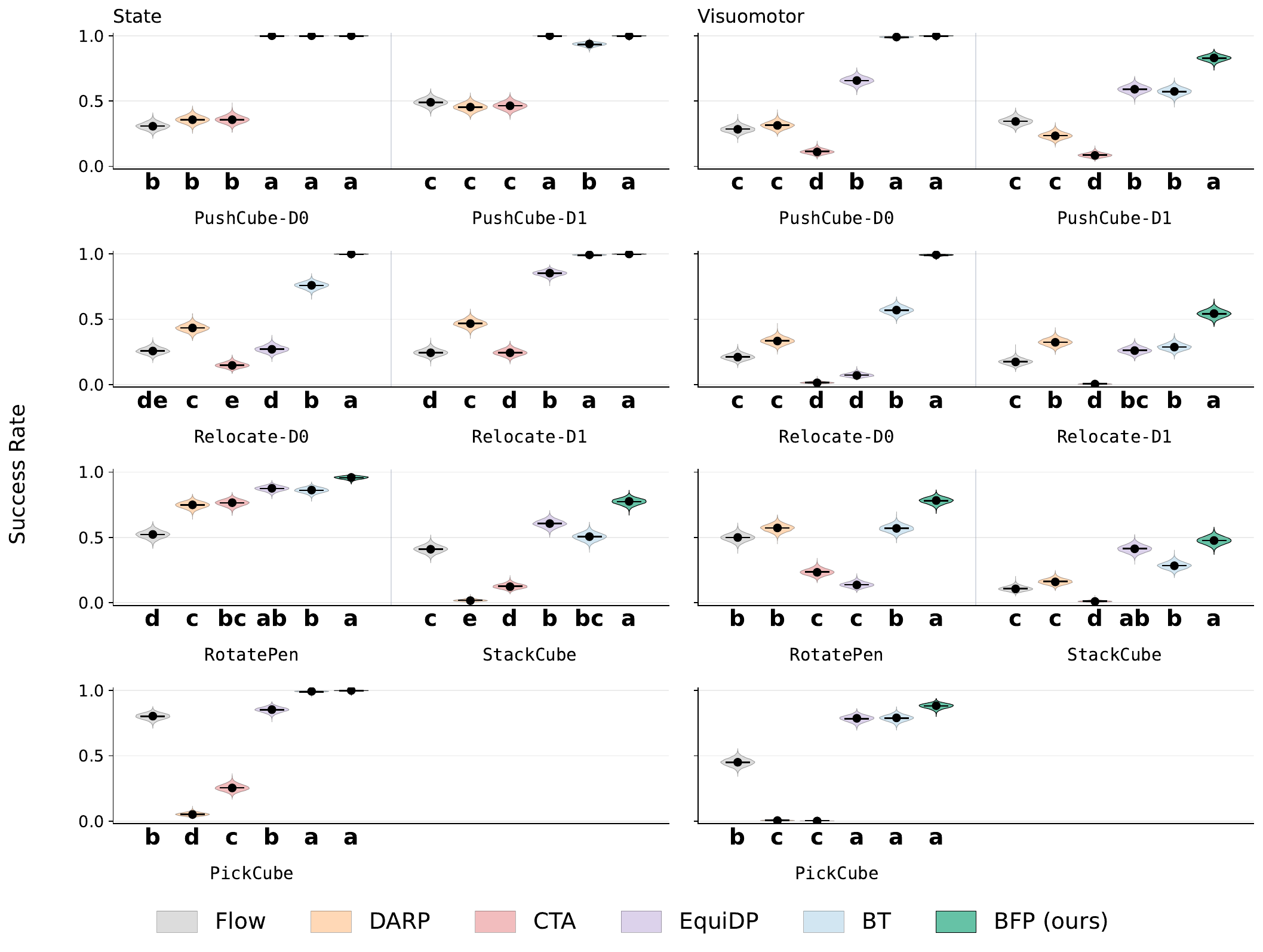}

    \caption{
        \textbf{Success rate under OOD evaluation.}
    }
    \label{fig:success_ood}
\end{figure}

\begin{figure}[t]
    \centering
    \includegraphics[width=\linewidth]{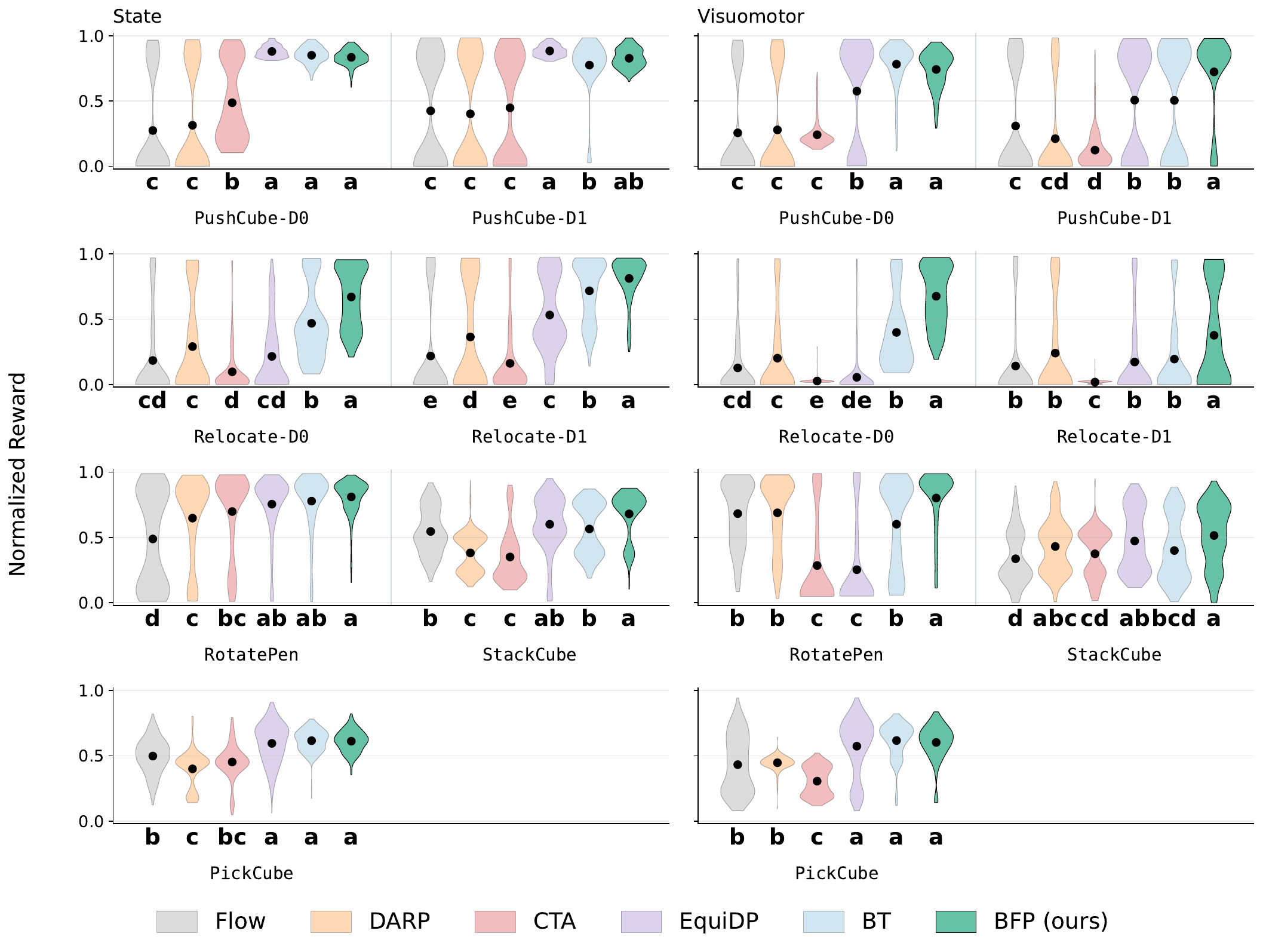}

    \caption{
        \textbf{Dense reward under OOD evaluation.}
    }
    \label{fig:dense_reward_ood}
\end{figure}

\begin{figure}[t]
    \centering
    \includegraphics[width=\linewidth]{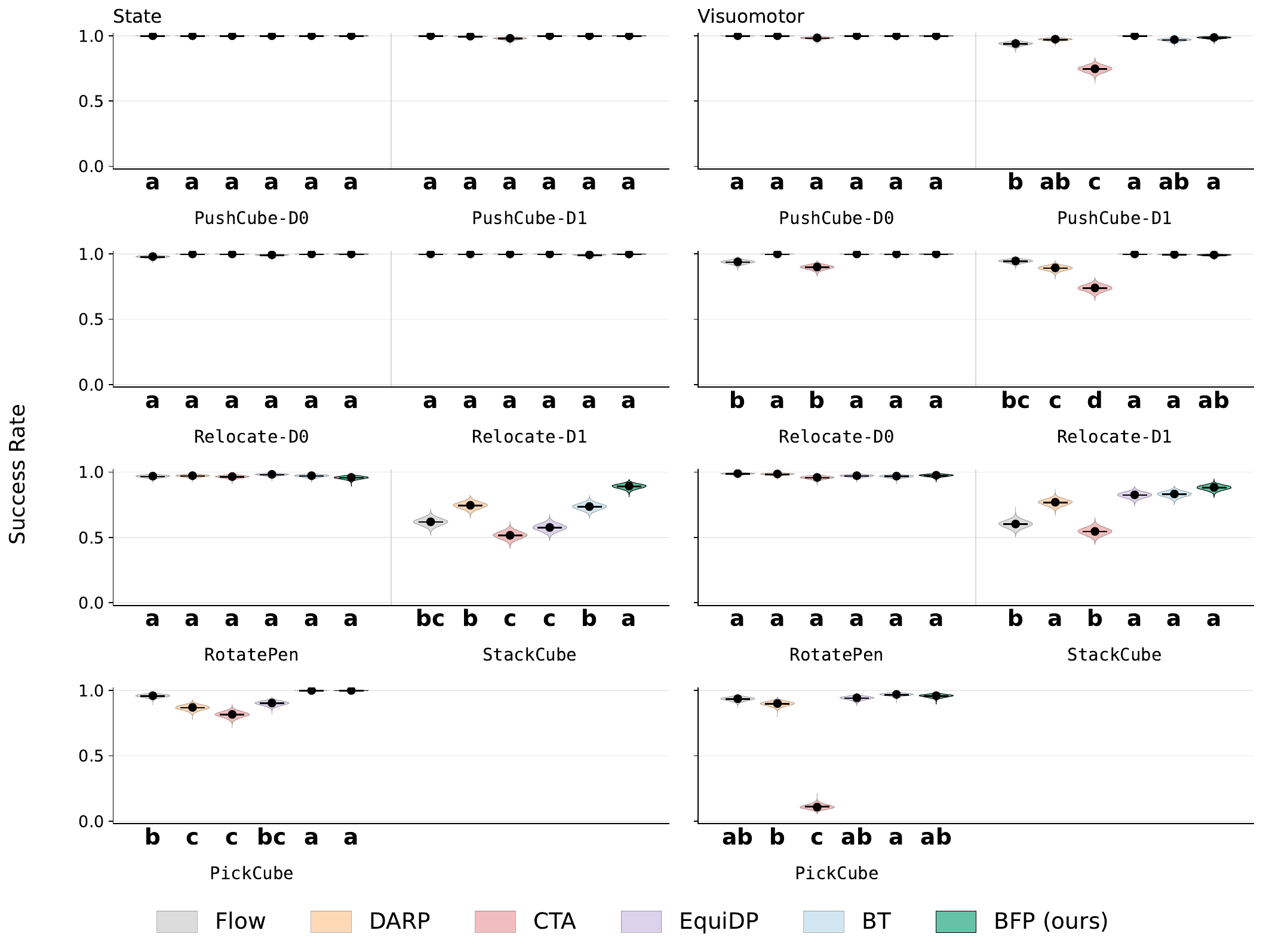}

    \caption{
    \textbf{Success rate under ID evaluation.}
    }
    \label{fig:success_id}
\end{figure}

\begin{figure}[t]
    \centering
    \includegraphics[width=\linewidth]{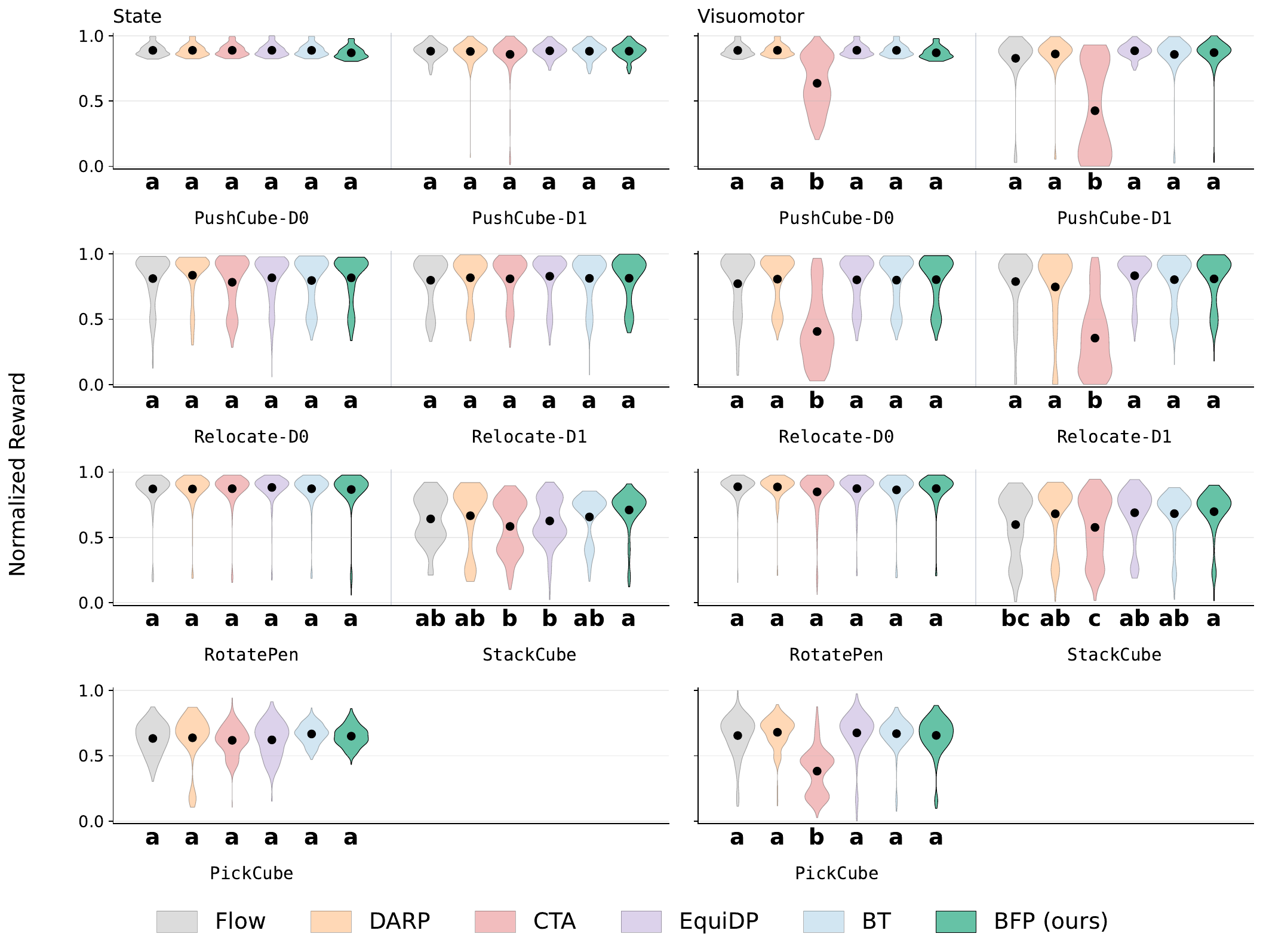}

    \caption{
        \textbf{Dense reward under ID evaluation.}
    }
    \label{fig:dense_reward_id}
\end{figure}

\clearpage

\section{Theoretical Result}
\label{app:theory}

\subsection{Background: Low-Rank Bilinear Combinatorial Extrapolation}
\label{app:theory:background}
In this section, we review the prior combinatorial-extrapolation result \citep{netanyahu2023learning, simchowitz2023tackling} used in the proof and analysis of our method.
We first formulate the regression problem in the anchor--residual notation used by \ours{}, then state the low-rank bilinear extrapolation bound and its matrix-completion interpretation.
\Cref{app:bfp_problem_setup} shows how the BFP encoder, anchor policies, and conditional flow matching instantiate this problem for the velocity field.

\subsubsection{Combinatorial extrapolation}
\label{app:prior_pair_problem}

Consider a regression problem in which a target $Y$ is predicted from a residual feature $\RelDelta$ and an anchor feature $\AnchorLatent$.
For each regime $s\in\{\Train,\Test\}$, let $\overline D_s$ denote the distribution of $(\RelDelta,\AnchorLatent)$, and assume that training and test share the same Bayes regression function
\begin{equation}
    h^\star(\reldelta,\anchorlatent)
    :=
    \E_s\!\left[
        Y
        \mid
        \RelDelta=\reldelta,
        \AnchorLatent=\anchorlatent
    \right],
    \qquad
    s\in\{\Train,\Test\}.
    \label{eq:app_prior_pair_target}
\end{equation}
A rank-$r$ bilinear predictor has the form
\begin{equation}
    h_\theta(\reldelta,\anchorlatent)
    :=
    \left\langle
        f_\theta(\reldelta),
        g_\theta(\anchorlatent)
    \right\rangle_r,
    \qquad
    f_\theta(\reldelta),g_\theta(\anchorlatent)\in\R^r.
    \label{eq:app_prior_pair_predictor}
\end{equation}
Its squared prediction risk in regime $s$ is
\begin{equation}
    \mathcal R_s(h_\theta)
    :=
    \E_{(\RelDelta,\AnchorLatent)\sim\overline D_s}
    \left[
        \left(
            h_\theta(\RelDelta,\AnchorLatent)
            -
            h^\star(\RelDelta,\AnchorLatent)
        \right)^2
    \right].
    \label{eq:app_prior_pair_risk}
\end{equation}
Combinatorial extrapolation asks when small training risk implies small test risk even though $\overline D_{\Test}$ may contain residual--anchor combinations absent from $\overline D_{\Train}$.
Vector-valued targets are handled coordinate-wise, and their squared risks are summed over output coordinates.

\noindent\textbf{Notation.}
We use the BFP symbols $\RelDelta$ and $\AnchorLatent$ from the outset.
Within this appendix, they simply name the two regression features; in BFP, $\AnchorLatent$ is the encoded anchor and $\RelDelta$ is the encoded query--anchor residual.
All inner products $\langle\cdot,\cdot\rangle_r$ below are taken in $\R^r$.

\citet{simchowitz2023tackling} study this combinatorial-extrapolation problem directly for paired features.
Bilinear Transduction \citep{netanyahu2023learning} instead constructs it from an ordinary regression problem: given a query $X$, select an anchor $X^\dagger$ and define
\begin{equation}
    \RelDelta:=X-X^\dagger,
    \qquad
    \AnchorLatent:=X^\dagger,
    \qquad
    h^\star(\reldelta,\anchorlatent)
    :=
    h_{\mathrm{orig}}^\star(\anchorlatent+\reldelta).
    \label{eq:app_bt_instantiation}
\end{equation}
\ours{} uses the same construction in representation space,
$\AnchorLatent=\obsenc(X^\dagger)$ and
$\RelDelta=\obsenc(X)-\obsenc(X^\dagger)$;
\Cref{app:bfp_problem_setup} derives the resulting train and test laws and identifies the conditional-flow regression target.

\subsubsection{Low-rank bilinear combinatorial extrapolation bound}
\label{app:prior_bilinear_extrapolation}

For nonnegative finite measures $\mu$ and $\nu$, write
$\mu\preceq_\kappa\nu$ when
\begin{equation}
    \mu(B)\leq\kappa\nu(B)
    \qquad
    \text{for every measurable set }B.
    \label{eq:app_measure_domination}
\end{equation}
The two assumptions below correspond to \coveragereq{} and \structurereq{} in the main text.

\begin{assum}[Combinatorial coverage]
\label{ass:app_coverage}
There exist residual marginals
$D_{\Delta,1},D_{\Delta,2}$,
anchor marginals
$D_{\AnchorLatent,1},D_{\AnchorLatent,2}$,
and a constant $\kappa\geq1$ such that, with
\begin{equation}
    D_{i\otimes j}
    :=
    D_{\Delta,i}\otimes D_{\AnchorLatent,j},
    \qquad
    i,j\in\{1,2\},
    \label{eq:app_product_blocks}
\end{equation}
we have
\begin{equation}
    \sum_{(i,j)\neq(2,2)}D_{i\otimes j}
    \preceq_\kappa
    \trainpair,
    \qquad
    \testpair
    \preceq_\kappa
    \sum_{i,j\in\{1,2\}}D_{i\otimes j}.
    \label{eq:app_blockwise_coverage}
\end{equation}
\end{assum}

\Cref{ass:app_coverage} requires training to cover the three blocks $D_{1\otimes1}$, $D_{1\otimes2}$, and $D_{2\otimes1}$, while evaluation may additionally contain the unseen combination $D_{2\otimes2}$.

\begin{assum}[Well-conditioned low-rank bilinear structure]
\label{ass:app_structure}
There exist feature maps
\begin{equation}
    f^\star:\latspace_{\Delta}\rightarrow\R^r,
    \qquad
    g^\star:\latspace_{\ancmark}\rightarrow\R^r,
\end{equation}
such that, for all residual--anchor pairs,
\begin{equation}
    h^\star(\reldelta,\anchorlatent)
    =
    \left\langle
        f^\star(\reldelta),
        g^\star(\anchorlatent)
    \right\rangle_r,
    \qquad
    \left|h^\star(\reldelta,\anchorlatent)\right|
    \leq M.
    \label{eq:app_low_rank_structure}
\end{equation}
Furthermore, under $D_{1\otimes1}$, define
\begin{equation}
\begin{aligned}
    \Sigma_\Delta^\star
    &:={}
    \E_{\RelDelta\sim D_{\Delta,1}}
    \left[
        f^\star(\RelDelta)f^\star(\RelDelta)^\top
    \right],
    \\
    \Sigma_{\AnchorLatent}^\star
    &:={}
    \E_{\AnchorLatent\sim D_{\AnchorLatent,1}}
    \left[
        g^\star(\AnchorLatent)g^\star(\AnchorLatent)^\top
    \right].
\end{aligned}
\label{eq:app_reference_factor_covariances}
\end{equation}
The feature directions are non-degenerate:
\begin{equation}
    \sigma^2
    :=
    \lambda_{\min}(\Sigma_\Delta^\star)
    \lambda_{\min}(\Sigma_{\AnchorLatent}^\star)
    >0.
    \label{eq:app_feature_excitation}
\end{equation}
\end{assum}

For \ours{}, this assumption becomes central design objective of encoder and retrieval policy. 
We aim to find a representation and retrieval policy that action is a low-rank bilinear product of residual and the anchor, while avoiding the collapse.

\begin{thm}[Low-rank bilinear combinatorial extrapolation bound; prior result]
\label{thm:app_bilinear_extrapolation}
Suppose \Cref{ass:app_coverage} and \Cref{ass:app_structure} hold, and let $h_\theta$ be the rank-$r$ predictor in \eqref{eq:app_prior_pair_predictor}.
If
\begin{equation}
    \mathcal R_{\Train}(h_\theta)
    \leq
    \frac{\sigma^2}{4\kappa},
    \label{eq:app_bilinear_small_error}
\end{equation}
then
\begin{equation}
    \mathcal R_{\Test}(h_\theta)
    \leq
    \kappa^2
    \left(
        1+\frac{64M^4}{\sigma^4}
    \right)
    \mathcal R_{\Train}(h_\theta).
    \label{eq:app_bilinear_extrapolation_bound}
\end{equation}
\end{thm}

\subsubsection{Matrix-completion interpretation}
\label{app:matrix_completion_identity}
A key proof idea underlying low-rank bilinear combinatorial extrapolation is its reduction to a structured
missing-not-at-random matrix-completion problem
\citep{shah2020sample,simchowitz2023tackling,netanyahu2023learning}.
Suppose $D_{\Delta,i}$ and $D_{\AnchorLatent,j}$ have probability masses
$p_{\Delta,i}$ and $p_{\AnchorLatent,j}$.
For rank-$r$ factors $f$ and $g$, define the probability-weighted factor matrices
\begin{equation}
\begin{aligned}
    \relax[\mathbf F_i]_{\reldelta,:}
    &:={}
    \sqrt{p_{\Delta,i}(\reldelta)}\,
    f(\reldelta)^\top,
    \\
    [\mathbf G_j]_{\anchorlatent,:}
    &:={}
    \sqrt{p_{\AnchorLatent,j}(\anchorlatent)}\,
    g(\anchorlatent)^\top.
\end{aligned}
\label{eq:app_generic_weighted_factors}
\end{equation}
Let
\begin{equation}
\begin{aligned}
    \mathbf M_{i\otimes j}(f,g)
    &:={}
    \mathbf F_i\mathbf G_j^\top,
    \\
    [\mathbf M^\star_{i\otimes j}]_{\reldelta,\anchorlatent}
    &:={}
    \sqrt{
        p_{\Delta,i}(\reldelta)
        p_{\AnchorLatent,j}(\anchorlatent)
    }\,
    h^\star(\reldelta,\anchorlatent).
\end{aligned}
\label{eq:app_generic_prediction_matrices}
\end{equation}
Then, for every block $(i,j)$,
\begin{equation}
\begin{aligned}
    &\E_{\substack{
        \RelDelta\sim D_{\Delta,i},\\
        \AnchorLatent\sim D_{\AnchorLatent,j}
    }}
    \left[
        \left(
            \left\langle
                f(\RelDelta),
                g(\AnchorLatent)
            \right\rangle_r
            -
            h^\star(\RelDelta,\AnchorLatent)
        \right)^2
    \right]
    \\
    &\qquad=
    \left\|
        \mathbf M_{i\otimes j}(f,g)
        -
        \mathbf M^\star_{i\otimes j}
    \right\|_F^2.
\end{aligned}
\label{eq:app_generic_frobenius_identity}
\end{equation}
Notice that, each product-block prediction risk is exactly a probability-weighted Frobenius error.
\ours{} applies this trick to Flow Matching loss as a key proof idea to distributional extrapolation.

\subsection{The Bivariate Conditional-Flow Problem Induced by \ours{}}
\label{app:bfp_problem_setup}

This appendix shows how \ours{} constructs the bivariate regression problem via transductive lens of the core modules.
We begin with the expert action distribution, re-express it in the anchor--residual coordinates induced by the BFP encoder and anchor policies, and then show that conditional flow matching turns action generation into regression of a marginal velocity field.
After fixing a flow state, flow time, and output coordinate, this regression is exactly the combinatorial extrapolation problem of \Cref{app:theory:background}.

\subsubsection{BFP-induced anchor--residual coordinates}
\label{app:induced_pair_distribution}
Condition throughout on the realized demonstration dataset $\DemoData$.
For each regime $s\in\{\Train,\Test\}$, let
\begin{equation}
    X\sim P_s^X,
    \qquad
    \achunk\mid X=x
    \sim
    \expertdist(\cdot\mid x),
    \label{eq:app_original_expert_problem}
\end{equation}
where $P_s^X$ is the marginal distribution of policy conditions at $s\in[\Train, \Test]$ and $\expertdist(\cdot\mid x)$ is the expert distribution over action chunks.
Training and evaluation may use different condition distributions, but they share the same expert conditional $\expertdist(\cdot\mid x)$.
The distributional-extrapolation objective is to generate from this conditional when $x$ is drawn from $P_{\Test}^X$, including conditions outside the support of $P_{\Train}^X$.

For a query $X$, the anchor policy selects an anchor
\begin{equation}
    \anc{X}
    \sim
    \rho_s(\cdot\mid X,\DemoData),
    \qquad
    s\in\{\Train,\Test\}.
    \label{eq:app_anchor_kernel}
\end{equation}
BFP applies the shared observation encoder $\obsenc$ to the query and anchor and forms
\begin{equation}
    \QryEmb:=\obsenc(X),
    \qquad
    \AnchorLatent:=\obsenc(\anc{X}),
    \qquad
    \RelDelta:=\QryEmb-\AnchorLatent.
    \label{eq:app_bfp_pair_map}
\end{equation}
Thus, the encoder and anchor policy define the transformation
\begin{equation}
    T_{\obsenc}(x,x^{\ancmark})
    :=
    \left(
        \obsenc(x)-\obsenc(x^{\ancmark}),
        \obsenc(x^{\ancmark})
    \right).
    \label{eq:app_bfp_pair_map_operator}
\end{equation}
Together with \Cref{eq:app_original_expert_problem}, this transformation induces a joint law
\begin{equation}
    \overline P_s^{\obsenc,\rho_s}
    :=
    \Law_s(\RelDelta,\AnchorLatent,\achunk),
    \label{eq:app_bfp_joint_pushforward}
\end{equation}
which is equivalently characterized by
\begin{equation}
\begin{aligned}
    &\E_{\overline P_s^{\obsenc,\rho_s}}
    \left[
        F(\RelDelta,\AnchorLatent,\achunk)
    \right]
    \\
    &\quad=
    \E_{X\sim P_s^X}
    \E_{\achunk\sim\expertdist(\cdot\mid X)}
    \E_{\anc{X}\sim\rho_s(\cdot\mid X,\DemoData)}
    \left[
        F\!\left(
            T_{\obsenc}(X,\anc{X}),
            \achunk
        \right)
    \right]
\end{aligned}
\label{eq:app_bfp_pushforward_expectation}
\end{equation}
for every bounded measurable $F$.
Let $\overline D_s^{\obsenc,\rho_s}$ denote the $(\RelDelta,\AnchorLatent)$ marginal of this law.
This is the train or test combinatorial-extrapolation distribution constructed by the BFP modules.

The pair reconstructs the encoded query,
\begin{equation}
    \QryEmb
    =
    \RelDelta+\AnchorLatent,
    \label{eq:app_encoded_query_reconstruction}
\end{equation}
but it need not reconstruct the raw condition $X$.
We therefore make explicit when the induced train and test problems share one expert target.

\begin{assum}[Policy-sufficient representation]
\label{ass:app_policy_sufficiency}
There exists a conditional action law
$q_{\obsenc}^{\star}(\cdot\mid\xi)$ such that, for both
$s\in\{\Train,\Test\}$,
\begin{equation}
    \expertdist(\cdot\mid X)
    =
    q_{\obsenc}^{\star}
    (\cdot\mid\obsenc(X))
    \qquad
    P_s^X\text{-almost surely}.
    \label{eq:app_policy_sufficiency}
\end{equation}
\end{assum}

Under \Cref{ass:app_policy_sufficiency}, define the expert action distribution in anchor--residual coordinates by
\begin{equation}
    q_{\mathrm{pair}}^{\star}
    (\mathrm d\mathbf a\mid\reldelta,\anchorlatent)
    :=
    q_{\obsenc}^{\star}
    (\mathrm d\mathbf a\mid\reldelta+\anchorlatent).
    \label{eq:app_pair_endpoint_law}
\end{equation}

For reference, an arbitrary encoded conditional
$q_{\obsenc}^{\star}$ incurs the information loss induced by representation
\begin{equation}
    \varepsilon_{\obsenc,s}^2
    :=
    \E_{X\sim P_s^X}
    W_2^2\!\left(
        \expertdist(\cdot\mid X),
        q_{\obsenc}^{\star}
        (\cdot\mid\obsenc(X))
    \right).
    \label{eq:app_representation_information_loss}
\end{equation}
Exact policy sufficiency corresponds to
$\varepsilon_{\obsenc,\Train}=\varepsilon_{\obsenc,\Test}=0$.
We use exact sufficiency in the main theorem; \Cref{eq:app_representation_information_loss} only records the action-distribution information discarded by an approximate representation.

\subsubsection{Conditional flow matching in anchor--residual coordinates}
\label{app:pair_conditioned_cfm}

The objective at a fixed pair is to sample
\begin{equation}
    \achunk
    \sim
    q_{\mathrm{pair}}^{\star}
    (\cdot\mid\reldelta,\anchorlatent).
    \label{eq:app_pair_generation_objective}
\end{equation}
Let $p_{\tau\mid1}(\mathrm dy\mid\mathbf a)$ be the conditional probability path used by flow matching \citep{lipman_flow_2024}, and let $u_{\tau\mid1}(y\mid\mathbf a)$ be its conditional vector field.
The corresponding marginal path is
\begin{equation}
    p_\tau(\mathrm dy\mid\reldelta,\anchorlatent)
    :=
    \int
    p_{\tau\mid1}(\mathrm dy\mid\mathbf a)
    q_{\mathrm{pair}}^{\star}
    (\mathrm d\mathbf a\mid\reldelta,\anchorlatent).
    \label{eq:app_pair_marginal_path}
\end{equation}
with the corresponding marginal vector field:
\begin{equation}
\begin{aligned}
    u_\tau^{\star}
    (y\mid\reldelta,\anchorlatent)
    &:={}
    \frac{
        \int
        u_{\tau\mid1}(y\mid\mathbf a)
        p_{\tau\mid1}(y\mid\mathbf a)
        q_{\mathrm{pair}}^{\star}
        (\mathrm d\mathbf a\mid\reldelta,\anchorlatent)
    }{
        p_\tau(y\mid\reldelta,\anchorlatent)
    }.
\end{aligned}
\label{eq:app_pair_marginal_vector_field}
\end{equation}
Under the standard flow-matching regularity conditions, this vector field generates the marginal path in \Cref{eq:app_pair_marginal_path} and therefore transports the base distribution $p_0$ to the $q_{\mathrm{pair}}^{\star}
    (\cdot\mid\reldelta,\anchorlatent)$ at $\tau=1$.
\Cref{app:bfp_theory} states the precise well-posedness condition used by our guarantee.

To identify the target optimized by CFM, sample
\begin{equation}
\begin{gathered}
    (\RelDelta,\AnchorLatent,\achunk)
    \sim
    \overline P_s^{\obsenc,\rho_s},
    \qquad
    \flowtime\sim\mathcal U[0,1],
    \\
    \achunk_{\flowtime}
    \sim
    p_{\flowtime\mid1}(\cdot\mid\achunk),
    \qquad
    C:=(\achunk_{\flowtime},\flowtime),
    \\
    U
    :=
    u_{\flowtime\mid1}
    (\achunk_{\flowtime}\mid\achunk).
\end{gathered}
\label{eq:app_cfm_sampling}
\end{equation}
The random variable $U$ is the conditional vector field evaluated at the sampled path state.
The CFM objective regresses the learned vector field against this quantity.

\begin{prop}[CFM regression identity]
\label{prop:app_cfm_identity}
For every square-integrable vector field
$v(C,\RelDelta,\AnchorLatent)$,
\begin{equation}
\begin{aligned}
    &\E_s\!\left[
        \left\|
            v(C,\RelDelta,\AnchorLatent)-U
        \right\|_2^2
    \right]
    \\
    &\quad=
    \E_s\!\left[
        \left\|
            v(C,\RelDelta,\AnchorLatent)
            -u^{\star}(C\mid\RelDelta,\AnchorLatent)
        \right\|_2^2
    \right]
    +\mathcal V_s,
\end{aligned}
\label{eq:app_cfm_regression_identity}
\end{equation}
where
\begin{equation}
    u^{\star}(C\mid\RelDelta,\AnchorLatent)
    :=
    \E_s[U\mid C,\RelDelta,\AnchorLatent]
    \label{eq:app_cfm_bayes_target}
\end{equation}
and
\begin{equation}
    \mathcal V_s
    :=
    \E_s\!\left[
        \left\|
            U-u^{\star}(C\mid\RelDelta,\AnchorLatent)
        \right\|_2^2
    \right]
\end{equation}
is independent of $v$.
Under \Cref{ass:app_policy_sufficiency}, the conditional expectation in \Cref{eq:app_cfm_bayes_target} equals the marginal vector field in \Cref{eq:app_pair_marginal_vector_field} evaluated at $C=(y,\tau)$.
\end{prop}

\begin{proof}
Add and subtract
$u^{\star}(C\mid\RelDelta,\AnchorLatent)$ inside the squared norm.
The cross term vanishes after conditioning on
$(C,\RelDelta,\AnchorLatent)$ because
$u^{\star}$ is the conditional expectation of $U$.
The remaining two terms are the squared error to the conditional expectation and the irreducible conditional variance.
The final statement follows from \Cref{eq:app_pair_marginal_vector_field}.
\end{proof}

\subsubsection{Fixed-context combinatorial extrapolation}
\label{app:cfm_matrix_completion}

For compactness, call
\begin{equation}
    c=(y,\tau)
    \label{eq:app_flow_context_definition}
\end{equation}
a flow context.
Fix a context $c$ and output coordinate
$k\in[\chunkdim]$.
The marginal vector field defines the scalar target
\begin{equation}
    h_{k,c}^{\star}
    (\reldelta,\anchorlatent)
    :=
    u_k^{\star}
    (c\mid\reldelta,\anchorlatent).
    \label{eq:app_fixed_context_target}
\end{equation}
BFP parameterizes the corresponding predictor as
\begin{equation}
\begin{aligned}
    h_{\theta,k,c}
    (\reldelta,\anchorlatent)
    &:={}
    v_{\theta,k}
    (c\mid\reldelta,\anchorlatent)
    \\
    &:={}
    \left\langle
        f_{\theta,k}(\reldelta;c),
        g_{\theta,k}(\anchorlatent;c)
    \right\rangle_r.
\end{aligned}
\label{eq:app_fixed_context_predictor}
\end{equation}
Let $\nu_s$ denote the law of $C$ in regime $s$, and let
\begin{equation}
    \overline D_{s,c}^{\mathrm{BFP}}
    :=
    \Law_s
    \bigl(
        \RelDelta,\AnchorLatent
        \mid C=c
    \bigr)
    \label{eq:app_context_conditioned_pair_law}
\end{equation}
be a regular conditional anchor--residual law.
For every fixed $(c,k)$, \Cref{eq:app_fixed_context_target,eq:app_fixed_context_predictor,eq:app_context_conditioned_pair_law} instantiate the combinatorial-extrapolation problem of \Cref{app:prior_pair_problem}: the train and test covariate laws are
$\overline D_{\Train,c}^{\mathrm{BFP}}$ and
$\overline D_{\Test,c}^{\mathrm{BFP}}$, the shared regression function is $h_{k,c}^{\star}$, and the rank-$r$ predictor is $h_{\theta,k,c}$.

Whenever these conditional laws admit the block decomposition of \Cref{ass:app_coverage}, the matrix-completion identity in \Cref{eq:app_generic_frobenius_identity} applies directly to their coordinate-wise squared error.
Thus, after fixing $(c,k)$, BFP velocity prediction is the same structured low-rank matrix-completion problem reviewed in \Cref{app:matrix_completion_identity}.
\Cref{app:bfp_theory} next states the conditions under which the prior combinatorial-extrapolation bound applies to this family of fixed flow-context problems and converts the resulting vector-field control into an action-distribution guarantee.

\subsection{Distributional Extrapolation of BFP}
\label{app:bfp_theory}
We now complete the argument.
We first identify the flow risk controlled by CFM, then combine per-context bilinear transfer with ODE stability to prove the distributional extrapolation theorem.
We next state the multi-anchor corollary and conclude by discussing sufficient conditions for the theorem assumptions.

\subsubsection{CFM loss induced by BFP}
\label{app:flow_risk}

For $s\in\{\Train,\Test\}$, define the FM error to the marginal vector field by
\begin{equation}
\begin{aligned}
    \mathcal E_s^{\mathrm{FM}}(\theta)
    &:={}
    \E_{\substack{
        (\RelDelta,\AnchorLatent)
        \sim\overline D_s^{\obsenc,\rho_s},\\
        \flowtime\sim\mathcal U[0,1],\\
        \achunk_{\flowtime}
        \sim
        p_{\flowtime}
        (\cdot\mid\RelDelta,\AnchorLatent)
    }}
    \left[
        \left\|
            \flowfield(
                \achunk_{\flowtime},\flowtime
                \mid\RelDelta,\AnchorLatent
            )
            -
            u_{\flowtime}^{\star}(
                \achunk_{\flowtime}
                \mid\RelDelta,\AnchorLatent
            )
        \right\|_2^2
    \right].
\end{aligned}
\label{eq:app_fm_error}
\end{equation}
By \Cref{prop:app_cfm_identity}, the expected CFM objective in regime $s$ decomposes as
\begin{equation}
    \mathcal L_s^{\mathrm{CFM}}(\theta)
    =
    \mathcal E_s^{\mathrm{FM}}(\theta)
    +
    \mathcal V_s,
    \label{eq:app_cfm_fm_decomposition}
\end{equation}
where $\mathcal V_s$ is independent of $\theta$.
The theorem therefore controls the FM error to the marginal vector field; the practical CFM objective differs from it only by the irreducible conditional-variance term.
Recall from \Cref{eq:app_context_conditioned_pair_law} that
$\overline D_{s,c}^{\mathrm{BFP}}$ is the conditional law of the anchor--residual pair given flow context $C=c$, and let $\nu_s$ denote the marginal law of $C$.
For output coordinate $k$, define the fixed-context error
\begin{equation}
\begin{aligned}
    \mathcal E_{s,k}(c;\theta)
    &:={}
    \E_{(\RelDelta,\AnchorLatent)
        \sim\overline D_{s,c}^{\mathrm{BFP}}}
    \left[
        \left(
            v_{\theta,k}
            (c\mid\RelDelta,\AnchorLatent)
            -
            u_k^{\star}
            (c\mid\RelDelta,\AnchorLatent)
        \right)^2
    \right],
    \\
    \mathcal E_s(c;\theta)
    &:={}
    \sum_{k=1}^{\chunkdim}
    \mathcal E_{s,k}(c;\theta).
\end{aligned}
\label{eq:app_context_fm_error}
\end{equation}
Disintegration of the joint pair--context law gives
\begin{equation}
    \mathcal E_s^{\mathrm{FM}}(\theta)
    =
    \int
    \mathcal E_s(c;\theta)
    \,\nu_s(\mathrm dc).
    \label{eq:app_fm_error_disintegration}
\end{equation}

\subsubsection{Main theorem: distributional extrapolation}
\label{app:main_distributional_theorem}

The proof has two steps.
First, the prior combinatorial-extrapolation bound transfers the marginal-vector-field error from training to evaluation at each fixed flow context.
Second, ODE stability converts evaluation FM error into Wasserstein error of the generated action distribution.

\textbf{Combinatorial extrapolation of the marginal vector field.}
\label{app:marginal_field_extrapolation}
Define the bilinear extrapolation factor
\begin{equation}
    \Gamma_{\mathrm{bil}}
    :=
    \kappa^2
    \left(
        1+\frac{64M^4}{\sigma^4}
    \right).
    \label{eq:app_bilinear_factor}
\end{equation}

\begin{prop}[Marginal-vector-field extrapolation]
\label{prop:app_fm_extrapolation}
Assume that, for $\nu_{\Train}$-almost every flow context $c$ and every output coordinate $k$, the train and test problems
$\overline D_{\Train,c}^{\mathrm{BFP}}$ and
$\overline D_{\Test,c}^{\mathrm{BFP}}$ satisfy
\coveragereq{} and \structurereq{} from Appendix~\ref{app:prior_bilinear_extrapolation} with common constants
$(r,M,\sigma,\kappa)$.
Assume also that the fixed-context training errors satisfy
\begin{equation}
    \mathcal E_{\Train,k}(c;\theta)
    \leq
    \frac{\sigma^2}{4\kappa}
    \label{eq:app_context_small_error}
\end{equation}
for every relevant $(c,k)$, and that the evaluation flow-context law is covered by the training law,
\begin{equation}
    \nu_{\Test}
    \preceq_{\kappa_C}
    \nu_{\Train}.
    \label{eq:app_flow_context_coverage}
\end{equation}
Then
\begin{equation}
    \mathcal E_{\Test}^{\mathrm{FM}}(\theta)
    \leq
    \kappa_C
    \Gamma_{\mathrm{bil}}
    \mathcal E_{\Train}^{\mathrm{FM}}(\theta).
    \label{eq:app_fm_extrapolation_bound}
\end{equation}
\end{prop}

\begin{proof}
Fix $c$ and $k$.
Appendix~\ref{app:cfm_matrix_completion} identifies
$u_k^{\star}(c\mid\reldelta,\anchorlatent)$ as the shared scalar regression function and
$v_{\theta,k}(c\mid\reldelta,\anchorlatent)$ as the rank-$r$ bilinear predictor.
Theorem~\ref{thm:app_bilinear_extrapolation} and \Cref{eq:app_context_small_error} therefore give
\begin{equation}
    \mathcal E_{\Test,k}(c;\theta)
    \leq
    \Gamma_{\mathrm{bil}}
    \mathcal E_{\Train,k}(c;\theta).
    \label{eq:app_context_coordinate_extrapolation}
\end{equation}
Summing over coordinates, integrating with respect to $\nu_{\Test}$, applying \Cref{eq:app_flow_context_coverage}, and using \Cref{eq:app_fm_error_disintegration} prove \Cref{eq:app_fm_extrapolation_bound}.
\end{proof}

\Cref{eq:app_context_small_error} is context-wise because the prior theorem is applied separately at each fixed context.
A small average value of $\mathcal E_{\Train}^{\mathrm{FM}}$ alone does not exclude a large error on a rare training context that receives substantial evaluation mass.

\textbf{From FM error to action-distribution error.}
\label{app:fm_to_distribution}
Fix an anchor--residual pair
$p=(\reldelta,\anchorlatent)$.
Let $\mathbf Y_\tau^{\star}$ solve the expert ODE driven by the marginal vector field,
\begin{equation}
    \frac{\mathrm d\mathbf Y_\tau^{\star}}
         {\mathrm d\tau}
    =
    u_\tau^{\star}
    (\mathbf Y_\tau^{\star}\mid p),
    \qquad
    \mathbf Y_0^{\star}\sim p_0,
    \label{eq:app_expert_ode}
\end{equation}
and let $\widehat{\mathbf Y}_\tau$ solve the learned ODE from the same initial sample,
\begin{equation}
    \frac{\mathrm d\widehat{\mathbf Y}_\tau}
         {\mathrm d\tau}
    =
    \flowfield(
        \widehat{\mathbf Y}_\tau,\tau
        \mid p
    ),
    \qquad
    \widehat{\mathbf Y}_0
    =
    \mathbf Y_0^{\star}.
    \label{eq:app_learned_ode}
\end{equation}

\begin{prop}[ODE stability]
\label{prop:app_ode_stability}
Assume that the expert ODE is well posed and realizes the marginal path in \Cref{eq:app_pair_marginal_path},
\begin{equation}
    \Law(\mathbf Y_\tau^{\star})
    =
    p_\tau(\cdot\mid p)
    \qquad
    \text{for every }\tau\in[0,1].
    \label{eq:app_path_realization}
\end{equation}
Assume also that
$\flowfield(\cdot,\tau\mid p)$ is
$L$-Lipschitz in its flow-state argument, uniformly over
$\tau$ and $p$.
Then
\begin{equation}
\begin{aligned}
    &W_2^2\!\left(
        \Law(\widehat{\mathbf Y}_1),
        q_{\mathrm{pair}}^{\star}(\cdot\mid p)
    \right)
    \\
    &\quad\leq
    C_{\mathrm{ODE}}
    \int_0^1
    \E_{\mathbf Y_\tau^{\star}
        \sim p_\tau(\cdot\mid p)}
    \left[
        \left\|
            \flowfield(
                \mathbf Y_\tau^{\star},\tau
                \mid p
            )
            -
            u_\tau^{\star}
            (\mathbf Y_\tau^{\star}\mid p)
        \right\|_2^2
    \right]
    \mathrm d\tau,
    \\
    &\hspace{45mm}
    C_{\mathrm{ODE}}
    :=
    \exp(2L+1).
\end{aligned}
\label{eq:app_ode_stability_bound}
\end{equation}
\end{prop}

\begin{proof}
Set
$e_\tau:=\widehat{\mathbf Y}_\tau-\mathbf Y_\tau^{\star}$.
Adding and subtracting
$\flowfield(\mathbf Y_\tau^{\star},\tau\mid p)$, using Lipschitzness, and applying Young's inequality yield
\begin{equation}
\begin{aligned}
    \frac{\mathrm d}{\mathrm d\tau}
    \|e_\tau\|_2^2
    \leq{}&
    (2L+1)\|e_\tau\|_2^2
    \\
    &+
    \left\|
        \flowfield(
            \mathbf Y_\tau^{\star},\tau
            \mid p
        )
        -
        u_\tau^{\star}
        (\mathbf Y_\tau^{\star}\mid p)
    \right\|_2^2.
\end{aligned}
\label{eq:app_gronwall_step}
\end{equation}
Gronwall's inequality and $e_0=0$ bound the expected terminal coupling cost by the right-hand side of \Cref{eq:app_ode_stability_bound}.
\Cref{eq:app_path_realization} gives
$\Law(\mathbf Y_1^{\star})=q_{\mathrm{pair}}^{\star}(\cdot\mid p)$, so the squared Wasserstein distance is no larger than this coupling cost.
\end{proof}

For a query $x$ and anchor $x^{\ancmark}$, let
$\pi_{\theta,x^{\ancmark}}(\cdot\mid x)$ denote the terminal law of \Cref{eq:app_learned_ode} using
$p=(\obsenc(x)-\obsenc(x^{\ancmark}),\obsenc(x^{\ancmark}))$.
The test-time BFP policy samples one anchor before integration:
\begin{equation}
    \pi_{\theta,\rhotest}(\cdot\mid x)
    :=
    \int
    \pi_{\theta,x^{\ancmark}}(\cdot\mid x)
    \,\rhotest(
        \mathrm dx^{\ancmark}
        \mid x,\DemoData
    ).
    \label{eq:app_bfp_mixture_policy}
\end{equation}

\begin{thm}[Distributional extrapolation of BFP]
\label{thm:app_bfp_distributional_extrapolation}
Suppose Assumption~\ref{ass:app_policy_sufficiency} holds, the assumptions of Proposition~\ref{prop:app_fm_extrapolation} hold, and the ODE assumptions of Proposition~\ref{prop:app_ode_stability} hold.
Then
\begin{equation}
    \distrisk_{\Test}
    (\pi_{\theta,\rhotest})
    \leq
    C_{\mathrm{ODE}}
    \kappa_C
    \Gamma_{\mathrm{bil}}
    \mathcal E_{\Train}^{\mathrm{FM}}(\theta).
    \label{eq:app_distributional_extrapolation_bound}
\end{equation}
\end{thm}

\begin{proof}
Under Assumption~\ref{ass:app_policy_sufficiency}, every anchor--residual representation of a query $x$ has the same target
$q_{\mathrm{pair}}^{\star}(\cdot\mid p)=\expertdist(\cdot\mid x)$.
Apply Proposition~\ref{prop:app_ode_stability} for each test query and sampled anchor, then average the resulting coupling over
$\anc{X}\sim\rhotest(\cdot\mid X,\DemoData)$ and
$X\sim P_{\Test}^X$.
This gives
\begin{equation}
    \distrisk_{\Test}
    (\pi_{\theta,\rhotest})
    \leq
    C_{\mathrm{ODE}}
    \mathcal E_{\Test}^{\mathrm{FM}}(\theta).
    \label{eq:app_test_fm_to_distribution}
\end{equation}
Proposition~\ref{prop:app_fm_extrapolation} then yields \Cref{eq:app_distributional_extrapolation_bound}.
\end{proof}

\subsubsection{Corollary: multi-anchor test-time policies}
\label{app:multi_anchor_policy}

Theorem~\ref{thm:app_bfp_distributional_extrapolation} already permits a stochastic test-time anchor policy.
The following corollary makes the resulting policy distribution explicit.

\begin{cor}[Mixture over anchors]
\label{cor:app_multi_anchor}
For a fixed query $x$, let $\rho(\mathrm da\mid x)$ be any anchor policy and let
$\pi_{\theta,a}(\cdot\mid x)$ be the terminal distribution obtained by fixing anchor $a$ throughout the ODE solve.
Then
\begin{equation}
    \pi_{\theta,\rho}(\cdot\mid x)
    =
    \int
    \pi_{\theta,a}(\cdot\mid x)
    \,\rho(\mathrm da\mid x),
    \label{eq:app_anchor_mixture_identity}
\end{equation}
and, for every target distribution $q(\cdot\mid x)$,
\begin{equation}
\begin{aligned}
    &W_2^2\!\left(
        \pi_{\theta,\rho}(\cdot\mid x),
        q(\cdot\mid x)
    \right)
    \\
    &\quad\leq
    \int
    W_2^2\!\left(
        \pi_{\theta,a}(\cdot\mid x),
        q(\cdot\mid x)
    \right)
    \,\rho(\mathrm da\mid x).
\end{aligned}
\label{eq:app_anchor_mixture_wasserstein}
\end{equation}
\end{cor}

\begin{proof}
Sample the anchor from $\rho$ and then sample from the corresponding anchor-conditioned terminal law, which gives \Cref{eq:app_anchor_mixture_identity}.
For each anchor, choose an optimal coupling between
$\pi_{\theta,a}(\cdot\mid x)$ and $q(\cdot\mid x)$.
Sampling the anchor first and then sampling from its coupling produces a valid coupling for the mixture whose expected transport cost is the right-hand side of \Cref{eq:app_anchor_mixture_wasserstein}.
\end{proof}

The anchor must be sampled before ODE integration and held fixed throughout the solve.
Averaging the anchor-conditioned vector fields instead defines a different ODE and does not generally produce the mixture in \Cref{eq:app_anchor_mixture_identity}.

\subsubsection{Affine-Gaussian inheritance of bilinear structure}
\label{app:affine_gaussian_example}
The main theorem requires policy sufficiency, context-wise bilinear structure and combinatorial coverage, flow-context coverage, and ODE regularity.
These conditions concern different parts of BFP and should not be conflated.

The following proposition provides one sufficient setting for the bilinear-realizability part of \structurereq{}.
It is an explanatory result and is not required by the implementation.

\begin{prop}[Affine-Gaussian inheritance of bilinearity]
\label{prop:app_affine_gaussian}
Let
$p=(\reldelta,\anchorlatent)$ and suppose the pair-conditioned expert endpoint is
\begin{equation}
    \achunk\mid p
    \sim
    \mathcal N(m(p),\Sigma_1),
    \label{eq:app_affine_gaussian_endpoint}
\end{equation}
where $\Sigma_1$ does not depend on $p$.
Assume an affine Gaussian probability path whose conditional vector field has the form
\begin{equation}
    u_{\tau\mid1}(y\mid\achunk)
    =
    c_\tau(y)+B_\tau\achunk.
    \label{eq:app_affine_endpoint_field}
\end{equation}
Then there exist matrices $K_\tau$ and $L_\tau$, independent of $p$, such that
\begin{equation}
    \E[
        \achunk
        \mid
        \achunk_\tau=y,p
    ]
    =
    K_\tau y+L_\tau m(p),
    \label{eq:app_affine_gaussian_posterior_mean}
\end{equation}
and therefore
\begin{equation}
    u_\tau^{\star}(y\mid p)
    =
    \widetilde c_\tau(y)
    +
    \widetilde B_\tau m(p).
    \label{eq:app_affine_marginal_field}
\end{equation}
If every endpoint-mean coordinate has a bilinear realization
\begin{equation}
    m_\ell(p)
    =
    \left\langle
        f_\ell^{\star}(\reldelta),
        g_\ell^{\star}(\anchorlatent)
    \right\rangle_{r_\ell},
    \label{eq:app_bilinear_endpoint_mean}
\end{equation}
then output coordinate $k$ of the marginal vector field has bilinear rank at most
\begin{equation}
    1+
    \sum_{\ell:
        [\widetilde B_\tau]_{k\ell}\neq0}
    r_\ell.
    \label{eq:app_affine_rank_bound}
\end{equation}
If $\widetilde B_\tau$ is diagonal, the rank is at most $r_k+1$.
\end{prop}

\begin{proof}
\Cref{eq:app_affine_gaussian_endpoint} and the affine Gaussian path form a linear-Gaussian model, so the posterior mean of the endpoint is affine in the observed path state and in the prior mean, yielding \Cref{eq:app_affine_gaussian_posterior_mean}.
Taking the conditional expectation of \Cref{eq:app_affine_endpoint_field} gives \Cref{eq:app_affine_marginal_field}.
For coordinate $k$, expand
$[\widetilde B_\tau m(p)]_k$ over the endpoint coordinates, concatenate the scaled bilinear features for every nonzero term, and add one constant feature for
$\widetilde c_{\tau,k}(y)$.
The resulting inner product has the rank in \Cref{eq:app_affine_rank_bound}.
\end{proof}

\subsubsection{When do the remaining assumptions hold?}
\label{app:remaining_sufficient_conditions}

\paragraph{Representation sufficiency.}
Assumption~\ref{ass:app_policy_sufficiency} holds when $\obsenc(X)$ is a sufficient statistic for expert behavior: raw conditions mapped to the same encoded query must have the same expert action distribution.
The representation information loss in \Cref{eq:app_representation_information_loss} measures violations at the action distribution.

\paragraph{Combinatorial coverage and feature excitation.}
These are properties of the anchor--residual problems constructed by $(\obsenc,\rhotrain,\rhotest)$.
The train-time policy must expose the three product blocks required by Assumption~\ref{ass:app_coverage} and sufficiently excite every retained residual-side and anchor-side feature direction in Assumption~\ref{ass:app_structure}.
The test-time policy must represent an evaluation query using factors covered by those training blocks.
The learned shortlist, non-collapse regularization, and compact representation are mechanisms intended to make these properties plausible; they do not certify them.

\paragraph{Flow-context coverage.}
A simple sufficient condition follows when training and evaluation use the same endpoint-conditioned probability-path kernel.
Let
\begin{equation}
    \overline q_s(\mathrm d\mathbf a)
    :=
    \int
    \expertdist(\mathrm d\mathbf a\mid x)
    P_s^X(\mathrm dx)
    \label{eq:app_aggregate_endpoint_law}
\end{equation}
be the aggregate expert endpoint law.
If
\begin{equation}
    \overline q_{\Test}
    \preceq_{\kappa_A}
    \overline q_{\Train},
    \label{eq:app_endpoint_coverage}
\end{equation}
then the common path kernel preserves this domination and implies
\begin{equation}
    \nu_{\Test}
    \preceq_{\kappa_A}
    \nu_{\Train}.
    \label{eq:app_endpoint_implies_context_coverage}
\end{equation}
Indeed, for every measurable context set $B$, integrating the same conditional probability
$\Pr(C\in B\mid\achunk=\mathbf a)$ against \Cref{eq:app_endpoint_coverage} gives \Cref{eq:app_endpoint_implies_context_coverage}.
This is sufficient for \Cref{eq:app_flow_context_coverage}, but it is not implied merely by bounded actions or by overlapping Gaussian support.

\paragraph{ODE regularity.}
Proposition~\ref{prop:app_ode_stability} uses standard ODE conditions: the expert marginal vector field must realize the prescribed path, and the learned vector field must be uniformly Lipschitz in the flow state over the relevant domain.
These assumptions concern stable flow integration and are separate from the combinatorial-extrapolation assumptions on the anchor--residual coordinates.

\clearpage
\section{Retrieval Policy Method Details}
\label{app:anchor_policy}

\begin{figure}[H]
\centering

% ============================================================
% Algorithm 1
% ============================================================
\begin{minipage}[t]{0.485\linewidth}
\captionsetup{type=algorithm}

\hrule height 0.8pt
\vspace{2pt}

\caption{Training}
\label{alg:bfp_training}

\vspace{-3pt}
\hrule height 0.4pt
\vspace{4pt}

\footnotesize
\begin{algorithmic}[1]
\Require $\DemoData$; encoder $\obsenc$; predictor $v_\theta$; scorer $s_\psi$
\Repeat
    \State sample train queries $(\state,\achunk)$ and build $\mathcal C(\state)$
    \State compute $p_\psi$ by \eqref{eq:anchor_policy}
    \State update $\theta$ with \eqref{eq:weighted_predictor_loss}
    \State update EMA predictor $\bar\theta$
    \If{retrieval-update step}
        \State sample held-out queries and rebuild $\mathcal C(\state)$
        \State evaluate candidate losses with $\bar\theta$
        \State update $\psi$ with \eqref{eq:retrieval_objective}
    \EndIf
\Until{convergence}
\end{algorithmic}

\vspace{3pt}
\hrule height 0.4pt
\end{minipage}
\hfill
% ============================================================
% Algorithm 2
% ============================================================
\begin{minipage}[t]{0.485\linewidth}
\captionsetup{type=algorithm}

\hrule height 0.8pt
\vspace{2pt}

\caption{Inference}
\label{alg:bfp_inference}

\vspace{-3pt}
\hrule height 0.4pt
\vspace{4pt}

\footnotesize
\begin{algorithmic}[1]
\Require query $\state$; $\DemoData$; $\obsenc,v_\theta,s_\psi$
\State retrieve demonstration $j^\star$ using \eqref{eq:test_anchor_goal_delta}
\State build the phase-compatible $\mathcal C(\state)$ within $j^*$
\State compute anchor weights $\omega_m(\state)$ from the scorer
\State sample $M\sim\mathrm{Cat}(\omega(\state))$
\State form $\anchorlatent_M$ and $\reldelta_M$
\State sample $\flowstate_0\sim p_0$
\State integrate
    $\dot \flowstate_\tau =
    v_\theta(\flowstate_\flowtime,\flowtime\mid\reldelta_M,\anchorlatent_M)$
\State set $\achunk=\flowstate_1$ and execute its first $T_a$ actions
\State re-observe and repeat
\end{algorithmic}

\vspace{3pt}
\hrule height 0.4pt
\end{minipage}

\end{figure}

\begin{table}[t]
    \centering
    \small
    \setlength{\tabcolsep}{6pt}
    \caption{\textbf{Key hyperparameters for simulation experiments.}
    Anchor parameters apply only to BFP.}
    \label{tab:hyperparameters}
    \begin{tabular}{@{}ll@{}}
        \toprule
        Hyperparameter & Value \\
        \midrule
        Optimizer & Adam, $\beta=(0.9,0.999)$ \\
        Learning rate (BFP predictor/scorer) & $10^{-4}$ \\
        Batch size & 32 queries \\
        Gradient clipping & Global norm 1.0 \\
        EMA power / maximum decay & $0.75 / 0.9999$ \\
        Training steps &
            120k (Relocate-D0), 420k (Relocate-D1) \\
        & 300k (remaining tasks) \\
        ODE solver / integration steps & Euler / 10 \\
        Action horizon / execution horizon &
            MuJoCo: $1/1$; ManiSkill: $16/8$ \\
        Bilinear rank & 32; ablations: 128, 512 \\
        \midrule
        Anchor demonstrations $J$ & $\{1,4\}$ \\
        Frames per demonstration $M$ & 16 \\
        Phase window $w$ & $\{0.25,0.5\}$ \\
        Scorer update interval & 5 predictor steps \\
        Maximum predictor weight floor & 0.5 \\
        \midrule
        Training seeds & 3 \\
        Evaluation conditions per seed & 100 ID + 100 OOD \\
        \bottomrule
    \end{tabular}
\end{table}

This section specifies the retrieval policy used by \ours{}, that includes the specific phase-conditioned candidate set algorithm, scorer architectures, and predictor--scorer updates.

\paragraph{Phase-conditioned Candidate Set $\mathcal{C}(\state)$.}
For query $(x_t,A_t)$ from trajectory $i$ at time $t$, let $f_t$ denote its
phase, $\mathrm{seg}_j(f)=[s_j(f),e_j(f)]$ the phase-$f$ segment of
demonstration $j$, and $\mathrm{el}_t=t-s_i(f_t)$ the elapsed time within
the current phase.
We first select $J$ anchor demonstrations that reach phase $f_t$, uniformly sampled phase-compatible demonstrations, excluding
trajectory $i$.

For each selected demonstration $j$, we align by elapsed phase time and define
\begin{equation}
    \rho_j
    =
    s_j(f_t)+\min\{\mathrm{el}_t,\ell_j(f_t)-1\},
    \qquad
    W_j
    =
    [\rho_j-h_j,\rho_j+h_j]\cap\mathrm{seg}_j(f_t),
    \qquad
    h_j=\left\lfloor w(\ell_j(f_t)-1)\right\rfloor .
\end{equation}
We sample $M$ frames from each $W_j$, yielding
$\mathcal C(x_t)=\{x^\dagger_{tm}\}_{m=1}^{K}$ with
$K=JM$.
\label{app:eset_v2_scorer}

\paragraph{Scorer Architecture $s_\psi$.}
For candidate $\anc{\state}_{tm}\in\mathcal C(\state_t)$, let
\[
    \anchorlatent_{tm}=\obsenc(\anc{\state}_{tm}),
    \qquad
    \reldelta_{tm}=\obsenc(\state_t)-\anchorlatent_{tm}.
\]
MLP scores each candidate from features
$
[\reldelta_{tm}/\sigma,\,
  \reldelta_{tm}/\|\reldelta_{tm}\|_2,\,
  \log\|\reldelta_{tm}\|_2,\,
  \anchorlatent_{tm}]
$.
The resulting responsibility is
\begin{equation}
    q_{tm}
    =
    p_\psi(m\mid x_t,\mathcal C(x_t))
    =
    \operatorname{softmax}_m
    \left(
        \frac{s_\psi(\reldelta_{tm},\anchorlatent_{tm})}{\tau_s}
    \right),
    \label{eq:app_anchor_responsibility}
\end{equation}
where $\tau_s$ is learned and initialized to $1$.

\paragraph{Predictor update.}
We sample predictor queries from $\mathcal B\cup\mathcal C$, construct
$\mathcal C(\state_i)$, and compute the responsibilities in
\Cref{eq:app_anchor_responsibility}.
With these responsibilities treated as fixed, the predictor minimizes
\begin{equation}
    \mathcal L_\theta
    =
    \frac{1}{|\mathcal B_\theta|}
    \sum_{i\in\mathcal B_\theta}
    \sum_{m=1}^{K}
    \sg[w_{im}]\,
    \ell_{\mathrm{BFP}}
    (\theta;\state_i,\anc{\state}_{im},\achunk_i),
    \label{eq:app_predictor_update}
\end{equation}
where $w_{im}$ denotes the responsibility after applying the scheduled
predictor-weight floor.

% TODO: Insert the exact map q_{im} -> w_{im}.
% alpha_max = 0.5, ramped between 15% and 50% of training.

\paragraph{Scorer update.}
Every $K_\theta=5$ predictor updates, we draw a fresh batch
$\mathcal B_\psi$ of $32$ queries from the scorer pool $\mathcal A$.
For each query and candidate, we evaluate the BFP loss using an
exponential-moving-average predictor $\bar\theta$.
We average over $R_\psi=4$ flow contexts,
\begin{equation}
    \bar\ell_{im}
    =
    \frac{1}{R_\psi}
    \sum_{r=1}^{R_\psi}
    \ell_{\mathrm{BFP}}^{(r)}
    (\bar\theta;\state_i,\anc{\state}_{im},\achunk_i),
    \label{eq:app_candidate_loss}
\end{equation}
using the same sampled flow contexts across candidates of a query.
This shared sampling makes candidate losses directly comparable.

Because the absolute CFM-loss scale varies across queries, candidate losses
are robustly normalized within each shortlist:
\begin{equation}
    \widetilde\ell_{im}
    =
    \frac{
        \bar\ell_{im}
        -
        \mathrm{med}_{n}(\bar\ell_{in})
    }{
        \mathrm{MAD}_{n}(\bar\ell_{in})
        +\epsilon
    }.
    \label{eq:app_robust_candidate_loss}
\end{equation}
The scorer minimizes
\begin{equation}
    \mathcal L_\psi
    =
    \frac{1}{|\mathcal B_\psi|}
    \sum_{i\in\mathcal B_\psi}
    \sum_{m=1}^{K}
    q_{im}\,
    \sg[\widetilde\ell_{im}]
    +
    \lambda_{\rm nc}\mathcal R_{\rm nc},
    \label{eq:app_scorer_update}
\end{equation}
so candidates with lower held-out BFP loss receive higher probability.

To prevent collapse onto a small number of candidates, we use an entropy
hinge
\begin{equation}
    \mathcal R_{\rm nc}
    =
    \frac{1}{|\mathcal B_\psi|}
    \sum_{i\in\mathcal B_\psi}
    \left[
        \eta-H(q_i)
    \right]_+,
    \qquad
    \eta=0.6\log K,
    \qquad
    \lambda_{\rm nc}=1.
    \label{eq:app_noncollapse}
\end{equation}
Predictor and scorer parameters are optimized separately, and gradients are
not propagated through the responsibilities in
\Cref{eq:app_predictor_update} or through candidate losses in
\Cref{eq:app_scorer_update}.

\clearpage
\section{Numerical Example: Distributional Extrapolation}
\label{app:toy}

\begin{figure}[h]
    \centering
    \includegraphics[
        width=\textwidth,
        trim={0 0 0 0},
        clip
    ]{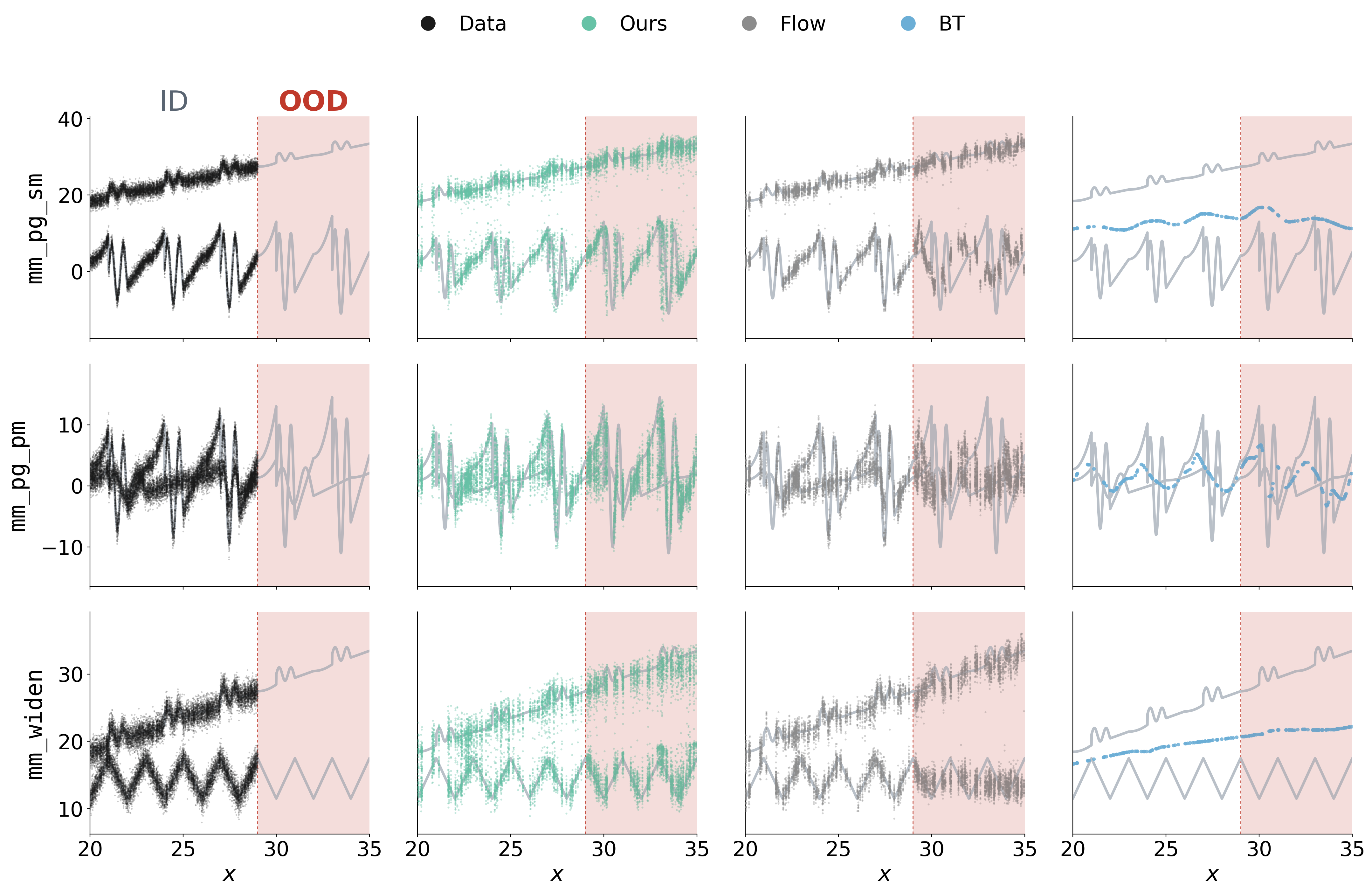}
    \caption{
        \textbf{Distributional extrapolation sample visualization.}
        Rows correspond to different tasks: \texttt{mm\_pg\_sm}, \texttt{mm\_pg\_pm}, and \texttt{mm\_widen}, while columns show the samples of training data (gray), \ours (red), \flow{} (purple), and \brp{} (orange).
    }
    \label{fig:toy_all_generators}
\end{figure}

\begin{table}[t]
    \centering
    \scriptsize
    \setlength{\tabcolsep}{5pt}
    \renewcommand{\arraystretch}{1.05}
    \begin{tabular*}{\columnwidth}{
    @{\extracolsep{\fill}}lcccc@{}
}
        \toprule
        &
        \multicolumn{2}{c}{\(\mathrm{NW}_2 \downarrow\)} &
        \multicolumn{2}{c}{Balance \(\downarrow\)} \\
        \cmidrule(lr){2-3}
        \cmidrule(l){4-5}
        Method & ID & OOD & ID & OOD \\
        \midrule

        \multicolumn{5}{@{}l}{\texttt{mm\_pg\_sm}} \\
        \emph{Oracle}
            & \(0.10{\pm}0.01\)
            & \(0.11{\pm}0.01\)
            & \(0.02{\pm}0.00\)
            & \(0.02{\pm}0.00\) \\
        Flow
            & \(0.23{\pm}0.01\)
            & \(0.30{\pm}0.01\)
            & \(0.09{\pm}0.00\)
            & \(0.07{\pm}0.01\) \\
        \brp{}
            & \(0.51{\pm}0.00\)
            & \(0.53{\pm}0.00\)
            & \(0.50{\pm}0.00\)
            & \(0.50{\pm}0.00\) \\
        \ours{}
            & \(\mathbf{0.15{\pm}0.01}\)
            & \(\mathbf{0.17{\pm}0.00}\)
            & \(\mathbf{0.03{\pm}0.00}\)
            & \(\mathbf{0.03{\pm}0.01}\) \\

        \addlinespace[2pt]
        \multicolumn{5}{@{}l}{\texttt{mm\_pg\_pm}} \\
        \emph{Oracle}
            & \(0.07{\pm}0.00\)
            & \(0.06{\pm}0.00\)
            & \(0.02{\pm}0.00\)
            & \(0.02{\pm}0.00\) \\
        Flow
            & \(\mathbf{0.25{\pm}0.00}\)
            & \(0.97{\pm}0.01\)
            & \(0.09{\pm}0.00\)
            & \(0.33{\pm}0.00\) \\
        \brp{}
            & \(1.13{\pm}0.05\)
            & \(1.24{\pm}0.08\)
            & \(0.50{\pm}0.00\)
            & \(0.50{\pm}0.00\) \\
        \ours{}
            & \(0.43{\pm}0.01\)
            & \(\mathbf{0.58{\pm}0.03}\)
            & \(\mathbf{0.07{\pm}0.01}\)
            & \(\mathbf{0.20{\pm}0.00}\) \\

        \addlinespace[2pt]
        \multicolumn{5}{@{}l}{\texttt{mm\_widen}} \\
        \emph{Oracle}
            & \(0.07{\pm}0.00\)
            & \(0.09{\pm}0.01\)
            & \(0.02{\pm}0.00\)
            & \(0.02{\pm}0.00\) \\
        Flow
            & \(0.19{\pm}0.00\)
            & \(0.25{\pm}0.00\)
            & \(0.09{\pm}0.00\)
            & \(\mathbf{0.08{\pm}0.00}\) \\
        \brp{}
            & \(0.53{\pm}0.00\)
            & \(0.51{\pm}0.00\)
            & \(0.50{\pm}0.00\)
            & \(0.50{\pm}0.00\) \\
        \ours{}
            & \(\mathbf{0.10{\pm}0.01}\)
            & \(\mathbf{0.21{\pm}0.01}\)
            & \(\mathbf{0.03{\pm}0.00}\)
            & \(\mathbf{0.08{\pm}0.00}\) \\

        \bottomrule
    \end{tabular*}
    \caption{
        \textbf{Controlled distributional extrapolation.}
        Mean \(\pm\) standard deviation over three seeds.
        Bold denotes the best learned method; Oracle is the finite-sampling floor.
    }
    \label{tab:toy}
\end{table}

\begin{table}[t]
    \centering
    \small
    \setlength{\tabcolsep}{6pt}
    \begin{tabular}{lcc}
        \toprule
        Problem & \(f_1\) & \(f_2\) \\
        \midrule
        \texttt{mm\_pg\_sm}
            & \(f_{\mathrm{pg}}\)
            & \(f_{\mathrm{sm}}\) \\
        \texttt{mm\_pg\_pm}
            & \(f_{\mathrm{pg}}\)
            & \(f_{\mathrm{pm}}\) \\
        \texttt{mm\_widen}
            & \(f_{\mathrm{sm}}\)
            & \(f_{\mathrm{sbl}}\) \\
        \bottomrule
    \end{tabular}
    \caption{
        \textbf{Distributional extrapolation problem setting.}
        Each problem combines two mode functions from
        \eqref{eq:toy_mode_functions}.
    }
    \label{tab:toy_generators}
\end{table}

This section specifies the controlled experiment summarized in \Cref{fig:bfp_regression} and \Cref{tab:toy}.
The figure in the main text reports \texttt{mm\_pg\_sm}, while we report two additional problems.
The experiment separates deterministic bilinear extrapolation, conditional generation without a relational extrapolation rule, and their combination in BFP.

\subsection{Expert Conditional Distribution}
\label{app:toy_problem_definition}
We define the expert conditional distribution \(\expertdist(\cdot\mid x)\) using pairs of analytical mode functions.
For \(T>0\), let
\begin{equation}
    r_T(x)
    :=
    x-T\left\lfloor \frac{x}{T}\right\rfloor
    \in[0,T)
\end{equation}
denote the remainder of \(x\) modulo \(T\).
The shared base profile \(g:[0,3)\rightarrow\R\) is
\begin{equation}
    g(u)
    :=
    \begin{cases}
        \sin(10u),
            & 0\le u<1,\\
        \sin(10)+(u-1),
            & 1\le u<2,\\
        \sin(10)+1+(u-2)^2,
            & 2\le u<3.
    \end{cases}
    \label{eq:toy_base_profile}
\end{equation}

We additionally define the time warp
\begin{equation}
    w(x)
    :=
    \begin{cases}
        r_9(x),
            & r_9(x)<3,\\
        \bigl(r_9(x)-3\bigr)/2,
            & r_9(x)\ge 3,
    \end{cases}
    \label{eq:toy_time_warp}
\end{equation}
and the triangular wave
\begin{equation}
    \operatorname{tri}(x)
    :=
    \begin{cases}
        r_1(x),
            & \lfloor x\rfloor \text{ even},\\
        1-r_1(x),
            & \lfloor x\rfloor \text{ odd}.
    \end{cases}
    \label{eq:toy_triangular_wave}
\end{equation}
The four mode functions are
\begin{align}
    f_{\mathrm{pg}}(x)
        &:=
        \left\lfloor \frac{x}{3}\right\rfloor g(r_3(x)),
        &
    f_{\mathrm{sm}}(x)
        &:=
        3\left\lfloor \frac{x}{3}\right\rfloor+g(r_3(x)),
        \nonumber\\
    f_{\mathrm{pm}}(x)
        &:=
        \left\lfloor \frac{x}{9}\right\rfloor g(w(x)),
        &
    f_{\mathrm{sbl}}(x)
        &:=
        6\operatorname{tri}(x)+11.5.
    \label{eq:toy_mode_functions}
\end{align}

For each problem, the expert conditional distribution is
\begin{equation}
    \expertdist(y\mid x)
    =
    \frac12\mathcal N\!\left(f_1(x),\sigma^2\right)
    +
    \frac12\mathcal N\!\left(f_2(x),\sigma^2\right),
    \qquad
    \sigma=1,
    \label{eq:toy_conditional_law}
\end{equation}
where \(f_1\) and \(f_2\) are selected according to
\Cref{tab:toy_generators}.

The training and OOD condition distributions are
\begin{equation}
    P_{\mathrm{tr}}^X
    =
    \operatorname{Unif}([20,29]),
    \qquad
    P_{\mathrm{te}}^X
    =
    \operatorname{Unif}([29,35]).
    \label{eq:toy_condition_distributions}
\end{equation}
Thus, evaluation changes only the support of the condition while preserving the same expert conditional law.

\subsection{Experimental Details}
\label{app:toy_evaluation}

\textbf{Setting.}
We sample \(2,000\) training conditions from \(P_{\mathrm{tr}}^X\) and 10 independent targets from \(\expertdist(\cdot\mid x)\) for each condition, yielding \(20,000\) training data.
Evaluation uses \(100\) ID conditions from \(P_{\mathrm{tr}}^X\) and \(100\) OOD conditions from \(P_{\mathrm{te}}^X\).
We compare a standard Flow model, \brp{}, and \ours{}.
All models use depth-\(3\), width-\(1{,}000\) MLPs and are trained for 300 epochs, 3 random seeds.
Flow and \ours{} use \(50\) Euler steps for generation.
\brp{} and \ours{} use $r=16$ for rank.

\textbf{Metric.}
For each query \(x\), we draw \(K=500\) samples and form the empirical policy law \(\widehat q_{\theta,K}(\cdot\mid x)\).
Our primary metric is the Wasserstein distance to the true expert conditional, normalized by the separation between its two modes:
\begin{equation}
    \operatorname{NW}_2(x)
    :=
    \frac{
        W_2\!\left(
            \widehat q_{\theta,K}(\cdot\mid x),
            \expertdist(\cdot\mid x)
        \right)
    }{
        |f_1(x)-f_2(x)|
    }.
    \label{eq:toy_normalized_w2}
\end{equation}
We also report mode-balance error.
Let
\[
    f_{\mathrm{low}}(x):=\min\{f_1(x),f_2(x)\},
    \qquad
    f_{\mathrm{high}}(x):=\max\{f_1(x),f_2(x)\},
\]
and define the empirical probability assigned to the lower mode as
\begin{equation}
    \widehat p_{\mathrm{low}}(x)
    :=
    \frac{1}{K}
    \sum_{k=1}^{K}
    \mathbbm{1}\!\left[
        |\widehat y_k(x)-f_{\mathrm{low}}(x)|
        \le
        |\widehat y_k(x)-f_{\mathrm{high}}(x)|
    \right].
    \label{eq:toy_mode_assignment}
\end{equation}
The mode-balance error is
\begin{equation}
    \operatorname{Bal}(x)
    :=
    \left|
        \widehat p_{\mathrm{low}}(x)-\frac12
    \right|.
    \label{eq:toy_balance}
\end{equation}
Because both metrics have a finite-sample floor, we measure the \textit{oracle row} by applying the same evaluation to samples drawn directly from \(\expertdist(\cdot\mid x)\).

\textbf{Result.}
\ours{} achieves the lowest OOD \(\operatorname{NW}_2\) on all three problems, with a mean of \(0.321\) compared with \(0.506\) for Flow, a \(36.6\%\) reduction.
It also yields better OOD mode balance on two problems and ties Flow on the third.
In contrast, \brp{} cannot represent both modes and performs worst in distributional error.
These results show that bilinear structure provides the extrapolation rule, while the generative flow head is necessary to preserve the conditional distribution.
We visualize the full samples at \Cref{fig:toy_all_generators}.

%%%%%%%%%%%%%%%%%%%%%%%%%%%%%%%%%%%%%%%%%%%%
% Implementation Details for Reproducibility
%%%%%%%%%%%%%%%%%%%%%%%%%%%%%%%%%%%%%%%%%%%%
\clearpage
\section{Baseline Instantiation and Ablation Results}
\label{app:sec:baseline_study}

\begin{figure}[h]
    \centering

    % (a) Encoder sensitivity table
    \begin{subfigure}{0.82\linewidth}
        \centering
        \normalsize
        \setlength{\tabcolsep}{5pt}
        \renewcommand{\arraystretch}{1.08}
        \begin{tabular}{@{}lcccccc@{}}
            \toprule
            & \multicolumn{2}{c}{\flow{}}
            & \multicolumn{2}{c}{CTA}
            & \multicolumn{2}{c}{\ours{}} \\
            \cmidrule(lr){2-3}
            \cmidrule(lr){4-5}
            \cmidrule(l){6-7}
            Encoder
            & ID & OOD
            & ID & OOD
            & ID & OOD \\
            \midrule
        
            Scratch
            & 96.9 & 36.9
            & 85.4 & 9.2
            & 59.4 & \textbf{45.0}
            \\
        
            R3M-34 (frozen)
            & 96.3 & 30.2
            & 50.4 & 6.8
            & 99.4 & \textbf{76.4}
            \\
        
            \bottomrule
        \end{tabular}
        \caption{\textbf{Is PVR sufficient for extrapolation?}}
        \label{fig:encoder_sensitivity}
    \end{subfigure}

    \vspace{0.8em}

    % (b) EquiDP sensitivity figure
    \begin{subfigure}{\linewidth}
        \centering
        \includegraphics[width=0.8\linewidth]{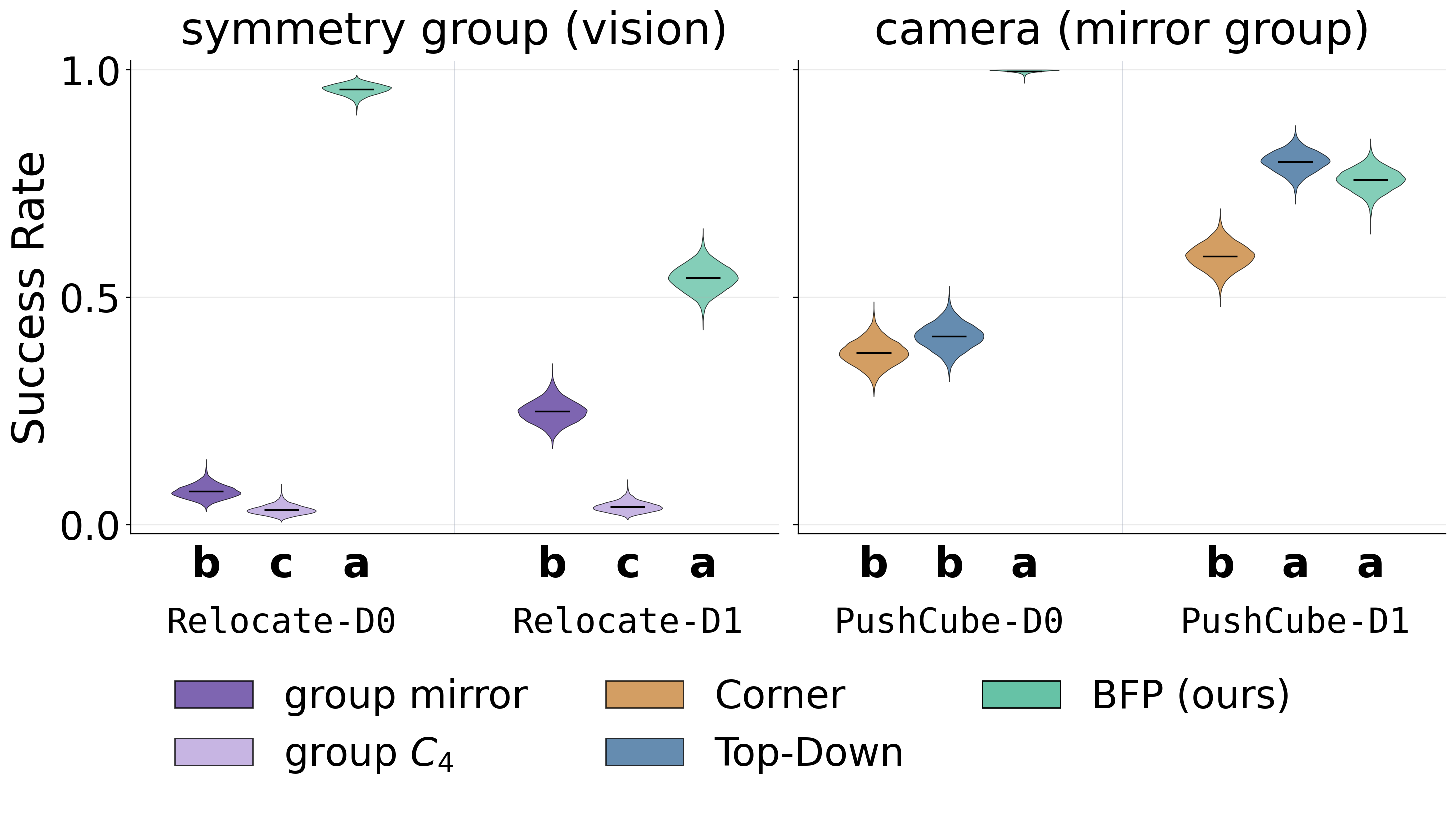}
        \caption{\textbf{\ours{} wins hand-specified Extrapolation Rule.}}
        \label{fig:equidp_sensitivity}
    \end{subfigure}

    \caption{
        \textbf{Sensitivity to representation and prescribed extrapolation rules.}
        (a) ID/OOD success rate under visuomotor, averaged over five tasks.
        Scratch uses a jointly trained ResNet-34 for \flow{} and \ours{} and IMPALA encoder for CTA.
        Bold denotes the best-performing method within each encoder setting and split.
        (b) OOD sensitivity of \edp{} to the prescribed symmetry and camera viewpoint.
    }
    \label{fig:baseline_sensitivity}
\end{figure}

\textbf{Overview.}
Our main result (\Cref{fig:visuomotor_ood_summary}) uses the baseline instantiations we found most appropriate for unseen-goal generalization.
Because several methods were not originally designed for this setting, we validate the implementation choices that affect OOD performance.
These studies also show that (i) BFP's gains are not explained by pretrained visual representations alone, and (ii) learning the extrapolation rule from demonstrations can be preferable to prescribing it a priori.

\textbf{Setup.}
We match the demonstration data, evaluation episodes, action representation, prediction and execution horizons, and policy architecture whenever compatible with the original method.
We retain method-specific components when they define the baseline's generalization mechanism, such as the equivariant structure of \edp{} or the
learned representation of CTA.
\Cref{tab:baseline_design} summarizes the resulting comparison and the main design axes separating each baseline from \ours{}.

\subsection{Implementation Details}
\label{app:impl_details}

\textbf{Flow Policy.}
\flow{} uses the same frozen R3M representation, conditional flow-matching objective, action
parameterization, and policy backbone as \ours{}, but directly conditions on the observation and
goal without an anchor--residual extrapolation rule.

\textbf{Bilinear Regression Policy.}
\brp{} uses the same frozen R3M representation and policy backbone as \ours{}, while retaining the
synchronized anchor policy of bilinear transduction.
It replaces conditional flow matching with deterministic squared-loss action regression.

\textbf{DARP.}
We retain DARP's local anchor retrieval and residual-based extrapolation rule, but use conditional flow matching for a matched generative comparison.
We use the same frozen R3M representation and policy backbone as \ours{} and tune its method-specific retrieval hyperparameters.

\textbf{CTA.}
We retain CTA's native hierarchical policy and learned IMPALA representation, which jointly define
its temporal-distance analogy mechanism.
Its low-level policy predicts primitive actions, so we evaluate CTA with $T_p=1$ in the main
comparison.
For Tp=16 task, we use Q-chunking.

\textbf{Equivariant Policy.}
\edp{} uses conditional flow matching while retaining its native equivariant architecture and task-specific geometric group action.
We use the common benchmark camera in the main comparison; the effect of symmetry and viewpoint
choice is studied below

\subsection{Additional Baseline Studies}
\label{app:baseline_sensitivity}

\textbf{Is pretrained visual representation sufficient?}
We test whether BFP's visuomotor gains can be explained by pretrained visual representations (PVRs) alone.
As shown in \Cref{fig:encoder_sensitivity}, frozen R3M does not improve \flow{} over a jointly trained ResNet-34 ($30.2\%$ vs.\ $36.9\%$ OOD), and degrades CTA relative to its from-scratch IMPALA encoder ($6.8\%$ vs.\ $9.2\%$).
In contrast, R3M substantially improves \ours{} from $45.0\%$ to $76.4\%$ OOD success.
Thus, PVR is not itself sufficient for extrapolation; rather, its benefit in \ours{} appears to come from providing a representation in which anchor--residual differences form a useful coordinate system for bilinear extrapolation.

\textbf{How sensitive is a prescribed extrapolation rule?}
\Cref{fig:baseline_sensitivity} evaluates two choices required by \edp{}: the prescribed symmetry group
and the camera viewpoint.
On \texttt{Relocate-D1}, changing the geometric group reduces OOD success from $26\%$ to $6\%$, while on \texttt{PushCube-D1}, using a symmetry-aligned top-down camera improves OOD success from $59\%$ to $80\%$.
These choices are not reliably identified by ID performance.
In the same settings, \ours{} remains competitive or stronger without requiring the deployment transformation to be specified in advance.
This sensitivity highlights the value of learning the extrapolation rule from demonstrations rather than prescribing it a priori.

\begin{table}[t]
    \centering
    \normalsize
    \setlength{\tabcolsep}{4.0pt}
    \renewcommand{\arraystretch}{1.08}
    \caption{
        \textbf{Design axes of the main baselines.}
        The table separates three components of each method:
        the observation representation, the extrapolation rule used beyond demonstrated goals,
        and the prediction model used to generate actions.
        This decomposition highlights which components are matched to \ours{} and which define each baseline's extrapolation mechanism.
    }
    \label{tab:baseline_design}
    \begin{tabular}{@{}llll@{}}
        \toprule
        Method
        & Representation
        & Extrapolation Rule
        & Prediction Model
        \\
        \midrule

        \ours{}
        & PVR
        & learned anchor--residual bilinear
        & conditional flow
        \\

        \flow{}
        & PVR
        & none
        & conditional flow
        \\

        \brp{}
        & PVR
        & synchronized anchor--residual bilinear
        & deterministic regression
        \\

        \darp{}
        & PVR
        & local retrieved residual
        & conditional flow
        \\

        \edp{}
        & Unfrozen
        & prescribed geometric equivariance
        & conditional flow
        \\

        CTA
        & Unfrozen
        & learned temporal-distance analogy
        & hierarchical policy
        \\

        \bottomrule
    \end{tabular}
\end{table}

\subsection{Pre-rollout Extrapolation Diagnosis}
\label{app:extrapolation_diagnostics}

We evaluate whether the two theory-motivated diagnostics in \Cref{fig:extrapolation_diagnostics} can anticipate extrapolation error without executing the learned policy in closed loop.
To do so, we first construct a privileged measure of OOD prediction error, which is used only as the response variable for evaluating the diagnostics.

\textbf{Privileged OOD error.}
For each task, we collect 8 expert rollouts under held-out goals.
These rollouts provide OOD observation--action pairs: the observations are those induced by unseen goals, while the corresponding actions are labeled by the expert policy.
This privileged OOD information is \emph{not} provided to either diagnostic; it is used only to measure whether the diagnostic agrees with the error that would be observed if OOD expert labels were available.
For each learned checkpoint, we replay the OOD observations open loop and evaluate its flow-matching (FM) loss against the expert action targets.
The velocity-matching MSE is averaged over the fixed flow-time grid $\tau\in\{0.2,0.4,0.6,0.8,1.0\}$ and frozen Gaussian noise samples.
Averaging this loss over the evaluated OOD observations gives the checkpoint-level OOD error used in \structurereq{}; for \coveragereq{}, we retain the corresponding FM error for each individual OOD query.

\textbf{\structurereq{}: validation error for checkpoint selection.}
The diagnostic available without OOD labels is the FM error on a held-out in-distribution validation split.
We ask whether this held-out ID error can identify, among policies trained for the same task, which checkpoint will have lower privileged OOD error.
To isolate this model-selection question, checkpoints are compared only within same pair of the task and observation space.
The analysis contains five tasks under state and visuomotor observations.
We report the result across $180$ trained checkpoints in total.

Within each task, ID and OOD FM errors are independently converted to percentile ranks.
For visualization in \Cref{fig:extrapolation_diagnostics}, the within-pool ID-error ranks are pooled and divided into quintiles, and we plot the mean OOD-error percentile rank in each quintile.
This evaluation therefore asks whether held-out ID error preserves the \emph{ordering} of checkpoints that would be obtained from privileged OOD error, rather than their absolute values.

\textbf{\coveragereq{}: residual coverage for goal selection.}
The second diagnostic operates within a fixed \ours{} checkpoint and asks whether the residual induced for a particular OOD query is supported by the residuals encountered during training.
For each checkpoint, we construct a reference bank of $20{,}000$ training residuals.
A training query $\state$ is paired with an anchor $\anc{\state}$ sampled from the learned train-time anchor policy, and we store
\begin{equation}
    \reldelta
    =
    \obsenc(\state)-\obsenc(\anc{\state}).
\end{equation}
The resulting bank therefore represents the residual distribution induced by the learned anchor policy on training data.

For an OOD query, the deployment-time anchor policy selects $\anc{\state}$ and forms its query residual $\reldelta_{\mathrm{OOD}}$ in the same representation space.
We score the coverage of this residual against the training residual bank using the realism metric \citep{kynkaanniemi2019improved}; larger scores
indicate that the query residual lies more deeply inside the local support of the training residual distribution.

Within each checkpoint, OOD queries are ranked by residual coverage and divided into quintiles, while their FM errors are independently converted to within-checkpoint percentile ranks.
\Cref{fig:extrapolation_diagnostics} aggregates the final rank-$32$ checkpoint from $42$ trained policies---seven simulation settings, two observation spaces, and three seeds.
For each coverage quintile, we report the median error rank across these $42$ checkpoints.

\end{document}